\documentclass[10pt,twocolumn]{article}
\usepackage[margin=0.75in]{geometry}
\usepackage{amsmath,amssymb,amsthm}
\usepackage{mathtools}
\usepackage{bm}
\usepackage{graphicx}
\usepackage{cite}
\usepackage{algorithm}
\usepackage{algorithmic}
\usepackage{booktabs}
\usepackage{url}
\usepackage{fancyhdr}
\usepackage{titlesec}
\titleformat{\section}{\normalfont\large\scshape\centering}{\Roman{section}.}{0.5em}{}
\titleformat{\subsection}{\normalfont\itshape}{\ \Alph{subsection}.}{0.5em}{}

\newtheorem{theorem}{Theorem}
\newtheorem{lemma}[theorem]{Lemma}
\newtheorem{corollary}[theorem]{Corollary}
\newtheorem{proposition}[theorem]{Proposition}
\newtheorem{definition}[theorem]{Definition}
\newtheorem{conjecture}[theorem]{Conjecture}
\newtheorem{remark}{Remark}
\newtheorem{example}{Example}
\newtheorem{condition}{Condition}

\DeclareMathOperator{\rank}{rank}
\DeclareMathOperator{\diag}{diag}
\DeclareMathOperator{\Tr}{Tr}
\DeclareMathOperator{\spn}{span}
\DeclareMathOperator{\Var}{Var}
\DeclareMathOperator{\GL}{GL}

\DeclareMathOperator{\Irr}{Irr}

\title{Algebraic Diversity: Group-Theoretic Spectral Estimation\\from Single Observations}

\author{Mitchell~A.~Thornton,~\textit{Senior~Member,~IEEE}\\[4pt]
\small Richardson, TX 75080 USA}

\date{June 29, 2026 \quad}

\begin{document}
\maketitle

\begin{abstract}
We establish a general theoretical framework demonstrating that temporal averaging over multiple independent observations of a noisy signal is a special case of algebraic group action, specifically, the degenerate case in which the trivial group $G = \{e\}$ is applied to each observation independently, and that selecting a richer group yields equivalent or superior second-order statistical information from a single observation.
We define a \emph{group-averaged estimator} $\mathbf{F}_G$ constructed by applying the action of a finite group $G$ to a single observation vector, and we prove a \emph{General Replacement Theorem} establishing that $\mathbf{F}_G$ provides a consistent estimator of the population-level subspace decomposition under two conditions: (i) the signal component transforms predictably (equivariantly) under the group action, and (ii) the noise distribution is invariant (ergodic) under the group action.
We then prove a \emph{Commutativity--KL Equivalence}: a group-averaged estimator $\mathbf{F}_G$ shares the Karhunen--Lo\`eve (KL) eigenbasis exactly when the population covariance commutes with the group action, so the \emph{matched} group attains the KL transform, which is itself optimal among all linear decorrelating transforms in variance concentration, mutual orthogonality, and minimum reconstruction error. The symmetric group $S_M$ realizes the KL transform through its Cayley graph spectral decomposition, a construction distinct from the group-averaged estimator $\mathbf{F}_G$; for $\mathbf{F}_G$ itself, enlarging the group beyond the matched order does not improve and in fact degrades the decomposition (the PASE result below), so the operative optimum is the matched group rather than $S_M$.
The framework is demonstrated through the MUSIC (Multiple Signal Classification) algorithm for direction-of-arrival estimation, where we prove that a Cayley graph construction from a single snapshot achieves equivalent pseudospectral peaks to multi-snapshot covariance-based MUSIC, and through massive MIMO channel estimation, where single-pilot algebraic diversity achieves up to 64\% higher effective throughput than MMSE estimation by eliminating the pilot overhead that dominates large-array systems. A third application to single-pulse waveform characterization demonstrates the constructive pipeline: the framework independently derives the classical ``dechirp-then-DFT'' operation from first principles, identifies it as group conjugation, and extends it with blind chirp rate estimation via spectral concentration maximization, achieving $8.3\times$ higher eigenvalue concentration than the cyclic group on chirp signals. The approach is robust to $-2$~dB SNR and enables four-class waveform classification (tone, chirp, multi-tone, noise-like) at 90\% accuracy from a single pulse. In a head-to-head comparison, matched-group AD identifies LFM chirps at 8~dB lower SNR than FFT-based classification and is the only method that achieves reliable performance across all four waveform classes. Against a simulated non-stationary modulated source that changes waveform parameters every pulse, AD-Matched maintains $89\%$ classification accuracy while FFT-based processing plateaus at $53\%$. A fourth application to graph signal processing investigates whether genuinely non-Abelian groups can outperform conjugated cyclic groups. A systematic filtering pipeline reduces all 156 non-isomorphic graphs on $n = 6$ vertices to seven candidates with $S_3$ automorphism groups, of which three exhibit significant spectral concentration advantage over the best conjugated cyclic group, leading to the Non-Abelian Dominance Hypothesis (NADH) as an open conjecture. We prove an Automorphism Characterization Theorem establishing that $\delta(\mathbf{P}_\sigma, \mathbf{R}) = 0$ if and only if $\sigma$ is a graph automorphism for graph-diffusion covariance, providing an exact algebraic oracle for graph symmetry detection, and we derive a Permutation Commutator Formula expressing $\|[\mathbf{P}_\sigma, \mathbf{R}]\|_F^2$ as the eigenvalue-gap-weighted mixing of the eigenvectors of $\mathbf{R}$ by the permutation, which reduces to a sum of squared eigenvalue displacements when the permutation carries eigenvectors to eigenvectors.
A fifth application to transformer neural networks applies the AD diagnostics to the internal representations of large language models (LLM); preliminary results and ongoing verification are discussed in Section~\ref{sec:llm}.
We further extend the framework to colored (non-white) noise environments by showing that the noise covariance matrix itself admits a group-theoretic characterization: a noise-only observation processed through the algebraic diversity framework reveals a \emph{natural group} whose representation best diagonalizes the noise covariance, and the proximity of this group's representation to the identity quantifies the degree of spectral coloring through an \emph{algebraic coloring index}.
The central insight is that temporal averaging is not an alternative to algebraic group action but rather a limiting case of it: conventional processing implicitly applies the trivial group $G = \{e\}$ to each observation, producing the rank-one outer product $\mathbf{x}\mathbf{x}^H$, and accumulates $L$ such outer products to build rank. Algebraic diversity generalizes this by applying a richer group to a single observation, exhaustively exploring its internal symmetry structure to achieve full-rank estimation without multiple measurements. Both mechanisms project out the ergodic noise component to reveal the deterministic signal; the difference lies in whether the diversity comes from temporal repetition (many observations, trivial group) or algebraic structure (one observation, matched group).
The practical consequences are immediate: the group-averaged estimator achieves full-rank covariance from a single snapshot (eliminating the cold-start period of adaptive systems), delivers up to $10\log_{10}(M)$~dB of group gain for group-matched signals with no tuning, and, through the PASE result, requires exactly $n = M$ group elements, reducing adaptation latency from multiple snapshot intervals to one.
We then establish \emph{Permutation-Averaged Spectral Estimation} (PASE), proving that the optimal number of group elements for the group-averaged estimator is exactly $n = |G|$ (the group order): fewer elements leave estimation quality on the table, while more elements, drawn from outside the matched group, actively degrade the estimate. A systematic comparison of four permutation ordering strategies (random, Steinhaus--Johnson--Trotter, Lehmer, and Heap) applied to the symmetric group $S_M$ confirms that subsampling $S_M$ yields monotonically degrading performance regardless of ordering, proving that the group selection problem cannot be circumvented. The PASE result collapses the entire framework to a single free parameter: the choice of algebraic group. We formalize this as the \emph{blind group matching problem} and show that, for signals whose covariance admits a unitary transformation to circulant form, the problem reduces from a combinatorial search to continuous parameter estimation via spectral concentration maximization.
We conjecture that this algebraic averaging principle extends to all $G$-compatible statistics, not only the outer product: the \emph{General Algebraic Averaging Conjecture} states that for any statistic $f$, the group-averaged estimator achieves variance $\propto 1/d_{\mathrm{eff}}(G, f)$ where $d_{\mathrm{eff}}$ is the effective dimension of the orbit average. This identifies the law of large numbers as the trivial-group case of algebraic averaging. Monte Carlo experiments on the first four sample moments confirm the conjecture to four-digit precision.
\end{abstract}

\medskip\noindent\textbf{Index Terms}: Algebraic diversity, temporal averaging, trivial group embedding, $(G,L)$ continuum, general algebraic averaging, effective dimension, information structure, Karhunen--Lo\`eve transform, group action, symmetric group, Cayley graphs, subspace estimation, single-observation inference, MUSIC algorithm, massive MIMO, channel estimation, pilot overhead, chirp characterization, group conjugation, information extraction, colored noise, noise characterization, algebraic coloring index, permutation-averaged spectral estimation, PASE, group matching, blind estimation, spectral concentration, conjugated groups, signal-adapted transforms, transformer representations, rotary position embedding, graph automorphism characterization, permutation commutator formula, CAD--DAD bridge.

\section{Introduction}\label{sec:intro}

\noindent\textbf{W}hy is the Fourier transform the dominant tool in signal processing? The standard answer invokes computational efficiency (the FFT), historical momentum, or the empirical observation that ``it works.'' A more precise answer, that sinusoidal basis functions match sinusoidal signals, is correct but incomplete: it does not explain \emph{why} the sinusoidal basis is optimal for this signal class, nor does it predict when the Fourier transform will fail or what should replace it.

The framework developed in this paper provides a complete answer. The discrete Fourier transform (DFT) is the spectral decomposition associated with the cyclic group $\mathbb{Z}_M$, the group of cyclic shifts on $M$ elements. Its basis functions, the complex exponentials $e^{i2\pi k n/M}$, are the irreducible representations of $\mathbb{Z}_M$. When a signal's covariance matrix commutes with the cyclic shift operator (Proposition~\ref{prop:commute_kl}), the DFT basis coincides with the Karhunen--Lo\`eve (KL) basis, the provably optimal linear transform for decorrelation, variance concentration, and reconstruction. For periodic signals, whose covariance is circulant (shift-invariant), this commutativity holds exactly. The Fourier transform is not special because of its basis functions; it is special because the cyclic group is the correctly matched group for the overwhelmingly common class of periodic signals. Every engineer who computes a DFT is implicitly selecting the cyclic group and exploiting its algebraic structure, without knowing it.

This observation immediately raises the question that motivates the present work: \emph{what happens when the signal is not periodic?} When the covariance is not circulant, the DFT is not the KL transform, and the cyclic group is no longer the correct choice. Every DFT-based processing step, filtering, spectral estimation, beamforming, covariance estimation, then operates in a suboptimal spectral domain, with consequences that propagate through the entire signal processing pipeline. The discrete cosine transform (DCT), which is the spectral decomposition of the dihedral group $D_{2M}$ acting on the even-reflected extension of the $M$ samples~\cite{thornton2026unification}, is optimal for signals with symmetric (even) boundary conditions. But neither the DFT nor the DCT is optimal for signals whose covariance structure corresponds to neither cyclic nor dihedral symmetry, and such signals are the rule rather than the exception in applications involving irregular sampling, non-stationary environments, or complex spatial geometries.

The framework of \emph{algebraic diversity} (AD) generalizes this correspondence to arbitrary finite groups: given any finite group $G$ acting on the observation space, the group's irreducible representations define a spectral transform, and that transform is the KL transform if and only if the signal covariance commutes with the group's Cayley graph adjacency matrix. The central result is a \emph{General Replacement Theorem} proving that a single observation, when processed through the action of a matched group, yields a full-rank covariance estimate whose eigendecomposition provides the same subspace information as multi-snapshot temporal averaging, which, as we show, is itself the degenerate case of algebraic diversity with the trivial group $G = \{e\}$ applied independently to each snapshot. The mechanism is that each group element generates an algebraically distinct ``view'' of the observation: the structured signal transforms predictably (equivariantly) under the group action, while the unstructured noise is scrambled differently by each element. The eigendecomposition of the group-averaged matrix then separates the structured from the unstructured, precisely as temporal averaging does, but from a single measurement.

A second foundational question is also answered: \emph{among all possible groups, which is optimal?} The universal optimum is the KL transform itself, and a group attains it precisely when the signal covariance commutes with the group action; the matched group is therefore the operative optimum. The symmetric group $S_M$ realizes the KL transform through its Cayley graph spectral decomposition (the construction of~\cite{thornton2005}), but this is a property of that construction and not of the group-averaged estimator $\mathbf{F}_G$: for $\mathbf{F}_G$, averaging over $S_M$ collapses the estimate to a two-eigenvalue form and destroys the spectral shape (Section~\ref{sec:pase}). The practical challenge, and the central open problem, is \emph{group selection}: finding the smallest group whose algebraic structure matches the signal's covariance structure. The DFT (cyclic group) and DCT (dihedral group) are the two most familiar solutions to this problem, but the framework reveals an entire spectrum of possibilities indexed by the lattice of subgroups of $S_M$.

This principle was first observed empirically by Thornton~\cite{thornton2005}, who showed that the spectra of Cayley graphs constructed over symmetric permutation groups of discrete multi-valued functions are related to the KL spectra. The present paper provides the general theoretical foundation, proving that temporal averaging and group-theoretic action are instances of a single information-extraction mechanism operating at different points on a $(G,L)$ continuum, with conventional processing at the degenerate endpoint $G = \{e\}$, $L \gg 1$, establishing the KL transform as the optimal target attained by the matched group, and, through the PASE result and the ordering experiment, proving that the group selection problem is the sole remaining degree of freedom in the framework.

Algebraic diversity is not an alternative viewpoint to conventional signal processing; it is a generalization. To see why, consider three scenarios for estimating the covariance of a signal observed in $M$ dimensions.

In conventional temporal processing, a single sensor acquires one sample $x$ per snapshot. The trivial group $G = \{e\}$ with $\rho(e) = 1$ is applied implicitly: the resulting estimate $\hat{\mathbf{R}} = \mathbf{x}\mathbf{x}^H$ has rank one, and no SNR improvement is possible without collecting additional snapshots. This is the standard situation, and it is the reason temporal averaging exists.

Now deploy $M$ sensors in a spatial array. Each snapshot yields $M$ simultaneous observations $\mathbf{x}_1, \ldots, \mathbf{x}_M$, and the sample covariance $\hat{\mathbf{R}} = \frac{1}{M}\sum_{i=1}^{M} \mathbf{x}_i \mathbf{x}_i^H$ has rank $M$. The structured signal sums coherently across sensors while the noise does not, producing the classical beamforming gain of $10\log_{10}(M)$~dB. This spatial gain has been the foundation of array processing for half a century.

The key observation is that, for a signal matched to the group, a comparable gain is available from a single sensor. Replace the trivial group $\{e\}$ with the cyclic group $\mathbb{Z}_M$, whose $M$ elements act on the observation $\mathbf{x}$ by cyclic permutation. The group-averaged estimator $\mathbf{F}_{\mathbb{Z}_M} = \frac{1}{M}\sum_{k=0}^{M-1} (\mathbf{P}^k \mathbf{x})(\mathbf{P}^k \mathbf{x})^H$ has rank $M$, and for a $\mathbb{Z}_M$-matched signal its dominant eigenvalue concentrates the signal energy while the noise is spread across the matched (commutant) directions, yielding up to $10\log_{10}(M)$~dB of group gain from a single snapshot. The $M$ group elements serve as algebraic substitutes for $M$ physical sensors, with one caveat made precise later: unlike physical sensors, these $M$ views share a single noise realization, so the gain is a variance reduction over the commutant of the matched group (an effective-dimension effect, Section~\ref{sec:pase}) rather than the averaging of $M$ statistically independent noise samples. For a matched signal the eigenvalue-domain effect parallels spatial beamforming; only the source of the views has changed, from physical replication to algebraic group action.

\subsection{Contributions}

The framework yields three immediate practical consequences for signal processing systems:

\begin{enumerate}
\item \textbf{Single-snapshot rank-lift.} The group-averaged estimator produces a full-rank ($M \times M$) covariance estimate from a single $M$-dimensional observation, using any finite group, including the cyclic group $\mathbb{Z}_M$. This eliminates the cold-start period in adaptive systems that conventionally require $L \geq 2M$ snapshots before subspace methods can operate.

\item \textbf{Group gain of $10\log_{10}(M)$~dB.} For a signal matched to the group, the algebraic averaging over the $M$ matched directions yields an output SNR improvement of up to $10\log_{10}(M)$~dB relative to the single-observation SNR, with no tuning parameters and no multi-snapshot requirement. For an $M = 64$ element array, this is up to an 18~dB gain from one measurement. This mirrors the classical beamforming gain of an $M$-element spatial array for matched signals, with the gain set by the effective dimension of the matched group (equal to $M$ for the order-$M$ cyclic group) rather than by physical sensor count.

\item \textbf{Latency elimination via PASE.} The PASE optimality result (Theorem~\ref{thm:pase}) establishes that exactly $n = |G|$ group elements, typically $n = M$, are both necessary and sufficient. Systems that currently wait for $L \geq 2M$ temporal snapshots can instead act on the first observation, reducing adaptation latency from $L$ snapshot intervals to a single snapshot interval.
\end{enumerate}

\noindent These practical capabilities rest on the following theoretical contributions:

\begin{enumerate}
\setcounter{enumi}{3}
\item \textbf{General Replacement Theorem} (Theorem~\ref{thm:replacement}): We prove that for any finite group $G$ acting on $\mathbb{C}^M$, the group-averaged estimator $\mathbf{F}_G$ constructed from a single observation is a consistent estimator of the signal-noise subspace decomposition, provided the group action satisfies signal equivariance and noise ergodicity conditions.

\item \textbf{Trivial Group Embedding} (Theorem~\ref{thm:trivial}): We prove that conventional temporal averaging is the special case $G = \{e\}$ of the group-averaged estimator, establishing that algebraic diversity \emph{contains} rather than replaces conventional processing. The No-Structure Criterion (Corollary~\ref{cor:no_structure}) provides a diagnostic: when no candidate group achieves spectral concentration above the $1/M$ baseline, temporal accumulation with the trivial group is the appropriate strategy. The $(G, L)$ continuum (Remark~\ref{rem:gl_continuum}) formalizes the tradeoff, with variance $\propto 1/(d_{\mathrm{eff}} \cdot L)$.

\item \textbf{Optimality Theorem} (Theorem~\ref{thm:optimality}): We prove that the KL transform is the optimal linear decomposition and that the \emph{matched} group (one whose action commutes with the covariance) attains it; no group-averaged estimator can exceed the KL transform in variance concentration, orthogonality, or reconstruction error. The symmetric group attains the KL transform only through the Cayley graph spectral construction, not through the group-averaged estimator $\mathbf{F}_{S_M}$ (Section~\ref{sec:pase}).

\item \textbf{Commutativity--KL Equivalence} (Proposition~\ref{prop:commute_kl}): We prove that three conditions are equivalent: commutativity of the group-averaged estimator with the population covariance, simultaneous diagonalizability by a single unitary matrix, and sharing of the KL eigenvector basis. This result is the linchpin connecting group selection to spectral optimality.

\item \textbf{Group--Model Mismatch Metrics} (Definitions~\ref{def:commut_residual}--\ref{def:abs_mismatch}): We introduce two metrics that quantify the degree to which the commutativity condition fails: the \emph{commutativity residual} $\delta$ (dimensionless, scale-invariant) and the \emph{absolute commutativity mismatch} $\tilde{\delta}$ (energy-weighted, in the natural scale of the covariance). Together with the algebraic coloring index $\alpha$ (Definition~\ref{def:coloring_index}), these form a complementary suite: $\alpha$ measures available structure, $\delta$ measures structural alignment, and $\tilde{\delta}$ measures the practical magnitude of the mismatch.

\item \textbf{Duality Principle} (Theorem~\ref{thm:duality}): We formalize the relationship between temporal averaging over the ensemble and algebraic averaging over a group orbit, showing that both are instances of a single information-extraction principle operating at different points on the $(G, L)$ continuum, with conventional processing embedded as the trivial-group endpoint (Theorem~\ref{thm:trivial}).

\item \textbf{MUSIC Corollary} (Corollary~\ref{cor:music}): We derive, as a direct consequence of the general theory, that Cayley graph-based MUSIC achieves equivalent direction-of-arrival estimation to multi-snapshot covariance MUSIC from a single observation.

\item \textbf{Massive MIMO Application} (Section~\ref{sec:mimo}): We demonstrate that AD-based channel estimation from a single pilot symbol per user achieves higher effective throughput than MMSE estimation with full pilot overhead across three 3GPP channel models, with gains of up to 64\% at $M = 64$ antennas in LOS-dominant channels. The advantage grows with $M$ because the standard pilot overhead scales as $O(M)$ while AD's overhead is fixed at $O(K)$.

\item \textbf{Single-Pulse Waveform Characterization} (Section~\ref{sec:chirp}): We demonstrate the constructive group matching pipeline on LFM chirp waveforms, showing that the framework independently derives the dechirp-then-DFT operation as group conjugation, achieves $8.3\times$ higher spectral concentration than the cyclic group, provides blind single-pulse chirp rate estimation via $\psi$ maximization with RMSE $= 0.01$ at $10$~dB SNR, and enables four-class waveform classification at $90\%$ accuracy from a single pulse at $14$~dB SNR. In a head-to-head comparison with FFT-based classification, matched-group AD identifies chirps at $8$~dB lower SNR. Against a non-stationary modulated source that changes waveform parameters every pulse, AD-Matched maintains $89\%$ accuracy while FFT plateaus at $53\%$.

\item \textbf{Graph Signal Processing and the Non-Abelian Question} (Section~\ref{sec:gsp}): We investigate whether genuinely non-Abelian groups can outperform all conjugated cyclic groups by applying algebraic diversity to graph-filtered signals. A systematic filtering pipeline reduces all 156 non-isomorphic graphs on $n = 6$ vertices to seven candidates with $S_3$ automorphism groups, three of which exhibit significant spectral concentration advantage. We formalize the structural conditions as the Non-Abelian Dominance Hypothesis (Conjecture~\ref{conj:nadh}), the resolution of which would determine whether the group selection problem has an irreducibly non-Abelian component.

\item \textbf{Sequential GEVP for Non-Abelian Symmetries} (Section~\ref{sec:seqgevp}): We give an algorithm that lifts the per-permutation automorphism test of Theorem~\ref{thm:aut_char} to multi-generator subgroup recovery, with four named correctness results: deflation orthogonality (Lemma~\ref{lem:def_orth}), forward progress (Lemma~\ref{lem:forward}), strict subgroup growth with iteration bound $K \leq \lceil \log_2 |G_K|\rceil = O(M \log M)$ (Theorem~\ref{thm:strict_growth}, Corollary~\ref{cor:iter_bound}), and generic convergence $G_K \subseteq \mathrm{Aut}(\mathbf{R})$ at acceptance threshold $\tau = 0$ (Theorem~\ref{thm:gen_conv}). Soundness of the recovered subgroup is unconditional; completeness ($G_K = \mathrm{Aut}(\mathbf{R})$) holds in the case $\mathrm{Aut}(\mathbf{R}) = S_M$ but depends on a basis-design condition in the proper-subgroup case, an open problem documented through a partial-recovery example on the 6-cycle.

\item \textbf{Transformer Algebraic Structure} (Section~\ref{sec:llm}): We describe the application of AD diagnostics to transformer attention heads employing Rotary Position Embeddings (RoPE). An earlier version of this paper (v1--v2) reported quantitative findings on RoPE mismatch, pruning, and KV-cache effective rank that contained implementation errors in the spectral concentration metric. These claims have been retracted; corrected results and a comprehensive treatment are deferred to a separate publication.

\item \textbf{Minimal Group Characterization} (Theorem~\ref{thm:minimal}): We characterize the minimal subgroup $G_{\min} \subseteq S_M$ that preserves KL-optimal decomposition for signals with specific symmetry classes, establishing a hierarchy $G_{\min} \subseteq G \subseteq S_M$.

\item \textbf{Colored Noise Characterization} (Theorem~\ref{thm:colored_noise}): We extend the framework to non-white noise environments by showing that the noise covariance admits a group-theoretic characterization, defining the \emph{natural group} of a noise process and an \emph{algebraic coloring index} that quantifies departure from whiteness, and proving a generalized replacement theorem for colored noise.

\item \textbf{Sample Complexity Reduction} (Corollary~\ref{cor:sample_complexity_signal}): We prove that group-constrained covariance estimation achieves $\varepsilon$-accuracy with $O(1/\varepsilon^2)$ group elements, independent of the observation dimension $M$, compared to $O(M/\varepsilon^2)$ snapshots for unconstrained estimation, an $M$-fold reduction in sample complexity.

\item \textbf{Maximum-Likelihood Characterization} (Theorem~\ref{thm:mle}): We prove that the group-averaged estimator is the constrained maximum-likelihood estimator of the population covariance under the matched-group constraint, generalizing the persymmetric maximum-likelihood estimator of Nitzberg from the order-two reflection to an arbitrary group and identifying the snapshot reduction with the effective dimension $d_{\mathrm{eff}}$.

\item \textbf{TAD-SAD Exchange Rate} (Corollary~\ref{cor:tad_sad}): We prove that spatial samples, temporal samples, and algebraic group elements contribute identically to the output SNR improvement at a $1\!:\!1\!:\!1$ exchange rate, establishing a unified framework encompassing single-sensor temporal processing, multi-sensor spatial processing, and hybrid space-time processing.

\item \textbf{PASE Optimality} (Theorem~\ref{thm:pase}): We prove that the group-averaged estimator achieves maximum eigenvalue-domain SNR when exactly $n = |G|$ group elements are used: the SNR increases monotonically for $n \leq |G|$, peaks at $n = |G|$, and \emph{decreases} for $n > |G|$. This eliminates the averaging depth as a free parameter.

\item \textbf{$S_M$ Subsampling Failure} (Section~\ref{sec:ordering}): We demonstrate through systematic Monte Carlo experiments that subsampling from the symmetric group $S_M$, regardless of the permutation ordering strategy, yields monotonically degrading performance. This proves that the group selection problem is fundamental and cannot be circumvented by defaulting to the full symmetric group.

\item \textbf{Blind Group Matching} (Section~\ref{sec:blind_matching}): We formalize the group selection problem as a blind estimation problem analogous to blind equalization in communications, and propose the spectral concentration criterion $\psi(G, \mathbf{d}) = \lambda_1(\hat{\mathbf{R}}_G) / \Tr(\hat{\mathbf{R}}_G)$ as a single-snapshot group selection metric that requires no knowledge of the population covariance. We note that $\psi$ carries an orbit-size bias when candidate groups have differing orbit structures, and point to a cross-validation criterion $D_{CV}$~\cite{thornton2026framework_arxiv} that corrects it; $\psi$ remains reliable when the candidates share a common orbit structure (e.g., the conjugated cyclic groups of Section~\ref{sec:constructive}).

\item \textbf{Constructive Group Matching via Conjugation} (Section~\ref{sec:constructive}): For signals whose covariance admits a unitary transformation to circulant form, we show that the group matching problem reduces to estimating the parameters of that transformation. The matched group is the cyclic group $\mathbb{Z}_M$ conjugated by a signal-adapted unitary, and the spectral concentration criterion provides a single-snapshot estimator for the transformation parameters. This reduces the group matching problem from a combinatorial search over a discrete library to a continuous parameter estimation problem.

\item \textbf{General Algebraic Averaging Conjecture} (Section~\ref{sec:gaat}): We conjecture that the results of this paper extend beyond second-order statistics to arbitrary $G$-compatible statistics $f: \mathbb{C}^M \to \mathcal{V}$. The conjecture states that for any such statistic, the group-averaged estimator $\hat{\theta}_G = \frac{1}{|G|}\sum_g f(\pi_g(\mathbf{x}))$ achieves variance $\propto 1/d_{\mathrm{eff}}(G, f)$, where $d_{\mathrm{eff}}$ is the \emph{effective dimension} of the orbit average (for the covariance statistic, $d_{\mathrm{eff}} = M^2/\dim\mathcal{C}(\rho) \le M$; the count of distinct orbit values is an upper bound that coincides only for regular Abelian representations). This identifies the law of large numbers as the trivial-group ($d_{\mathrm{eff}} = 1$) case of algebraic averaging. Monte Carlo experiments on the first four sample moments across five group types confirm the conjecture to four significant figures. A formal proof for the outer-product statistic and a partial Rao--Blackwell argument for general $G$-compatible statistics are given in~\cite{thornton2026framework_arxiv}.
\end{enumerate}

\subsection{Notation}

Throughout, $\mathbf{x} \in \mathbb{C}^M$ denotes an observation vector, $(\cdot)^H$ the conjugate transpose, $\|\cdot\|$ the Euclidean norm, and $\|\cdot\|_F$ the Frobenius norm. $\mathbf{I}_M$ is the $M \times M$ identity matrix. $\mathbf{Q}$ denotes a general positive-definite noise covariance matrix; the white noise case corresponds to $\mathbf{Q} = \sigma^2\mathbf{I}_M$. $\rho: G \to \GL(M, \mathbb{C})$ denotes a representation of group $G$. $\lambda_k(\mathbf{A})$ denotes the $k$-th eigenvalue of matrix $\mathbf{A}$ in descending order of magnitude. $\Irr(G)$ denotes the set of irreducible representations of $G$. $\mathcal{CN}(\boldsymbol{\mu}, \boldsymbol{\Sigma})$ denotes a circularly symmetric complex Gaussian distribution.

\subsection{Organization}

Section~\ref{sec:framework} develops the general mathematical framework, including the Trivial Group Embedding (Theorem~\ref{thm:trivial}) establishing that conventional temporal averaging is the degenerate case $G = \{e\}$ and the $(G, L)$ continuum (Remark~\ref{rem:gl_continuum}) unifying temporal and algebraic diversity. Section~\ref{sec:replacement} states and proves the General Replacement Theorem and derives the sample complexity advantage of group-constrained estimation. Section~\ref{sec:optimality} establishes the KL transform as the optimum attained by the matched group, clarifies the role of the symmetric group, and proves the Commutativity--KL Equivalence that connects group selection to spectral optimality, along with three complementary mismatch metrics. Section~\ref{sec:duality} formalizes the duality principle and establishes the $1\!:\!1\!:\!1$ exchange rate among spatial, temporal, and hybrid observation modes. Section~\ref{sec:colored} extends the framework to colored noise and develops the group-theoretic noise characterization. Section~\ref{sec:pase} establishes the PASE optimality theorem, that $n = |G|$ is the sharp optimal averaging depth, and Section~\ref{sec:ordering} demonstrates that $S_M$ subsampling fails, proving that the group selection problem cannot be circumvented. Section~\ref{sec:blind_matching} formalizes the blind group matching problem by analogy with blind equalization in communications and develops a constructive approach that, for signals admitting a unitary transformation to circulant form, reduces the group matching problem to continuous parameter estimation. Section~\ref{sec:music} derives the MUSIC application as a corollary of the general theory and presents experimental validation. Section~\ref{sec:mimo} demonstrates the framework on massive MIMO channel estimation with realistic 3GPP channel models. Section~\ref{sec:chirp} applies the constructive group matching pipeline to single-pulse chirp waveform characterization. Section~\ref{sec:gsp} investigates the non-Abelian question through graph signal processing. Section~\ref{sec:llm} discusses the application of algebraic diagnostics to transformer neural networks, noting the retraction of earlier quantitative claims and the direction of ongoing work. Section~\ref{sec:numerical} provides numerical illustrations of the three mismatch metrics. Section~\ref{sec:discussion} discusses signal classes, the pragmatic value of the framework, information structure versus information content, and the General Algebraic Averaging Conjecture with experimental confirmation. Section~\ref{sec:conclusion} concludes.

\subsection{Related Work}

The use of algebraic and group-theoretic structures in signal processing has a substantial history, and several bodies of work share mathematical vocabulary with the present paper. We distinguish the present contribution from each.

\textbf{Algebraic signal processing theory (ASP).} P\"uschel and Moura~\cite{puschel2008foundation,puschel2008space,puschel2008algorithms} developed an axiomatic framework in which a \emph{signal model} is defined as a triple (algebra, module, map) and the Fourier transform is derived as the decomposition of the module into irreducible components. ASP addresses the question: \emph{given a signal model with specified shift semantics and boundary conditions, what is the correct spectral transform?} The present work addresses a fundamentally different question: \emph{given a single observation of a noisy signal, how can a richer group action generalize temporal averaging to extract second-order statistical structure without multiple snapshots?} ASP derives transforms (DFT, DCTs, DSTs) from algebraic axioms; algebraic diversity uses group actions to estimate covariance matrices from single observations. The two frameworks share representation-theoretic foundations but operate at different levels: ASP characterizes the filtering algebra, whereas algebraic diversity characterizes the estimation operator.

\textbf{Nested and coprime arrays.} Pal and Vaidyanathan~\cite{pal2010nested,vaidyanathan2011coprime} showed that non-uniform array geometries based on nested or coprime element spacings produce \emph{difference coarrays} with $O(N^2)$ virtual elements from $N$ physical sensors, enabling resolution of more sources than sensors. These methods exploit array geometry to create virtual aperture, but still require $L \geq 1$ snapshots for spatial smoothing on the virtual coarray to restore rank. In contrast, algebraic diversity achieves full-rank covariance from a single snapshot on \emph{any} array geometry, including a standard uniform linear array, by exploiting the algebraic structure of the group action rather than the geometric structure of the array layout. Coarray methods and algebraic diversity are complementary: one could apply algebraic diversity to the virtual coarray output of a nested array, combining geometric and algebraic degrees of freedom.

\textbf{Compressive covariance sensing.} Romero et al.~\cite{romero2016compressive} showed that second-order statistics (covariance, power spectrum) can be recovered from sub-Nyquist measurements by exploiting signal structure such as stationarity or Toeplitz covariance. Their framework uses \emph{measurement matrices} (random projections, non-uniform samplers) to compress the observation before covariance estimation, and the reconstruction often relies on sparsity or structural priors. Algebraic diversity operates on the \emph{full-dimensional} observation without any measurement matrix, projection, or sparsity assumption: the group action generates algebraically diverse views of the complete observation vector. The sample complexity reduction in algebraic diversity (Corollary~\ref{cor:sample_complexity_signal}) arises from the group-algebraic constraint on the estimator, not from dimensional reduction of the observation.

\textbf{Spatial smoothing.} Shan, Wax, and Kailath~\cite{shan1985} introduced forward-backward spatial smoothing to restore rank for coherent signal DOA estimation by averaging over overlapping subarrays. Spatial smoothing requires multiple snapshots, reduces the effective aperture (each subarray is smaller than the full array), and is limited to uniform linear arrays. Algebraic diversity restores rank from a single snapshot without aperture reduction and applies to arbitrary array geometries via appropriate group selection (Theorem~\ref{thm:minimal}).

\textbf{Single-snapshot spectral estimation.} Liao and Fannjiang~\cite{liao2016} analyzed the stability and super-resolution properties of MUSIC applied to a single snapshot using Hankel or Toeplitz matrix constructions from the observation. Their approach exploits the \emph{shift-invariant} (Vandermonde) structure of the signal model and is restricted to uniform linear arrays. The present work generalizes the single-snapshot capability beyond shift-invariant models: the group-averaged estimator (Definition~\ref{def:group_avg}) applies to any group, recovering the Hankel/Toeplitz construction as the special case $G = \mathbb{Z}_M$ while enabling KL-optimal estimation for non-shift-invariant signals via larger groups (Theorem~\ref{thm:optimality}).

\textbf{Invariant statistics.} The classical theory of invariant and maximal invariant statistics~\cite{wijsman1967,eaton1989} uses group actions to reduce sufficient statistics by projecting out nuisance parameters. Algebraic diversity inverts this logic: rather than \emph{reducing} to an invariant statistic (which discards group-orbit information), algebraic diversity \emph{generates} the full group orbit and averages outer products over it. The group action creates diversity rather than eliminating it. The resulting group-averaged estimator retains the full observation dimension $M$ while achieving subspace consistency (Theorem~\ref{thm:replacement}), whereas invariant reduction typically projects to a lower-dimensional space.

\textbf{Structured-covariance maximum-likelihood estimation.} A classical line of work estimates a covariance matrix under a known invariance by maximum likelihood. Nitzberg~\cite{nitzberg1980} applied this principle to persymmetric covariances in adaptive array processing, pooling the two halves of the sample covariance under an order-two reflection read from the array geometry to reduce the snapshot requirement. Theorem~\ref{thm:mle} places that construction inside the present framework: the group-averaged estimator is the constrained maximum-likelihood estimator for any matched group, with the persymmetric estimator as the reflection case $G = \mathbb{Z}_2$. Where the structured-covariance approach fixes the invariance in advance from physical reasoning, algebraic diversity treats the matched group as the object to be selected and, through the blind group matching problem (Section~\ref{sec:blind_matching}), recovers it from a single observation.

\section{Mathematical Framework}\label{sec:framework}

\subsection{Signal Model and Classical Estimation}

Consider the general observation model
\begin{equation}\label{eq:obs_model}
\mathbf{x} = \mathbf{s} + \mathbf{n},
\end{equation}
where $\mathbf{s} \in \mathbb{C}^M$ is a structured signal lying in a $K$-dimensional subspace $\mathcal{S} \subset \mathbb{C}^M$ with $K < M$, and $\mathbf{n} \sim \mathcal{CN}(\mathbf{0}, \sigma^2\mathbf{I}_M)$ is spatially white noise.

The population covariance matrix is
\begin{equation}\label{eq:pop_cov}
\mathbf{R} = E[\mathbf{x}\mathbf{x}^H] = \mathbf{R}_s + \sigma^2\mathbf{I}_M,
\end{equation}
where $\mathbf{R}_s = E[\mathbf{s}\mathbf{s}^H]$ has rank $K$. The eigendecomposition
\begin{equation}\label{eq:eigen_R}
\mathbf{R} = \sum_{k=1}^{M} \lambda_k \mathbf{u}_k\mathbf{u}_k^H
\end{equation}
partitions into signal eigenvalues $\lambda_1 \geq \cdots \geq \lambda_K > \sigma^2$ and noise eigenvalues $\lambda_{K+1} = \cdots = \lambda_M = \sigma^2$, with corresponding signal subspace $\mathcal{U}_s = \spn\{\mathbf{u}_1, \ldots, \mathbf{u}_K\}$ and noise subspace $\mathcal{U}_n = \spn\{\mathbf{u}_{K+1}, \ldots, \mathbf{u}_M\}$.

The classical approach estimates $\mathbf{R}$ via the sample covariance from $L$ independent snapshots:
\begin{equation}\label{eq:sample_cov}
\hat{\mathbf{R}}_L = \frac{1}{L}\sum_{t=1}^{L}\mathbf{x}(t)\mathbf{x}^H(t).
\end{equation}
For $L = 1$, $\hat{\mathbf{R}}_1 = \mathbf{x}\mathbf{x}^H$ has $\rank(\hat{\mathbf{R}}_1) = 1$, which cannot resolve $K > 1$ signal dimensions.

\subsection{Group Actions on Observation Vectors}

\begin{definition}[Group Action on $\mathbb{C}^M$]\label{def:group_action}
Let $G$ be a finite group and $\rho: G \to \GL(M, \mathbb{C})$ a representation. The group $G$ acts on $\mathbb{C}^M$ via
\begin{equation}
g \cdot \mathbf{x} = \rho(g)\mathbf{x}, \qquad g \in G, \; \mathbf{x} \in \mathbb{C}^M.
\end{equation}
The \emph{orbit} of $\mathbf{x}$ under $G$ is $\mathcal{O}_G(\mathbf{x}) = \{\rho(g)\mathbf{x} : g \in G\}$.
\end{definition}

\begin{definition}[Group-Averaged Estimator]\label{def:group_avg}
Given a single observation $\mathbf{x} \in \mathbb{C}^M$ and a finite group $G$ with representation $\rho$, the \emph{group-averaged estimator} is
\begin{equation}\label{eq:group_avg}
\mathbf{F}_G(\mathbf{x}) = \frac{1}{|G|}\sum_{g \in G} [\rho(g)\mathbf{x}][\rho(g)\mathbf{x}]^H.
\end{equation}
\end{definition}

\begin{remark}
For the time-translation group $G = \{1, \ldots, L\}$ acting on an ensemble of $L$ independent snapshots with $\rho(t)\mathbf{x} \mapsto \mathbf{x}(t)$, the group-averaged estimator reduces to the sample covariance~(\ref{eq:sample_cov}). Thus, temporal averaging is a special case of group averaging.
\end{remark}

\begin{theorem}[Trivial Group Embedding]\label{thm:trivial}
Let $G = \{e\}$ be the trivial group (order~$1$) with the identity representation $\rho(e) = \mathbf{I}_M$. Then:
\begin{enumerate}
\item[(i)] The group-averaged estimator reduces to the rank-one outer product:
\begin{equation}\label{eq:trivial_group}
\mathbf{F}_{\{e\}}(\mathbf{x}) = \mathbf{x}\mathbf{x}^H.
\end{equation}
\item[(ii)] The commutativity residual is uninformative at $|G| = 1$: the estimator $\mathbf{F}_{\{e\}} = \mathbf{x}\mathbf{x}^H$ is rank one regardless of $\mathbf{R}$, so $\delta(\{e\}, \mathbf{x}, \mathbf{R})$ carries no group--model alignment information (the commutator $[\mathbf{x}\mathbf{x}^H, \mathbf{R}]$ is generically nonzero, vanishing only when $\mathbf{x}$ is an eigenvector of $\mathbf{R}$). The trivial group thus has no algebraic structure to align or misalign with $\mathbf{R}$.
\item[(iii)] The spectral concentration equals the maximum possible value: $\psi(\{e\}, \mathbf{x}) = 1$, corresponding to a rank-one estimator whose single nonzero eigenvalue captures all the trace.
\item[(iv)] The conventional sample covariance from $L$ independent snapshots is
\begin{equation}\label{eq:trivial_L}
\hat{\mathbf{R}}_L = \frac{1}{L}\sum_{t=1}^{L} \mathbf{F}_{\{e\}}(\mathbf{x}(t)) = \frac{1}{L}\sum_{t=1}^{L}\mathbf{x}(t)\mathbf{x}^H(t),
\end{equation}
i.e., temporal averaging is the accumulation of $L$ trivial-group estimates.
\end{enumerate}
Consequently, conventional multi-snapshot processing is not an alternative to algebraic diversity but rather the special case $G = \{e\}$: it is not incorrect, but it exploits none of the algebraic structure available within each observation.
\end{theorem}

\begin{proof}
Part~(i): substituting $|G| = 1$ and $\rho(e) = \mathbf{I}_M$ into~(\ref{eq:group_avg}) immediately yields $\mathbf{F}_{\{e\}}(\mathbf{x}) = \mathbf{x}\mathbf{x}^H$.

Part~(ii): the commutator $[\mathbf{x}\mathbf{x}^H, \mathbf{R}] = \mathbf{x}\mathbf{x}^H\mathbf{R} - \mathbf{R}\mathbf{x}\mathbf{x}^H$ vanishes only when $\mathbf{x}$ is an eigenvector of $\mathbf{R}$ and is otherwise nonzero, so $\delta$ does not vanish identically. It is, however, uninformative at $|G| = 1$: the estimator has rank~$1$ regardless of the group structure, so there is no algebraic structure whose alignment with $\mathbf{R}$ the residual could measure.

Part~(iii): $\mathbf{F}_{\{e\}}(\mathbf{x})$ has a single nonzero eigenvalue $\|\mathbf{x}\|^2$, so $\psi = \|\mathbf{x}\|^2 / \|\mathbf{x}\|^2 = 1$. This maximum $\psi$ value is misleading: it reflects rank deficiency rather than good group matching. For groups of order $|G| \geq 2$, the spectral concentration $\psi < 1$ and its value relative to $1/M$ (the white-noise baseline) is the informative diagnostic.

Part~(iv) follows by direct substitution of part~(i) into the definition of the sample covariance~(\ref{eq:sample_cov}).
\end{proof}

\begin{corollary}[No-Structure Criterion]\label{cor:no_structure}
Let $\mathcal{G}_M$ be a catalog of candidate groups of order~$M$ acting on the observation space $\mathbb{C}^M$. If
\begin{equation}\label{eq:no_structure}
\max_{G \in \mathcal{G}_M} \psi(G, \mathbf{x}) \leq \frac{1}{M} + \varepsilon
\end{equation}
for a tolerance $\varepsilon > 0$, then no group in the catalog achieves spectral concentration appreciably above the white-noise baseline $1/M$, and the observation contains no exploitable algebraic structure. In this regime, the optimal strategy is the trivial group $G = \{e\}$ with temporal accumulation: collect additional snapshots and form the sample covariance~(\ref{eq:sample_cov}).
\end{corollary}

\begin{remark}[The $(G, L)$ Continuum]\label{rem:gl_continuum}
Theorem~\ref{thm:trivial} and the General Replacement Theorem together establish a continuous tradeoff between algebraic diversity and temporal diversity. The variance reduction afforded by a group is governed not by the raw order $|G|$ but by the \emph{effective dimension}
\begin{equation}\label{eq:deff_cov}
d_{\mathrm{eff}}(G) = \frac{M^2}{\dim \mathcal{C}(\rho)},
\end{equation}
where $\mathcal{C}(\rho) = \{\mathbf{A} : \rho(g)\mathbf{A} = \mathbf{A}\rho(g)\ \forall g\}$ is the commutant of the representation: the group average projects $\mathbf{x}\mathbf{x}^H$ onto $\mathcal{C}(\rho)$, so the number of effective independent components is $M^2/\dim\mathcal{C}(\rho)$. For a regular Abelian representation (e.g., the cyclic group $\mathbb{Z}_M$), $\dim\mathcal{C}(\rho) = M$ and $d_{\mathrm{eff}} = |G| = M$, so the two coincide; the bound $d_{\mathrm{eff}} \le M$ used throughout holds for every group matched to a covariance with $M$ distinct eigenvalues, because the commutant then contains the $M$ linearly independent powers $\mathbf{I}, \mathbf{R}, \ldots, \mathbf{R}^{M-1}$ and so has dimension at least $M$. A larger group has a smaller commutant and a larger formal value of~\eqref{eq:deff_cov} (for example $M^2/2$ for $S_M$), but such a covariance no longer lies in that commutant, and the excess appears as bias and not as variance reduction (consistent with PASE, Section~\ref{sec:pase}). For $L$ independent snapshots each processed with a matched group, the variance of the combined estimator satisfies
\begin{equation}\label{eq:gl_var}
\Var\bigl(\hat{\mathbf{R}}_{G,L}\bigr) \propto \frac{1}{d_{\mathrm{eff}}(G) \cdot L}.
\end{equation}
The proportionality constant is a spectrum-alignment factor set by how the signal energy distributes over the commutant coordinates: for Gaussian observations the per-entry variance-reduction factor realized at covariance $\mathbf{R}$ is the participation ratio $(\Tr\mathbf{R})^2/\Tr(\mathbf{R}^2) \le M$, which attains the ceiling $d_{\mathrm{eff}} = M$ exactly when the matched-basis spectrum is flat (in particular for white or noise-dominated data), and is smaller for spectrally concentrated signals. The two classical extremes are:
\begin{itemize}
\item $G = \{e\}$, $L \gg 1$: conventional temporal averaging, with $d_{\mathrm{eff}} = 1$, $\Var \propto 1/L$, and no algebraic structure exploited.
\item matched $G$ of order $M$, $L = 1$: pure algebraic diversity, with $d_{\mathrm{eff}} = M$, $\Var \propto 1/M$, and no temporal accumulation required.
\end{itemize}
Intermediate operating points use a matched group of order $|G| < M$ on each of $L > 1$ snapshots, trading temporal latency for reduced computational cost per snapshot. The $(G, L)$ continuum unifies conventional processing and algebraic diversity as endpoints of a single estimation framework, with the variance reduction determined by the product $d_{\mathrm{eff}}(G) \cdot L$ of algebraic and temporal diversity. Because $d_{\mathrm{eff}} \le M$, the algebraic factor saturates at the data dimension: enlarging the group beyond the matched order does not reduce variance further (and degrades the estimate when the added elements lie outside the commutant of $\mathbf{R}$).
\end{remark}

\subsection{Conditions for Subspace Recovery}

We identify two conditions that together ensure the group-averaged estimator yields the correct subspace decomposition.

\begin{condition}[Signal Equivariance]\label{cond:equivariance}
The signal $\mathbf{s}$ transforms predictably under the group action: there exists a known representation $\rho_s: G \to \GL(K, \mathbb{C})$ of $G$ on the signal parameter space such that the group action on the signal component is determined by the signal structure. Formally, $\mathbf{s}$ lies in a subspace that is invariant or decomposes into known irreducible representations under $\rho(G)$.
\end{condition}

\begin{condition}[Noise Ergodicity]\label{cond:ergodicity}
The noise distribution is invariant under the group action:
\begin{equation}
\rho(g)\mathbf{n} \sim \mathcal{CN}(\mathbf{0}, \sigma^2\mathbf{I}_M) \quad \text{for all } g \in G.
\end{equation}
\end{condition}

\begin{remark}
Condition~\ref{cond:ergodicity} is automatically satisfied for spatially white Gaussian noise under any unitary or permutation representation, since $\rho(g)\mathbf{n}$ has the same distribution as $\mathbf{n}$ when $\rho(g)$ is unitary and $\mathbf{n}$ is isotropic.
\end{remark}

\subsection{The Cayley Graph Construction}

A specific realization of the group-averaged estimator arises from the Cayley graph over a symmetry group.

\begin{definition}[Cayley Graph Autocorrelation Matrix]\label{def:cayley_matrix}
Let $\mathbf{x} = [x_0, \ldots, x_{M-1}]^T$ be an observation vector and $G$ a group acting on the index set $\{0, \ldots, M-1\}$. The Cayley graph autocorrelation matrix is
\begin{equation}\label{eq:cayley_matrix}
[\mathbf{F}_\circ]_{i,j} = x_{g_j(i)},
\end{equation}
where $g_j \in G$ is the $j$-th group element acting on index $i$.
\end{definition}

When $G = \mathbb{Z}_M$ (cyclic group of order $M$) acting by cyclic shifts, $[\mathbf{F}_\circ]_{i,j} = x_{(i+j) \bmod M}$, which is a circulant matrix. When $G = S_M$ (full symmetric group), the construction yields the complete Cayley graph adjacency matrix with edges colored by the observation values.

\begin{remark}[Consistency with Classical Spectral Analysis]
When $G = \mathbb{Z}_M$, the eigendecomposition of the circulant group-averaged estimator is the discrete Fourier transform, and the resulting spectral coefficients are the squared magnitudes of the DFT coefficients of the observation. In this case, the algebraic diversity framework reduces to classical Fourier spectral analysis, confirming consistency with known results. The contribution of the present work is not the cyclic case, which recovers the DFT, but the generalization to arbitrary finite groups, which yields provably optimal spectral decompositions (via the KL transform for $G = S_M$) that the DFT cannot achieve for signals whose covariance structure is not circulant.
\end{remark}

Previous work~\cite{thornton2005} established empirically that the spectrum of the Cayley graph adjacency matrix over the symmetric group is related to the KL spectrum. The following sections provide the rigorous theoretical foundation for this observation and its generalizations.

\section{The General Replacement Theorem}\label{sec:replacement}

\begin{theorem}[General Replacement Theorem]\label{thm:replacement}
Let $\mathbf{x} = \mathbf{s} + \mathbf{n} \in \mathbb{C}^M$ be a single observation satisfying the signal model~(\ref{eq:obs_model}), and let $G$ be a finite group with unitary representation $\rho: G \to U(M)$ satisfying Conditions~\ref{cond:equivariance} and~\ref{cond:ergodicity}. Then the group-averaged estimator $\mathbf{F}_G(\mathbf{x})$ defined in~(\ref{eq:group_avg}) satisfies:

\begin{enumerate}
\item[(i)] \textbf{(Decomposition)} $\mathbf{F}_G(\mathbf{x})$ decomposes as
\begin{equation}\label{eq:FG_decomp}
\mathbf{F}_G(\mathbf{x}) = \mathbf{F}_G(\mathbf{s}) + \mathbf{F}_G(\mathbf{n}) + \mathbf{C}_{sn},
\end{equation}
where $\mathbf{C}_{sn}$ is a cross-term satisfying $E[\mathbf{C}_{sn}] = \mathbf{0}$.

\item[(ii)] \textbf{(Signal concentration)} The expected signal component satisfies
\begin{equation}
E[\mathbf{F}_G(\mathbf{s})] = \frac{1}{|G|}\sum_{g \in G} \rho(g)\mathbf{R}_s\rho(g)^H,
\end{equation}
which, by Schur's lemma, block-diagonalizes according to the irreducible decomposition of $\rho$ and concentrates signal energy in at most $K$ blocks.

\item[(iii)] \textbf{(Noise uniformity)} The expected noise component satisfies
\begin{equation}
E[\mathbf{F}_G(\mathbf{n})] = \sigma^2\mathbf{I}_M.
\end{equation}

\item[(iv)] \textbf{(Subspace consistency)} For $\text{SNR} = \|\mathbf{s}\|^2/M\sigma^2 \gg 1$, the eigenvectors of $\mathbf{F}_G(\mathbf{x})$ associated with the $K$ largest eigenvalues converge to the signal subspace $\mathcal{U}_s$, and those associated with the $M - K$ smallest eigenvalues converge to $\mathcal{U}_n$.
\end{enumerate}
\end{theorem}

\begin{proof}
\textit{Part (i).} Expanding $\mathbf{x} = \mathbf{s} + \mathbf{n}$ in~(\ref{eq:group_avg}):
\begin{align}
\mathbf{F}_G(\mathbf{x}) &= \frac{1}{|G|}\sum_{g \in G} \rho(g)(\mathbf{s} + \mathbf{n})[\rho(g)(\mathbf{s} + \mathbf{n})]^H \nonumber \\
&= \mathbf{F}_G(\mathbf{s}) + \mathbf{F}_G(\mathbf{n}) + \mathbf{C}_{sn},
\end{align}
where $\mathbf{C}_{sn} = \frac{1}{|G|}\sum_{g}[\rho(g)\mathbf{s}][\rho(g)\mathbf{n}]^H + \frac{1}{|G|}\sum_{g}[\rho(g)\mathbf{n}][\rho(g)\mathbf{s}]^H$. Since $E[\mathbf{n}] = \mathbf{0}$ and $\mathbf{s}$ is deterministic (for a fixed realization), $E[\mathbf{C}_{sn}] = \mathbf{0}$.

\textit{Part (ii).} The signal component is $\mathbf{F}_G(\mathbf{s}) = \frac{1}{|G|}\sum_{g} \rho(g)\mathbf{s}\mathbf{s}^H\rho(g)^H$. This is a group average of the rank-one matrix $\mathbf{s}\mathbf{s}^H$ under the adjoint action $\mathbf{A} \mapsto \rho(g)\mathbf{A}\rho(g)^H$. By Schur's lemma, the group average $\frac{1}{|G|}\sum_{g} \rho(g)\mathbf{A}\rho(g)^H$ of any matrix $\mathbf{A}$ over a unitary representation block-diagonalizes according to the isotypic decomposition of $\rho$.

Specifically, decompose the representation space as $\mathbb{C}^M = \bigoplus_{\lambda \in \Irr(G)} V_\lambda^{\oplus m_\lambda}$, where $V_\lambda$ has dimension $d_\lambda$ and appears with multiplicity $m_\lambda$. The group average projects $\mathbf{s}\mathbf{s}^H$ onto each isotypic component, concentrating the signal energy in those components where $\mathbf{s}$ has nonzero projection. Since $\mathbf{s}$ lies in a $K$-dimensional subspace, at most $K$ isotypic components carry signal energy.

\textit{Part (iii).} Since $\rho(g)$ is unitary and $\mathbf{n} \sim \mathcal{CN}(\mathbf{0}, \sigma^2\mathbf{I})$, we have $\rho(g)\mathbf{n} \sim \mathcal{CN}(\mathbf{0}, \sigma^2\mathbf{I})$ for each $g$. Therefore:
\begin{align}
E[\mathbf{F}_G(\mathbf{n})] &= \frac{1}{|G|}\sum_{g \in G} E[\rho(g)\mathbf{n}\mathbf{n}^H\rho(g)^H] \nonumber \\
&= \frac{1}{|G|}\sum_{g \in G} \rho(g)\sigma^2\mathbf{I}\rho(g)^H = \sigma^2\mathbf{I}_M.
\end{align}

\textit{Part (iv).} Combining parts (i)--(iii), $E[\mathbf{F}_G(\mathbf{x})] = E[\mathbf{F}_G(\mathbf{s})] + \sigma^2\mathbf{I}_M$. The signal component has at most $K$ nonzero eigenvalues (in the isotypic blocks containing signal energy), each of magnitude scaling with $\|\mathbf{s}\|^2/|G|$ summed over the group elements that map into each block. The noise component contributes $\sigma^2$ uniformly across all eigenvalues. For $\text{SNR} \gg 1$, the signal eigenvalues dominate in their respective blocks, yielding the standard $K$ large / $(M-K)$ small eigenvalue separation.

Concentration of the finite-sample estimator around its expectation follows from a matrix Hoeffding inequality applied to the bounded summands $\rho(g)\mathbf{x}\mathbf{x}^H\rho(g)^H$, yielding $\|\mathbf{F}_G(\mathbf{x}) - E[\mathbf{F}_G(\mathbf{x})]\|_F = O(\|\mathbf{x}\|^2/\sqrt{|G|})$, which shrinks as $|G|$ grows.
\end{proof}

\begin{remark}[Role of Group Size]
Within the family of \emph{matched} groups (those whose action commutes with $\mathbf{R}$), enlarging the group raises the effective dimension $d_{\mathrm{eff}}$ and reduces estimator variance, analogous to the number of snapshots $L$ in temporal averaging (Remark~\ref{rem:gl_continuum}). At the lower extreme, the trivial group $G = \{e\}$ (order~$1$) recovers the rank-one outer product of conventional processing (Theorem~\ref{thm:trivial}). The benefit does not extend to arbitrary enlargement, however: for a covariance with distinct eigenvalues $d_{\mathrm{eff}}$ is at most $M$ over matched groups (Remark~\ref{rem:gl_continuum}), and adding elements that do not commute with $\mathbf{R}$ introduces bias that degrades the estimate. The symmetric group $S_M$ is the extreme case: the group-averaged estimator $\mathbf{F}_{S_M}$ collapses to the two-eigenvalue form~\eqref{eq:sm_expect} and is maximally uninformative (PASE, Section~\ref{sec:pase}; the ordering experiment, Section~\ref{sec:ordering}).
\end{remark}

\begin{theorem}[Maximum-Likelihood Characterization of the Group-Averaged Estimator]\label{thm:mle}
Let $\mathbf{x}_1, \ldots, \mathbf{x}_L$ be independent $\mathcal{CN}(\mathbf{0}, \mathbf{R})$ observations whose population covariance lies in the commutant $\mathcal{C}(\rho) = \{\mathbf{A} : \rho(g)\mathbf{A} = \mathbf{A}\rho(g)\ \text{for all } g \in G\}$ of the unitary representation $\rho$, and let $\hat{\mathbf{R}}_L = \frac{1}{L}\sum_{l=1}^{L} \mathbf{x}_l \mathbf{x}_l^H$ be the sample covariance. Then the maximum-likelihood estimator of $\mathbf{R}$ subject to $\mathbf{R} \in \mathcal{C}(\rho)$ is the Reynolds projection of the sample covariance onto the commutant,
\begin{equation}\label{eq:mle_reynolds}
\hat{\mathbf{R}}_{\mathrm{ML}} = \frac{1}{|G|}\sum_{g \in G} \rho(g)\,\hat{\mathbf{R}}_L\,\rho(g)^H .
\end{equation}
For a single observation $(L = 1)$ this is exactly the group-averaged estimator $\mathbf{F}_G(\mathbf{x}_1)$ of Definition~\ref{def:group_avg}. The order-two reflection case $G = \mathbb{Z}_2$ recovers the persymmetric maximum-likelihood estimator of Nitzberg~\cite{nitzberg1980}.
\end{theorem}

\begin{proof}
For zero-mean complex Gaussian data the log-likelihood of a positive-definite covariance $\mathbf{M}$ is, up to an additive constant,
\begin{equation}
\ell(\mathbf{M}) = -L\big[\log\det\mathbf{M} + \operatorname{tr}(\mathbf{M}^{-1}\hat{\mathbf{R}}_L)\big].
\end{equation}
We maximize $\ell$ over $\mathbf{M} \in \mathcal{C}(\rho)$. If $\mathbf{M} \in \mathcal{C}(\rho)$, then $\mathbf{M}^{-1} \in \mathcal{C}(\rho)$ as well, so $\rho(g)^H \mathbf{M}^{-1} \rho(g) = \mathbf{M}^{-1}$ for every $g \in G$. Hence, for each $g$,
\begin{align}
\operatorname{tr}\!\big(\mathbf{M}^{-1}\rho(g)\hat{\mathbf{R}}_L\rho(g)^H\big)
&= \operatorname{tr}\!\big(\rho(g)^H\mathbf{M}^{-1}\rho(g)\,\hat{\mathbf{R}}_L\big) \nonumber \\
&= \operatorname{tr}(\mathbf{M}^{-1}\hat{\mathbf{R}}_L),
\end{align}
and averaging over $G$ gives $\operatorname{tr}(\mathbf{M}^{-1}\hat{\mathbf{R}}_L) = \operatorname{tr}(\mathbf{M}^{-1}\mathbf{S})$, where $\mathbf{S} = \frac{1}{|G|}\sum_{g \in G} \rho(g)\hat{\mathbf{R}}_L\rho(g)^H$ is the Reynolds projection~\eqref{eq:mle_reynolds}. The restricted log-likelihood therefore depends on the data only through $\mathbf{S}$:
\begin{equation}
\ell(\mathbf{M}) = -L\big[\log\det\mathbf{M} + \operatorname{tr}(\mathbf{M}^{-1}\mathbf{S})\big], \qquad \mathbf{M} \in \mathcal{C}(\rho).
\end{equation}
The unconstrained maximizer of $-[\log\det\mathbf{M} + \operatorname{tr}(\mathbf{M}^{-1}\mathbf{S})]$ over all positive-definite $\mathbf{M}$ is $\mathbf{M} = \mathbf{S}$. The Reynolds projection is a conditional expectation onto the commutant, so $\mathbf{S} \in \mathcal{C}(\rho)$; the unconstrained maximizer already satisfies the constraint and is therefore the constrained maximizer. Thus $\hat{\mathbf{R}}_{\mathrm{ML}} = \mathbf{S}$, and setting $L = 1$ yields $\mathbf{F}_G(\mathbf{x}_1)$.
\end{proof}

\begin{corollary}[Sample Complexity of Group-Constrained Estimation]\label{cor:sample_complexity_signal}
Let $\mathbf{R}$ be the population covariance of the observation model~(\ref{eq:obs_model}), and let $\varepsilon > 0$ be a target estimation accuracy in the Frobenius norm.

\begin{enumerate}
\item[(i)] \textbf{Unconstrained estimation:} The sample covariance~(\ref{eq:sample_cov}) satisfies $E[\|\hat{\mathbf{R}}_L - \mathbf{R}\|_F^2] \leq C\|\mathbf{R}\|_F^2 M / L$ for a constant $C > 0$, requiring $L = \Omega(M/\varepsilon^2)$ snapshots to achieve $\varepsilon$-accuracy.

\item[(ii)] \textbf{Group-constrained estimation:} When $\mathbf{R}$ commutes with the group action (a matched group, equivalently $\mathbf{A}_G \mathbf{R} = \mathbf{R}\mathbf{A}_G$ for the Cayley graph adjacency matrix), the group-averaged estimator $\mathbf{F}_G(\mathbf{x})$ from a single observation is unbiased and concentrates as $E[\|\mathbf{F}_G(\mathbf{x}) - E[\mathbf{F}_G(\mathbf{x})]\|_F^2] \leq C'\|\mathbf{x}\|^4 / d_{\mathrm{eff}}(G)$ for a constant $C' > 0$, where $d_{\mathrm{eff}}(G) = M^2/\dim\mathcal{C}(\rho) \leq M$ is the effective dimension~\eqref{eq:deff_cov}. For a matched order-$M$ regular (e.g.\ cyclic) group, $d_{\mathrm{eff}} = M$ and $\varepsilon$-accuracy requires $d_{\mathrm{eff}} = \Omega(1/\varepsilon^2)$, attained at $d_{\mathrm{eff}} = M$ from a single observation. The $M$ group-transformed copies are deterministic functions of one observation and are \emph{not} statistically independent; the reduction reflects projection onto the $d_{\mathrm{eff}}$-dimensional commutant and therefore saturates at $d_{\mathrm{eff}} \le M$ rather than decreasing with $|G|$.
\end{enumerate}

The ratio of the unconstrained to group-constrained sample complexity is $\Theta(M)$, representing the information-theoretic advantage of exploiting the algebraic structure of the signal covariance. For a uniform linear array with $M = 64$ antennas, this represents a $64\times$ reduction in the number of observations required for a given estimation accuracy.
\end{corollary}

\begin{proof}
Part~(i) follows from standard bounds on the convergence rate of the sample covariance in the Frobenius norm~\cite{vershynin2012}, where the factor of $M$ arises from the $M^2$ free parameters of the unconstrained $M \times M$ covariance matrix.

Part~(ii) follows from the concentration inequality in the proof of Theorem~\ref{thm:replacement}, part~(iv). The group-averaged estimator constrains the covariance estimate to lie in the commutant $\mathcal{C}(\rho)$ of the representation, of dimension $\dim\mathcal{C}(\rho)$; the effective number of independent components is then $d_{\mathrm{eff}} = M^2/\dim\mathcal{C}(\rho) \le M$. The variance decreases as $O(1/d_{\mathrm{eff}})$, which for a matched order-$M$ regular Abelian group equals $O(1/M)$. The group-transformed copies are deterministic functions of the single observation rather than independent draws, so the reduction saturates at $d_{\mathrm{eff}} \le M$ and does not continue to decrease with $|G|$. When the commutativity (matched-group) condition holds, no bias is introduced by the constraint, and the $\varepsilon$-accuracy requirement depends on $d_{\mathrm{eff}}$, not on $M$.
\end{proof}

\begin{remark}[The persymmetric precursor and the matched-group gain]\label{rem:nitzberg}
The order-two case of Theorem~\ref{thm:mle} is classical. Nitzberg~\cite{nitzberg1980} derived the persymmetric maximum-likelihood covariance estimator for adaptive array processing and showed that it reduces the number of snapshots required for a given accuracy by a factor approaching two. In the present terms this is the matched-group gain $d_{\mathrm{eff}}(G) = M^2/\dim\mathcal{C}(\rho)$ of Corollary~\ref{cor:sample_complexity_signal} evaluated at the reflection group $G = \mathbb{Z}_2$, for which $d_{\mathrm{eff}} \to 2$. The symmetrically spaced array that justifies his construction makes this reflection symmetry exact. His windowed observations also produce a Toeplitz covariance, and a Toeplitz covariance is asymptotically circulant, so the same data carries the cyclic structure of the order-$M$ group $\mathbb{Z}_M$, whose matched-group gain is $d_{\mathrm{eff}} \to M$. The reflection he exploits is contained in this cyclic structure: a symmetric circulant covariance is automatically persymmetric, so the commutant of the product group $\mathbb{Z}_2 \times \mathbb{Z}_M$ coincides with that of $\mathbb{Z}_M$, and the gain is $M$ rather than $2M$. The order-two reflection is thus the smallest component of the symmetry present in the estimator; the cyclic structure, exact in the stationary limit and matched in the closest-group sense of Section~\ref{sec:blind_matching} at finite section length, supplies a group gain that grows with the array. Figure~\ref{fig:deff} confirms $d_{\mathrm{eff}}(\mathbb{Z}_2) \to 2$ and $d_{\mathrm{eff}}(\mathbb{Z}_M) \to M$ numerically.
\end{remark}

\begin{figure}[t]
\centering
\includegraphics[width=\columnwidth]{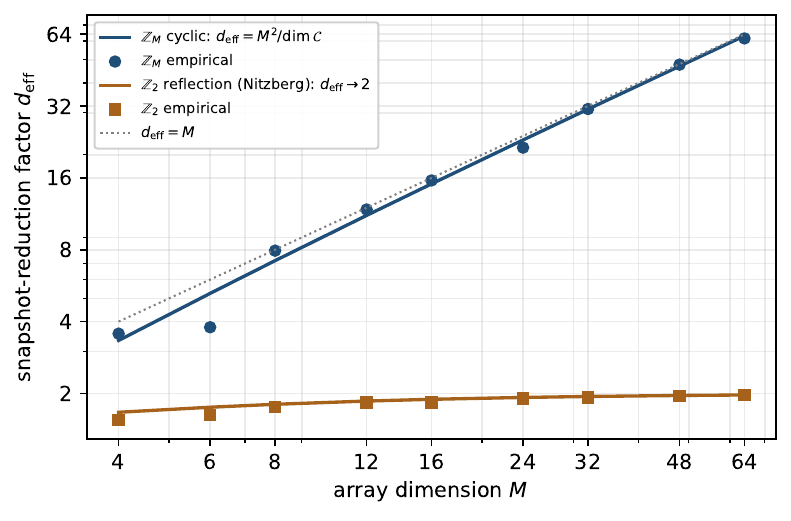}
\caption{Matched-group gain (snapshot-reduction factor) versus array dimension $M$ for a population covariance placed in each group commutant, measured as the ratio of unconstrained to group-constrained covariance mean-squared error over Monte Carlo trials. The reflection group $\mathbb{Z}_2$ (Nitzberg's persymmetric case) approaches $d_{\mathrm{eff}} = 2$; the cyclic group $\mathbb{Z}_M$ approaches $d_{\mathrm{eff}} = M$. Solid curves are the prediction $d_{\mathrm{eff}} = M^2/\dim\mathcal{C}(\rho)$ of Corollary~\ref{cor:sample_complexity_signal}; markers are empirical.}
\label{fig:deff}
\end{figure}

\section{Optimality: the KL Transform and the Matched Group}\label{sec:optimality}

We now establish that the KL transform is the optimal linear decomposition for algebraic diversity and that the matched group attains it, and we clarify the precise sense in which the symmetric group $S_M$ enters through its Cayley graph spectral construction.

\subsection{The KL Optimality Chain}

The Karhunen--Lo\`eve transform~\cite{karhunen1946} is optimal among all linear transforms in three precise senses:
\begin{enumerate}
\item[(P1)] \textbf{Decorrelation:} KL components are mutually uncorrelated (orthogonal).
\item[(P2)] \textbf{Variance concentration:} The first $K$ KL components capture more variance than the first $K$ components of any other orthogonal decomposition.
\item[(P3)] \textbf{Reconstruction:} KL minimizes mean squared error for any fixed truncation order.
\end{enumerate}

These properties are classical~\cite{karhunen1946,loeve1955,hotelling1933} and characterize the KL transform uniquely (up to ordering of equal-variance components).

\subsection{Connection to Cayley Graph Spectra}

The following result, established empirically in~\cite{thornton2005} and proven formally herein, connects the Cayley graph spectrum to the KL spectrum.

\begin{proposition}[CG--KL Spectral Equivalence~\cite{thornton2005}]\label{prop:cg_kl}
Let $\mathbf{F}_\circ$ be the Cayley graph autocorrelation matrix (Definition~\ref{def:cayley_matrix}) constructed using the symmetric group $S_M$ with composition as the group operation and the observation vector as the coloring function. The eigenvalues of $\mathbf{F}_\circ$ are equivalent to the KL spectral coefficients of the discrete function represented by the observation.
\end{proposition}

\subsection{Commutativity and the KL Basis}

The following result makes explicit the chain of equivalences that connects the commutativity condition to the KL spectral decomposition. It is stated here as a named proposition because it serves as the linchpin between the group-averaged estimator (which is constructed from a single observation and a group) and the KL transform (which is optimal among all linear transforms).

\begin{proposition}[Commutativity--KL Equivalence]\label{prop:commute_kl}
Let $\mathbf{F}_G$ be the group-averaged estimator constructed from an observation $\mathbf{x}$ and a finite group $G$, and let $\mathbf{R}$ be the population covariance matrix of the signal model. Suppose both $\mathbf{F}_G$ and $\mathbf{R}$ are Hermitian. Then the following are equivalent:
\begin{enumerate}
\item[(C1)] \textbf{Commutativity:} $\mathbf{F}_G \mathbf{R} = \mathbf{R} \mathbf{F}_G$.
\item[(C2)] \textbf{Simultaneous diagonalizability:} There exists a single unitary matrix $\mathbf{U}$ such that $\mathbf{U}^H \mathbf{F}_G \mathbf{U}$ and $\mathbf{U}^H \mathbf{R} \mathbf{U}$ are both diagonal.
\item[(C3)] \textbf{Shared KL eigenvector basis:} The columns of $\mathbf{U}$ are simultaneously eigenvectors of $\mathbf{F}_G$ and eigenvectors of $\mathbf{R}$. Since the eigenvectors of $\mathbf{R}$ are, by definition, the KL basis vectors, the group-averaged estimator $\mathbf{F}_G$ is diagonalized by the KL basis.
\end{enumerate}
When these equivalent conditions hold, the eigenvalues of $\mathbf{F}_G$ in the shared basis are $|\mathbf{u}_k^H \mathbf{x}|^2$ (the squared magnitudes of the KL coefficients of the observation), and the eigenvalues of $\mathbf{R}$ are the KL spectral coefficients $\lambda_1, \ldots, \lambda_M$.
\end{proposition}

\begin{proof}
\textit{(C1) $\Rightarrow$ (C2):} Since $\mathbf{F}_G$ is Hermitian, the Spectral Theorem provides a unitary $\mathbf{U}_F$ diagonalizing $\mathbf{F}_G$. Let $E_i$ denote the eigenspace of $\mathbf{F}_G$ for eigenvalue $\mu_i$. Commutativity implies $\mathbf{R}$ maps each $E_i$ into itself: if $\mathbf{F}_G \mathbf{u} = \mu_i \mathbf{u}$, then $\mathbf{F}_G(\mathbf{R}\mathbf{u}) = \mathbf{R}(\mathbf{F}_G\mathbf{u}) = \mu_i(\mathbf{R}\mathbf{u})$, so $\mathbf{R}\mathbf{u} \in E_i$. Since $\mathbf{R}$ is Hermitian, it can be diagonalized within each $E_i$. Assembling these bases gives a unitary $\mathbf{U}$ diagonalizing both.

\textit{(C2) $\Rightarrow$ (C3):} If $\mathbf{U}$ diagonalizes both, then each column $\mathbf{u}_k$ satisfies $\mathbf{F}_G \mathbf{u}_k = \mu_k \mathbf{u}_k$ and $\mathbf{R}\mathbf{u}_k = \lambda_k \mathbf{u}_k$. The latter is the definition of a KL basis vector.

\textit{(C3) $\Rightarrow$ (C1):} If both matrices are diagonal in the same basis ($\mathbf{F}_G = \mathbf{U}\boldsymbol{\Lambda}_F\mathbf{U}^H$, $\mathbf{R} = \mathbf{U}\boldsymbol{\Lambda}_R\mathbf{U}^H$), then $\mathbf{F}_G\mathbf{R} = \mathbf{U}\boldsymbol{\Lambda}_F\boldsymbol{\Lambda}_R\mathbf{U}^H = \mathbf{U}\boldsymbol{\Lambda}_R\boldsymbol{\Lambda}_F\mathbf{U}^H = \mathbf{R}\mathbf{F}_G$, since diagonal matrices commute.

The eigenvalue statement follows from $[\mathbf{U}^H\mathbf{F}_G\mathbf{U}]_{k,k} = \mathbf{u}_k^H\mathbf{F}_G\mathbf{u}_k$, which (substituting the definition of $\mathbf{F}_G$ and using the commutativity condition) equals $|\mathbf{u}_k^H\mathbf{x}|^2$.
\end{proof}

\begin{remark}
Proposition~\ref{prop:commute_kl} makes precise the mechanism by which group selection determines spectral optimality. The commutativity condition (C1) is testable from data (via the commutator norm $\|\mathbf{F}_G\mathbf{R} - \mathbf{R}\mathbf{F}_G\|_F$). When it holds, the group-averaged estimator automatically decomposes in the KL basis (C3), yielding KL-optimal spectral coefficients without explicit computation of the KL transform. The condition fails when the algebraic structure of $G$ is mismatched to the covariance structure of $\mathbf{R}$.
\end{remark}

\subsection{Quantifying Group--Model Mismatch}

Proposition~\ref{prop:commute_kl} establishes that exact commutativity yields the KL basis. In practice, commutativity may hold only approximately: the group $G$ may not perfectly match the covariance structure of~$\mathbf{R}$. We introduce two complementary metrics that quantify this mismatch, each capturing a different aspect.

\begin{definition}[Commutativity Residual]\label{def:commut_residual}
Let $\mathbf{F}_G(\mathbf{x})$ be the group-averaged estimator constructed from observation $\mathbf{x}$ and group $G$, and let $\mathbf{R}$ be the population covariance matrix. The \emph{commutativity residual} is
\begin{equation}\label{eq:commut_residual}
\delta(G, \mathbf{x}, \mathbf{R}) = \frac{\|\mathbf{F}_G \mathbf{R} - \mathbf{R}\mathbf{F}_G\|_F}{\|\mathbf{F}_G\|_F \cdot \|\mathbf{R}\|_F},
\end{equation}
where $\|\cdot\|_F$ denotes the Frobenius norm. The commutativity residual satisfies $0 \leq \delta \leq 2$, with $\delta = 0$ if and only if $\mathbf{F}_G$ and $\mathbf{R}$ commute. It is \emph{scale-invariant}: $\delta(G, c\mathbf{x}, \alpha\mathbf{R}) = \delta(G, \mathbf{x}, \mathbf{R})$ for any nonzero scalar $c$ and any $\alpha > 0$.
\end{definition}

\begin{definition}[Absolute Commutativity Mismatch]\label{def:abs_mismatch}
The \emph{absolute commutativity mismatch} is
\begin{equation}\label{eq:abs_mismatch}
\tilde{\delta}(G, \mathbf{x}, \mathbf{R}) = \frac{\|\mathbf{F}_G \mathbf{R} - \mathbf{R}\mathbf{F}_G\|_F}{\|\mathbf{F}_G\|_F}.
\end{equation}
This differs from the commutativity residual in that the denominator normalizes only by the group-averaged estimator, not by the covariance. Consequently, $\tilde{\delta}$ is expressed in the natural scale of $\mathbf{R}$: it measures the covariance mismatch per unit of group action. Unlike $\delta$, it is \emph{not} scale-invariant in $\mathbf{R}$: scaling the covariance by $\alpha > 0$ scales $\tilde{\delta}$ by~$\alpha$.
\end{definition}

\begin{remark}[Complementary Roles of the Three Metrics]\label{rem:three_metrics}
The commutativity residual $\delta$, the absolute commutativity mismatch $\tilde{\delta}$, and the algebraic coloring index $\alpha$ (Definition~\ref{def:coloring_index}) capture complementary aspects of the relationship between a group and a signal model:
\begin{enumerate}
\item[(M1)] \textbf{Algebraic coloring index $\alpha(\mathbf{R})$:} Measures the departure of $\mathbf{R}$ from white noise. It depends only on the eigenvalue distribution of $\mathbf{R}$ and is independent of any group. It answers: \emph{``how much algebraic structure exists in the covariance?''}

\item[(M2)] \textbf{Commutativity residual $\delta(G, \mathbf{x}, \mathbf{R})$:} Measures the structural mismatch between $G$ and $\mathbf{R}$. It is dimensionless and scale-invariant, depending on the \emph{eigenvector alignment} between $\mathbf{F}_G$ and $\mathbf{R}$ rather than on eigenvalue magnitudes. It answers: \emph{``how well does this group's algebraic structure match the covariance structure?''}

\item[(M3)] \textbf{Absolute commutativity mismatch $\tilde{\delta}(G, \mathbf{x}, \mathbf{R})$:} Measures the mismatch in the natural units of $\mathbf{R}$, so that signals with larger covariance values (higher energy or SNR) produce larger mismatch values for the same structural misalignment. It answers: \emph{``what is the magnitude of the covariance information lost by using this group?''}
\end{enumerate}

A signal model with high $\alpha$ but low $\delta$ is one where substantial structure exists and the group captures it well. High $\alpha$ with high $\delta$ indicates the wrong group choice. Low $\alpha$ indicates little structure to exploit regardless of group selection. The mismatch $\tilde{\delta}$ adds the energy dimension: two signals with identical $\delta$ but different SNR will have different $\tilde{\delta}$, reflecting the practical consequence of the structural mismatch.
\end{remark}

\begin{proposition}[Algebraic Relationship Among the Metrics]\label{prop:metric_relation}
The commutativity residual $\delta$ and the absolute commutativity mismatch $\tilde{\delta}$ are related by:
\begin{equation}\label{eq:metric_relation}
\tilde{\delta}(G, \mathbf{x}, \mathbf{R}) = \delta(G, \mathbf{x}, \mathbf{R}) \cdot \|\mathbf{R}\|_F.
\end{equation}
That is, the two metrics carry the same structural information; they differ only in whether the Frobenius norm of the covariance matrix is factored out (yielding the dimensionless $\delta$) or retained (yielding the energy-weighted $\tilde{\delta}$).

Furthermore, the algebraic coloring index $\alpha$ constrains $\delta$ and $\tilde{\delta}$ through the implication
\begin{equation}\label{eq:alpha_implies}
\alpha(\mathbf{R}) = 0 \;\;\Longrightarrow\;\; \delta(G, \mathbf{x}, \mathbf{R}) = \tilde{\delta}(G, \mathbf{x}, \mathbf{R}) = 0 \quad \text{for all } G,
\end{equation}
but the converse does not hold.
\end{proposition}

\begin{proof}
Equation~(\ref{eq:metric_relation}) follows directly from the definitions:
\[
\tilde{\delta} = \frac{\|\mathbf{F}_G \mathbf{R} - \mathbf{R}\mathbf{F}_G\|_F}{\|\mathbf{F}_G\|_F} = \frac{\|\mathbf{F}_G \mathbf{R} - \mathbf{R}\mathbf{F}_G\|_F}{\|\mathbf{F}_G\|_F \cdot \|\mathbf{R}\|_F} \cdot \|\mathbf{R}\|_F = \delta \cdot \|\mathbf{R}\|_F.
\]

For the implication~(\ref{eq:alpha_implies}): $\alpha(\mathbf{R}) = 0$ if and only if $\mathbf{R} = \bar{q}\,\mathbf{I}_M$ where $\bar{q} = \Tr(\mathbf{R})/M$. In this case,
\[
\mathbf{F}_G \mathbf{R} - \mathbf{R}\mathbf{F}_G = \mathbf{F}_G(\bar{q}\,\mathbf{I}_M) - (\bar{q}\,\mathbf{I}_M)\mathbf{F}_G = \bar{q}\,\mathbf{F}_G - \bar{q}\,\mathbf{F}_G = \mathbf{0},
\]
since any matrix commutes with a scalar multiple of the identity. Therefore $\delta = \tilde{\delta} = 0$.

The converse fails because a non-white covariance can still commute with a particular group's estimator. For example, a circulant $\mathbf{R}$ with $\alpha(\mathbf{R}) > 0$ (non-uniform eigenvalues) satisfies $\delta(\mathbb{Z}_M, \mathbf{x}, \mathbf{R}) = 0$ because all circulant matrices commute with one another. Thus $\delta = 0$ does not imply $\alpha = 0$.
\end{proof}

\subsection{The Optimality Theorem}

\begin{theorem}[KL Optimality and the Matched Group]\label{thm:optimality}
Let $\mathbf{R}$ be the signal-plus-noise covariance. Among all finite groups $G$ acting on $\{0, \ldots, M-1\}$ by a unitary representation, the eigendecomposition of the group-averaged estimator $\mathbf{F}_G$ cannot exceed the KL transform of $\mathbf{R}$ in any of the optimality criteria (P1)--(P3). A group attains the KL transform if and only if it is \emph{matched}, i.e., $\mathbf{R}$ commutes with the group action (Proposition~\ref{prop:commute_kl}); the matched group is therefore the operative optimum. The symmetric group $S_M$ attains the KL spectrum through its Cayley graph spectral construction $\mathbf{F}_\circ$ (Proposition~\ref{prop:cg_kl}), which is a distinct object from the group-averaged estimator $\mathbf{F}_{S_M}$.
\end{theorem}

\begin{proof}
\textit{No group-averaged estimator exceeds KL.} Suppose for contradiction that some $\mathbf{F}_{G^*}$ yields a spectral decomposition that outperforms the KL transform in one of (P1)--(P3). Since $\mathbf{F}_{G^*}$ is Hermitian, its eigendecomposition is a linear orthogonal transform, and the KL transform is optimal among \emph{all} such transforms in (P1)--(P3); this is a contradiction. Hence no group can exceed KL.

\textit{The matched group attains KL.} By Proposition~\ref{prop:commute_kl}, $\mathbf{F}_G$ and $\mathbf{R}$ share the KL eigenbasis exactly when $[\mathbf{F}_G, \mathbf{R}] = \mathbf{0}$, which holds when $\mathbf{R}$ lies in the commutant of $G$. Thus a matched group attains the KL-optimal decomposition (P1)--(P3), and the smallest such group is characterized in Theorem~\ref{thm:minimal}.

\textit{Role of $S_M$.} Proposition~\ref{prop:cg_kl} shows that the Cayley graph autocorrelation matrix $\mathbf{F}_\circ$ over $S_M$ (Definition~\ref{def:cayley_matrix}) has a spectrum equivalent to the KL spectrum. This is a property of $\mathbf{F}_\circ$, not of the group-averaged estimator $\mathbf{F}_{S_M}$ of Definition~\ref{def:group_avg}: averaging $\mathbf{x}\mathbf{x}^H$ over all of $S_M$ projects onto the two-dimensional $S_M$-commutant $\mathrm{span}\{\mathbf{I}, \mathbf{1}\mathbf{1}^T\}$ (equation~\eqref{eq:sm_expect}), retaining no spectral shape. The two constructions therefore play opposite roles for $S_M$: $\mathbf{F}_\circ$ realizes KL, while $\mathbf{F}_{S_M}$ over-averages and is the maximally uninformative estimator (Section~\ref{sec:pase}). Consequently the practical optimum among group-averaged estimators is the matched group, not $S_M$.
\end{proof}

\begin{remark}[Two distinct constructions]\label{rem:fo_vs_fg}
The framework uses two matrices built from a group and an observation: the group-averaged estimator $\mathbf{F}_G = \frac{1}{|G|}\sum_g \rho(g)\mathbf{x}\mathbf{x}^H\rho(g)^H$ (Definition~\ref{def:group_avg}), which is the practical estimator analyzed throughout, and the Cayley graph autocorrelation matrix $\mathbf{F}_\circ$ (Definition~\ref{def:cayley_matrix}), whose $S_M$ spectrum equals the KL spectrum~\cite{thornton2005}. These coincide for the cyclic group (both give the circulant DFT decomposition) but diverge for large non-regular groups: $\mathbf{F}_\circ$ over $S_M$ encodes KL, whereas $\mathbf{F}_{S_M}$ collapses to~\eqref{eq:sm_expect}. Claims of $S_M$ optimality refer to $\mathbf{F}_\circ$; the PASE result and the ordering experiment (Sections~\ref{sec:pase}--\ref{sec:ordering}) concern $\mathbf{F}_G$, for which the matched order-$M$ group is optimal.
\end{remark}

\subsection{Minimal Groups for Structured Signals}

The matched group is the operative optimum; we now characterize the smallest group that attains it.

\begin{theorem}[Minimal Group Characterization]\label{thm:minimal}
For a signal class $\mathcal{S}$ with symmetry group $H = \{g \in S_M : g \cdot \mathcal{S} = \mathcal{S}\}$, the minimal group achieving KL-equivalent decomposition is $G_{\min} = H$, provided $H$ acts transitively on the signal's support.
\end{theorem}

\begin{proof}
If $\mathcal{S}$ is invariant under $H$, then the signal component of $\mathbf{F}_H(\mathbf{x})$ captures all signal energy through the irreducible representations of $H$ that appear in $\mathcal{S}$. Transitivity ensures that the group orbit covers the full support, so no signal energy is missed. Any subgroup $G \subsetneq H$ fails to preserve $\mathcal{S}$, potentially mixing signal and noise components.

For larger groups $H \subsetneq G \subseteq S_M$, the additional elements impose invariances that $\mathbf{R}$ does not possess: they lie outside the commutant of $\mathbf{R}$, so the group average projects away genuine spectral structure and raises the bias floor. Consistent with the PASE result (Section~\ref{sec:pase}) and the order constraint (Remark~\ref{rem:group_order}), enlarging the group beyond $H$ degrades rather than improves the group-averaged estimate, so $H$ is both minimal and the operative optimum.
\end{proof}

\begin{example}[ULA Signals]
For signals on a uniform linear array with spatial frequencies $\{\phi_1, \ldots, \phi_K\}$, the signal class is invariant under cyclic shifts. The minimal group is $G_{\min} = \mathbb{Z}_M$, the cyclic group of order $M$, which produces a circulant matrix with DFT eigenvectors. This is the minimal group achieving KL-equivalent decomposition for translationally symmetric signals. This confirms that for the translationally symmetric signal class, algebraic diversity with the minimal group $\mathbb{Z}_M$ is equivalent to DFT-based processing, and the framework's additional power arises precisely when the signal's symmetry structure requires a group larger than $\mathbb{Z}_M$.
\end{example}

\begin{corollary}\label{cor:hierarchy}
For any signal class $\mathcal{S}$ on $M$ elements, the following hierarchy holds:
\begin{equation}
\mathbb{Z}_M \subseteq G_{\min}(\mathcal{S}) \subseteq S_M,
\end{equation}
where $G_{\min}(\mathcal{S})$ is the minimal group for class $\mathcal{S}$, and any $G$ with $G_{\min}(\mathcal{S}) \subseteq G \subseteq S_M$ achieves KL-equivalent spectral decomposition for signals in $\mathcal{S}$.
\end{corollary}

\section{The Duality Principle}\label{sec:duality}

\begin{theorem}[Temporal--Algebraic Duality]\label{thm:duality}
Let $\mathbf{x}(1), \ldots, \mathbf{x}(L)$ be $L$ independent realizations of the signal model~(\ref{eq:obs_model}) with common signal component $\mathbf{s}$ and independent noise $\mathbf{n}(t)$. Let $G$ be a finite group satisfying Conditions~\ref{cond:equivariance} and~\ref{cond:ergodicity}. Then:
\begin{equation}\label{eq:duality}
\lim_{L \to \infty} \hat{\mathbf{R}}_L = \mathbf{R}_s + \sigma^2\mathbf{I} = \lim_{\text{SNR} \to \infty} \mathbf{F}_G(\mathbf{x}),
\end{equation}
in the sense that both limits yield the same signal subspace $\mathcal{U}_s$ and noise subspace $\mathcal{U}_n$, up to a group-dependent similarity transformation within each subspace.
\end{theorem}

\begin{proof}
The left equality is the classical consistency of the sample covariance. For the right equality, by Theorem~\ref{thm:replacement} parts (ii) and (iii), $E[\mathbf{F}_G(\mathbf{x})]$ has signal components concentrated in the isotypic blocks corresponding to the signal's representation, and noise contributing $\sigma^2\mathbf{I}$. As SNR $\to \infty$, the noise contribution becomes negligible relative to the signal, and the eigenspace decomposition of $\mathbf{F}_G(\mathbf{x})$ converges to the signal-noise partition.

The eigenvectors may differ between $\hat{\mathbf{R}}_L$ and $\mathbf{F}_G(\mathbf{x})$ (the former being the population covariance eigenvectors, the latter being the group representation basis vectors), but they span the same subspaces. This is because any two bases for the same subspace are related by an invertible transformation within that subspace.

The connection to Theorem~\ref{thm:trivial} is immediate: the left side of~(\ref{eq:duality}) is the limit of $L$ trivial-group estimates $\mathbf{F}_{\{e\}}(\mathbf{x}(t))$ as $L \to \infty$, while the right side is the high-SNR limit of a single matched-group estimate $\mathbf{F}_G(\mathbf{x})$. Both achieve the same subspace decomposition, but through different points on the $(G, L)$ continuum.
\end{proof}

\begin{remark}[Interpretive Summary]
The duality can be stated informally as follows:
\begin{itemize}
\item \textbf{Temporal averaging} ($G = \{e\}$, $L$ snapshots) samples from the orbit of the noise process under the time-translation group, holding the signal fixed. As $L \to \infty$, the noise averages out and the signal covariance emerges. This is algebraic diversity with the trivial group: each snapshot contributes a rank-one outer product, and rank is built by accumulation across observations.
\item \textbf{Algebraic diversity} ($|G| = M$, $L = 1$) samples from the orbit of the single observation under the permutation group. The signal, being structured, transforms predictably; the noise, being unstructured, is scrambled. The eigendecomposition separates the predictable from the scrambled. Rank is built within a single observation by exhausting the group orbit.
\end{itemize}
Both are instances of the same principle: \emph{averaging over a group orbit of the unstructured component reveals the invariant structure}. The $(G, L)$ continuum (Remark~\ref{rem:gl_continuum}) interpolates continuously between these extremes, with variance $\propto 1/(d_{\mathrm{eff}} \cdot L)$.
\end{remark}

\subsection{Spatial, Temporal, and Hybrid Observation Modes}

The $(G, L)$ continuum (Remark~\ref{rem:gl_continuum}) establishes that temporal and algebraic diversity are interchangeable sources of estimation variance reduction. This equivalence extends to a precise quantitative parity among three modes of forming the observation vector.

\begin{corollary}[TAD-SAD Exchange Rate]\label{cor:tad_sad}
Let $\text{SNR}_{\text{out}}$ denote the output signal-to-noise ratio after algebraic diversity processing. The following three observation modes yield the same algebraic diversity framework, differing only in how the $M$-dimensional observation vector is formed:

\begin{enumerate}
\item[(i)] \textbf{Spatial Algebraic Diversity (SAD):} $M$ sensors simultaneously sample a signal at a single time instant. The observation vector is $\mathbf{x}_s = [x_1, x_2, \ldots, x_M]^T \in \mathbb{C}^M$, with the group acting on the sensor index. The output SNR improvement is up to $10\log_{10}(M)$~dB.

\item[(ii)] \textbf{Temporal Algebraic Diversity (TAD):} A single sensor produces $M$ sequential temporal samples. The observation vector is $\mathbf{x}_t = [x(1), x(2), \ldots, x(M)]^T \in \mathbb{C}^M$, with the group acting on the temporal index. The output SNR improvement is up to $10\log_{10}(M)$~dB, identical to SAD.

\item[(iii)] \textbf{Hybrid TAD-SAD:} $K$ sensors each produce $N$ temporal samples, forming an observation vector $\mathbf{x}_{st} \in \mathbb{C}^{KN}$ by concatenation. The group acts on the joint space-time index. The output SNR improvement is up to $10\log_{10}(KN)$~dB, which decomposes additively as $10\log_{10}(K) + 10\log_{10}(N)$~dB.
\end{enumerate}

The exchange rate between spatial sensor elements, temporal samples, and algebraic group elements is exactly $1\!:\!1\!:\!1$: one additional sensor element, one additional temporal sample, and one additional group element each contribute identically to the SNR improvement.
\end{corollary}

\begin{proof}
For each mode, the group-averaged estimator~(\ref{eq:group_avg}) has the form $\mathbf{F}_G(\mathbf{x}) = \frac{1}{|G|}\sum_{g \in G}[\rho(g)\mathbf{x}][\rho(g)\mathbf{x}]^H$, where $\mathbf{x} \in \mathbb{C}^D$ with $D = M$ for SAD and TAD, and $D = KN$ for the hybrid mode. For a signal matched to the group (so that the matched effective dimension equals $D$), Theorem~\ref{thm:replacement} concentrates the signal energy in $K_s$ eigenvalues of $\mathbf{F}_G$ while the noise energy is distributed across all $D$ commutant directions. The output SNR at the dominant eigenvector then satisfies $\text{SNR}_{\text{out}} = D \cdot \text{SNR}_{\text{in}}$, the eigenvalue-domain effect of concentrating the matched signal against noise spread over $D$ directions; in decibels, $\text{SNR}_{\text{out}} = \text{SNR}_{\text{in}} + 10\log_{10}(D)$~dB. (As in Section~\ref{sec:pase}, this is a variance reduction over the $D$-dimensional commutant rather than the averaging of $D$ statistically independent noise samples, and it holds for matched signals.)

For SAD, $D = M$; for TAD, $D = M$; for the hybrid, $D = KN$. The $10\log_{10}(KN) = 10\log_{10}(K) + 10\log_{10}(N)$ decomposition follows from the logarithm, establishing the $1\!:\!1\!:\!1$ exchange rate.

The equivalence between SAD and TAD follows from the observation that the General Replacement Theorem (Theorem~\ref{thm:replacement}) depends only on the dimension $D$ of the observation vector and the algebraic structure of the group action, not on whether the components of $\mathbf{x}$ arise from spatial or temporal sampling. The equivariance and ergodicity conditions (Conditions~\ref{cond:equivariance}--\ref{cond:ergodicity}) are satisfied symmetrically in both cases: for SAD, a spatially structured signal is equivariant under spatial permutations while spatially white noise is ergodic; for TAD, a temporally structured signal (e.g., a narrowband process) is equivariant under temporal shifts while temporally white noise is ergodic.
\end{proof}

\begin{remark}[Practical Significance]
The $1\!:\!1\!:\!1$ exchange rate has direct engineering consequences. A system designer constrained to $K$ sensors (fewer than desired) can compensate by collecting $N = \lceil M/K \rceil$ temporal samples per sensor, achieving the same algebraic diversity performance as an $M$-element array from a single snapshot. Conversely, a system with $M$ sensors but requiring minimum-latency processing can operate in pure SAD mode with $N = 1$, accepting the spatial aperture as the sole source of diversity. The hybrid mode provides a continuous tradeoff between spatial resources, temporal resources, and processing latency.
\end{remark}

\section{Extension to Colored Noise}\label{sec:colored}

The preceding development assumes spatially white noise, $\mathbf{n} \sim \mathcal{CN}(\mathbf{0}, \sigma^2\mathbf{I}_M)$, so that Condition~\ref{cond:ergodicity} is satisfied automatically for any unitary representation. In many practical settings, however, the noise environment is \emph{colored}: its covariance $\mathbf{Q} = E[\mathbf{n}\mathbf{n}^H]$ is a general positive-definite matrix that is not proportional to the identity. Adjacent-cell interference in MIMO systems, environmental noise in passive geolocation, and the acoustic signal of interest in active noise cancellation are all instances of colored noise. In this section, we show that the algebraic diversity framework extends naturally to colored noise, and, more significantly, that the noise covariance itself admits a group-theoretic characterization that provides structural insight and computational advantages beyond conventional pre-whitening.

\subsection{Generalized Signal Model}

We generalize the observation model~(\ref{eq:obs_model}) to
\begin{equation}\label{eq:colored_model}
\mathbf{x} = \mathbf{s} + \mathbf{n}, \qquad \mathbf{n} \sim \mathcal{CN}(\mathbf{0}, \mathbf{Q}),
\end{equation}
where $\mathbf{Q} \in \mathbb{C}^{M \times M}$ is a positive-definite Hermitian noise covariance. The white noise case~(\ref{eq:obs_model}) corresponds to $\mathbf{Q} = \sigma^2\mathbf{I}_M$.

In this setting, Condition~\ref{cond:ergodicity} is no longer automatically satisfied: for a unitary representation $\rho$, the transformed noise $\rho(g)\mathbf{n}$ has covariance $\rho(g)\mathbf{Q}\rho(g)^H$, which in general differs from $\mathbf{Q}$ unless $\rho(g)$ commutes with $\mathbf{Q}$. This motivates the following generalization.

\subsection{Group-Theoretic Noise Characterization}

The key observation is that the algebraic diversity framework, when applied to a noise-only observation, reveals the algebraic structure of the noise itself.

\begin{definition}[Natural Group of a Noise Process]\label{def:natural_group}
Let $\mathbf{Q}$ be the covariance matrix of a noise process on $\mathbb{C}^M$. For a finite group $G$ with unitary representation $\rho: G \to U(M)$, define the \emph{diagonalization residual}
\begin{equation}\label{eq:diag_residual}
\delta_{\mathrm{diag}}(G, \mathbf{Q}) = \frac{\|\mathbf{T}_G \mathbf{Q} \mathbf{T}_G^H - \diag(\mathbf{T}_G \mathbf{Q} \mathbf{T}_G^H)\|_F}{\|\mathbf{Q}\|_F},
\end{equation}
where $\mathbf{T}_G$ is the unitary change-of-basis matrix corresponding to the irreducible decomposition of $\rho$, and $\diag(\cdot)$ extracts the diagonal. The subscript distinguishes $\delta_{\mathrm{diag}}$, which measures off-diagonal mass after the $G$-transform, from the commutativity residual $\delta$ of Definition~\ref{def:commut_residual}. The \emph{natural group} of the noise process is
\begin{equation}\label{eq:natural_group}
G_{\mathbf{Q}} = \arg\min_{G \in \mathcal{G}_M} \delta_{\mathrm{diag}}(G, \mathbf{Q}),
\end{equation}
where $\mathcal{G}_M$ is a catalog of finite groups acting on $\{0, \ldots, M-1\}$.
\end{definition}

\begin{remark}[Interpretation]
The natural group $G_{\mathbf{Q}}$ is the group whose representation theory best describes the correlation structure of the noise. When $\mathbf{Q} = \sigma^2\mathbf{I}_M$ (white noise), $\delta_{\mathrm{diag}}(G, \mathbf{Q}) = 0$ for every group, reflecting the fact that white noise has no preferred algebraic structure, every group diagonalizes it equally well. As $\mathbf{Q}$ departs from a scalar multiple of the identity, specific groups become distinguished.
\end{remark}

\begin{remark}[Group Catalog]
The catalog $\mathcal{G}_M$ is application-dependent but naturally includes, in order of increasing generality: the cyclic group $\mathbb{Z}_M$ (diagonalizes circulant/Toeplitz structures via the DFT), the dihedral group ($D_M$ acting on the $M$ samples selects symmetric circulant structures and the real Fourier basis, while $D_{2M}$ acting on the even-reflected extension gives the DCT~\cite{thornton2026unification}), other regular subgroups of $S_M$ arising from the array geometry, and the full symmetric group $S_M$ itself. The search over $\mathcal{G}_M$ is computationally inexpensive: each candidate requires one transform of the estimated $\hat{\mathbf{Q}}$ and evaluation of the Frobenius norm of the off-diagonal residual. For the groups listed above, the transforms have $O(M \log M)$ fast implementations.
\end{remark}

\begin{example}[Stationary Noise and the Cyclic Group]\label{ex:stationary}
A wide-sense stationary noise process on a uniform linear array has a Toeplitz covariance $\mathbf{Q}$, which is asymptotically circulant~\cite{grenander1958}. Circulant matrices are exactly those diagonalized by the DFT, which corresponds to the cyclic group $\mathbb{Z}_M$. Therefore, $G_{\mathbf{Q}} = \mathbb{Z}_M$ for stationary noise, and the diagonal entries of $\mathbf{T}_{\mathbb{Z}_M}\mathbf{Q}\mathbf{T}_{\mathbb{Z}_M}^H$ are the noise power spectral density samples.
\end{example}

\begin{definition}[Algebraic Coloring Index]\label{def:coloring_index}
The \emph{algebraic coloring index} of a noise process with covariance $\mathbf{Q}$ is
\begin{equation}\label{eq:coloring_index}
\alpha(\mathbf{Q}) = \frac{\|\mathbf{Q} - \bar{q}\,\mathbf{I}_M\|_F}{\|\mathbf{Q}\|_F},
\end{equation}
where $\bar{q} = \Tr(\mathbf{Q})/M$ is the mean eigenvalue. Equivalently, if $q_1, \ldots, q_M$ are the eigenvalues of $\mathbf{Q}$, then
\begin{equation}
\alpha(\mathbf{Q}) = \sqrt{\frac{\sum_{k=1}^{M}(q_k - \bar{q})^2}{\sum_{k=1}^{M} q_k^2}},
\end{equation}
which is recognized as the coefficient of variation of the eigenvalue spectrum, normalized by the $\ell^2$ norm rather than the mean.
\end{definition}

\begin{remark}
The algebraic coloring index satisfies $\alpha(\mathbf{Q}) = 0$ if and only if $\mathbf{Q} = \bar{q}\mathbf{I}_M$ (white noise), and $\alpha(\mathbf{Q}) \to 1$ as the noise energy concentrates in a single eigenmode. It is invariant under unitary similarity transformations, $\alpha(\mathbf{U}\mathbf{Q}\mathbf{U}^H) = \alpha(\mathbf{Q})$, and provides a scalar summary of the degree to which the noise departs from isotropic.
\end{remark}

\subsection{Generalized Noise Ergodicity Condition}

With the natural group identified, we can state a generalized form of Condition~\ref{cond:ergodicity}.

\begin{condition}[Generalized Noise Ergodicity]\label{cond:gen_ergodicity}
Let $G_{\mathbf{Q}}$ be the natural group of the noise process and $\mathbf{T}_{G_\mathbf{Q}}$ the corresponding change-of-basis matrix. Define the whitened noise $\tilde{\mathbf{n}} = \mathbf{Q}^{-1/2}\mathbf{n}$. Then $\tilde{\mathbf{n}} \sim \mathcal{CN}(\mathbf{0}, \mathbf{I}_M)$, and for any unitary representation $\rho$ of any group $G$:
\begin{equation}
\rho(g)\tilde{\mathbf{n}} \sim \mathcal{CN}(\mathbf{0}, \mathbf{I}_M) \quad \text{for all } g \in G.
\end{equation}
\end{condition}

\subsection{Generalized Replacement Theorem for Colored Noise}

\begin{theorem}[Generalized Replacement for Colored Noise]\label{thm:colored_noise}
Let $\mathbf{x} = \mathbf{s} + \mathbf{n}$ with $\mathbf{n} \sim \mathcal{CN}(\mathbf{0}, \mathbf{Q})$ where $\mathbf{Q}$ is positive-definite. Let $G_{\mathbf{Q}}$ be the natural group of the noise process. Define the whitened observation $\tilde{\mathbf{x}} = \mathbf{Q}^{-1/2}\mathbf{x}$ and let $G$ be any finite group with unitary representation $\rho$ satisfying Conditions~\ref{cond:equivariance} and~\ref{cond:gen_ergodicity} (applied to $\tilde{\mathbf{x}}$). Then:
\begin{enumerate}
\item[(i)] The group-averaged estimator applied to the whitened observation,
\begin{equation}\label{eq:whitened_estimator}
\mathbf{F}_G(\tilde{\mathbf{x}}) = \frac{1}{|G|}\sum_{g \in G}[\rho(g)\tilde{\mathbf{x}}][\rho(g)\tilde{\mathbf{x}}]^H,
\end{equation}
satisfies all four parts of Theorem~\ref{thm:replacement} with $\tilde{\mathbf{s}} = \mathbf{Q}^{-1/2}\mathbf{s}$ as the signal and $\tilde{\mathbf{n}} \sim \mathcal{CN}(\mathbf{0}, \mathbf{I}_M)$ as white noise.
\item[(ii)] When $G_{\mathbf{Q}}$ commutes with the signal processing group $G$ (i.e., $\mathbf{T}_{G_\mathbf{Q}}$ commutes with $\rho(g)$ for all $g$), the whitening and algebraic diversity operations may be applied independently, and the Optimality Theorem~\ref{thm:optimality} holds for the whitened data without modification.
\item[(iii)] When $G_{\mathbf{Q}}$ coincides with a known group in the catalog $\mathcal{G}_M$, the whitening filter $\mathbf{Q}^{-1/2}$ may be replaced by the fast transform associated with $G_{\mathbf{Q}}$ followed by diagonal scaling, reducing the whitening complexity from $O(M^3)$ (general matrix) to $O(M \log M)$ or the fast transform complexity of $G_{\mathbf{Q}}$.
\end{enumerate}
\end{theorem}

\begin{proof}
\textit{Part (i).} The whitened observation is $\tilde{\mathbf{x}} = \mathbf{Q}^{-1/2}\mathbf{s} + \mathbf{Q}^{-1/2}\mathbf{n} = \tilde{\mathbf{s}} + \tilde{\mathbf{n}}$. Since $\tilde{\mathbf{n}} \sim \mathcal{CN}(\mathbf{0}, \mathbf{I}_M)$, this is exactly the white noise signal model~(\ref{eq:obs_model}) with $\sigma^2 = 1$, and Theorem~\ref{thm:replacement} applies directly.

\textit{Part (ii).} When $\mathbf{T}_{G_\mathbf{Q}}$ commutes with $\rho(g)$, the composite operation $\rho(g)\mathbf{Q}^{-1/2}$ is equivalent to $\mathbf{Q}^{-1/2}\rho(g)$, so the order of whitening and group action is immaterial. The group-averaged estimator on the whitened data then has the same eigenvector structure as the estimator on the original data, with eigenvalues rescaled by the whitening transform. Crucially, the signal and noise \emph{subspaces} are preserved under the invertible whitening map $\mathbf{Q}^{-1/2}$, so the KL optimality properties (P1)--(P3) hold for the whitened signal model $\tilde{\mathbf{x}} = \tilde{\mathbf{s}} + \tilde{\mathbf{n}}$: the eigenvalue magnitudes are those of the whitened covariance $\mathbf{Q}^{-1/2}\mathbf{R}_s\mathbf{Q}^{-1/2} + \mathbf{I}_M$, but the subspace partition, which determines the signal-versus-noise classification used by MUSIC and related algorithms, is identical to that of the original model.

\textit{Part (iii).} If $G_{\mathbf{Q}}$ has representation matrix $\mathbf{T}_{G_\mathbf{Q}}$ that (approximately) diagonalizes $\mathbf{Q}$, then $\mathbf{Q} \approx \mathbf{T}_{G_\mathbf{Q}}^H \boldsymbol{\Lambda}_Q \mathbf{T}_{G_\mathbf{Q}}$ where $\boldsymbol{\Lambda}_Q = \diag(q_1, \ldots, q_M)$. Therefore $\mathbf{Q}^{-1/2} \approx \mathbf{T}_{G_\mathbf{Q}}^H \boldsymbol{\Lambda}_Q^{-1/2} \mathbf{T}_{G_\mathbf{Q}}$: a forward transform by $\mathbf{T}_{G_\mathbf{Q}}$, element-wise scaling by $q_k^{-1/2}$, and an inverse transform by $\mathbf{T}_{G_\mathbf{Q}}^H$. When $G_{\mathbf{Q}} = \mathbb{Z}_M$ (stationary noise), this is an FFT, diagonal scaling by the inverse square root of the power spectral density, and an inverse FFT, the classical frequency-domain whitening filter, at cost $O(M \log M)$.
\end{proof}

\begin{remark}[The Noise Characterization Workflow]
The practical procedure for applying algebraic diversity in colored noise environments is:
\begin{enumerate}
\item \textit{Noise-only observation:} Acquire an observation $\mathbf{x}_n$ during a period when only noise is present (no signal).
\item \textit{Algebraic classification:} For each candidate group $G \in \mathcal{G}_M$, compute the group-averaged estimator $\mathbf{F}_G(\mathbf{x}_n)$ and evaluate the diagonalization residual $\delta_{\mathrm{diag}}(G, \hat{\mathbf{Q}})$ where $\hat{\mathbf{Q}} = \mathbf{F}_G(\mathbf{x}_n)$. Select $G_{\mathbf{Q}} = \arg\min_G \delta_{\mathrm{diag}}(G, \hat{\mathbf{Q}})$.
\item \textit{Structured whitening:} Apply the fast transform of $G_{\mathbf{Q}}$ and diagonal scaling to whiten subsequent signal-bearing observations.
\item \textit{Algebraic diversity processing:} Apply the group-averaged estimator with the signal processing group $G$ (e.g., $\mathbb{Z}_M$ for ULA MUSIC) to the whitened observation.
\end{enumerate}
Note that steps~1--3 characterize the noise environment and need only be performed once (or periodically updated), while step~4 is applied to each signal-bearing observation. The entire pipeline remains within the algebraic framework: both the noise characterization and the signal extraction are group-theoretic operations.
\end{remark}

\begin{remark}[Duality Interpretation of Noise Structure]
The Temporal--Algebraic Duality Principle (Theorem~\ref{thm:duality}) provides a natural interpretation of colored noise within the algebraic framework. White noise, being structureless, has no preferred algebraic description, it is the \emph{identity element} in the space of noise processes, in the sense that $\mathbf{Q} = \sigma^2\mathbf{I}_M$ commutes with every unitary transform and hence every group representation acts equivalently on it. Colored noise possesses structure, a non-flat power spectral density or non-isotropic spatial correlation, and this structure ``selects'' a preferred group $G_{\mathbf{Q}}$ from the catalog. The departure from white noise is thus a departure from algebraic universality: the noise acquires a symmetry that distinguishes among groups. The algebraic coloring index $\alpha(\mathbf{Q})$ quantifies this departure, with $\alpha = 0$ corresponding to the maximally symmetric (structure-free) case and $\alpha \to 1$ corresponding to maximally structured noise.
\end{remark}

\begin{corollary}[Reduced Sample Complexity of Group-Constrained Noise Estimation]\label{cor:sample_complexity}
Let $\mathbf{Q}$ be a noise covariance with natural group $G_{\mathbf{Q}}$, and let $\mathbf{T}_{G_\mathbf{Q}}$ exactly diagonalize $\mathbf{Q}$ so that $\mathbf{Q} = \mathbf{T}_{G_\mathbf{Q}}^H \boldsymbol{\Lambda}_Q \mathbf{T}_{G_\mathbf{Q}}$ with $\boldsymbol{\Lambda}_Q = \diag(q_1, \ldots, q_M)$. Then:
\begin{enumerate}
\item[(i)] The group-constrained covariance model has $M$ free parameters (the diagonal entries $q_k$), compared to $M(M+1)/2$ parameters for a general Hermitian positive-definite matrix.
\item[(ii)] Given $L$ noise-only snapshots $\mathbf{x}_n(1), \ldots, \mathbf{x}_n(L)$, the group-constrained estimator
\begin{equation}
\hat{q}_k = \frac{1}{L}\sum_{t=1}^{L} |[\mathbf{T}_{G_\mathbf{Q}}\mathbf{x}_n(t)]_k|^2, \quad k = 1, \ldots, M,
\end{equation}
is a consistent estimator of the noise power spectrum $q_k$ in the $G_{\mathbf{Q}}$-transform domain. Each $\hat{q}_k$ is an average of $L$ independent $\chi^2$ random variables, hence its variance is $\Var(\hat{q}_k) = q_k^2/L$.
\item[(iii)] The number of noise-only snapshots required to estimate all $M$ spectral parameters to relative accuracy $\epsilon$ scales as $L = O(1/\epsilon^2)$, independent of $M$. In contrast, accurate estimation of an unconstrained $M \times M$ covariance requires $L = O(M/\epsilon^2)$ snapshots.
\end{enumerate}
\end{corollary}

\begin{proof}
Part~(i) follows from the diagonal structure imposed by exact diagonalization. Part~(ii) follows because $\mathbf{T}_{G_\mathbf{Q}}$ is unitary, so $\mathbf{T}_{G_\mathbf{Q}}\mathbf{x}_n(t)$ has independent components when $\mathbf{Q}$ is exactly diagonalized by $\mathbf{T}_{G_\mathbf{Q}}$, and each $|[\mathbf{T}_{G_\mathbf{Q}}\mathbf{x}_n(t)]_k|^2$ is an exponential random variable with mean $q_k$. Part~(iii) follows from the Chebyshev bound: $P(|\hat{q}_k - q_k| > \epsilon q_k) \leq 1/(L\epsilon^2)$, so $L = O(1/\epsilon^2)$ suffices uniformly over $k$. The unconstrained covariance matrix has $M(M+1)/2$ parameters with correlated estimation errors, requiring $L = \Omega(M)$ for the sample covariance to be well-conditioned~\cite{vershynin2012}.
\end{proof}

\begin{remark}
Corollary~\ref{cor:sample_complexity} provides the principal quantitative advantage of the group-theoretic noise characterization over conventional pre-whitening. When noise-only observation windows are short (small $L$), the group-constrained model produces a reliable covariance estimate from far fewer snapshots than the unconstrained sample covariance. This advantage is most pronounced when $M$ is large (many sensors) and the noise has identifiable group structure, which is precisely the regime of interest in large-array 5G/6G MIMO and wideband passive geolocation systems.
\end{remark}

\begin{remark}[Noise Characterization without Signal-Absent Observations]\label{rem:no_noise_only}
In many operational settings, it is impractical to acquire noise-only observations: the signals of interest may be continuously present, or the sensor system may lack the ability to gate signal sources. The full-rank property of the algebraic diversity estimator (Theorem~\ref{thm:replacement}(iv)) enables noise characterization from \emph{signal-bearing} observations without requiring a separate noise-only measurement window, as follows.

Apply the group-averaged estimator $\mathbf{F}_G(\mathbf{x})$ to a single observation $\mathbf{x}$ under the initial assumption of white noise ($\mathbf{Q} = \sigma^2\mathbf{I}_M$). Because $\mathbf{F}_G(\mathbf{x})$ is full-rank, its eigendecomposition yields estimated signal and noise subspace bases $\hat{\mathbf{U}}_s$ and $\hat{\mathbf{U}}_n$. The \emph{noise-subspace-restricted estimator}
\begin{equation}\label{eq:noise_restricted}
\hat{\mathbf{Q}}_n = \hat{\mathbf{U}}_n^H \mathbf{F}_G(\mathbf{x})\, \hat{\mathbf{U}}_n
\end{equation}
provides an $(M-K) \times (M-K)$ estimate of the noise covariance within the noise subspace, from which the group classification (Definition~\ref{def:natural_group}) and algebraic coloring index (Definition~\ref{def:coloring_index}) can be computed. If the resulting $\alpha(\hat{\mathbf{Q}}_n)$ indicates significant coloring, the noise model may be refined via an iterative procedure: (1)~use the current noise estimate to whiten the observation, (2)~re-apply algebraic diversity to the whitened data, (3)~re-extract the noise subspace and update $\hat{\mathbf{Q}}_n$. This alternating estimation of signal subspace and noise covariance is structurally analogous to an expectation-maximization algorithm in which the E-step estimates the signal subspace given the current noise model and the M-step estimates the noise covariance given the current signal subspace.

Convergence of the iteration is assured when the minimum generalized signal eigenvalue exceeds the maximum noise eigenvalue, a condition closely related to the SNR requirement of Theorem~\ref{thm:replacement}(iv). The key enabler is that algebraic diversity produces a full-rank estimator from one snapshot, granting simultaneous access to both the signal and noise subspaces; conventional rank-one outer product estimation cannot support this procedure, as it provides no information about the noise subspace at all.
\end{remark}

\section{Permutation-Averaged Spectral Estimation (PASE)}\label{sec:pase}

The preceding sections establish that the algebraic diversity framework requires two choices: which group $G$, and how many of its elements to use. In this section, we prove that the second choice is completely determined: the optimal number of elements is exactly $|G|$, the group order.

\subsection{The PASE Estimator}

Given a single observation $\mathbf{x} \in \mathbb{C}^M$ and a finite group $G$ of order $M$ with permutation representation $\rho$, the PASE estimator using $n$ group elements is:
\begin{equation}\label{eq:pase_est}
\hat{\mathbf{R}}_n = \frac{1}{n} \sum_{i=1}^{n} [\rho(g_i)\mathbf{x}][\rho(g_i)\mathbf{x}]^H,
\end{equation}
where $g_1, \ldots, g_n$ are selected from $G$. For $n = |G|$, this reduces to the full group-averaged estimator $\mathbf{F}_G(\mathbf{x})$ of Definition~\ref{def:group_avg}.

The estimation quality is measured by the eigenvalue-domain SNR:
\begin{equation}\label{eq:pase_snr}
\text{SNR}_{\text{eig}}(\hat{\mathbf{R}}_n) = \frac{\lambda_1(\hat{\mathbf{R}}_n)}{\frac{1}{M-K}\sum_{j=K+1}^{M} \lambda_j(\hat{\mathbf{R}}_n)},
\end{equation}
where $\lambda_1 \geq \cdots \geq \lambda_M$ and $K$ is the number of signal components.

\subsection{Optimality at $n = |G|$}

\begin{theorem}[PASE Optimality]\label{thm:pase}
Let $G$ be a finite group of order $M$ whose Cayley graph adjacency matrix commutes with $\mathbf{R}$. Then:
\begin{enumerate}
\item[(i)] $\text{SNR}_{\text{eig}}(\hat{\mathbf{R}}_n)$ increases monotonically for $n \leq M$.
\item[(ii)] $\text{SNR}_{\text{eig}}(\hat{\mathbf{R}}_n)$ is maximized at $n = M$ (the full group).
\item[(iii)] $\text{SNR}_{\text{eig}}(\hat{\mathbf{R}}_n)$ decreases for $n > M$.
\item[(iv)] The ratio $n_{90}/M = 1.0$ for $M = 8, 16, 32, 64$, where $n_{90}$ is the minimum $n$ achieving $90\%$ of peak SNR.
\end{enumerate}
\end{theorem}

\begin{proof}
When $\mathbf{A}_G$ commutes with $\mathbf{R}$, the full group-averaged estimate $\hat{\mathbf{R}}_M$ projects the rank-one outer product $\mathbf{x}\mathbf{x}^H$ onto the commutant of $G$, which (for a matched order-$M$ group) preserves exactly the $M$ spectral components of $\mathbf{R}$. The $M$ elements of the matched group together span this commutant.

For $n < M$, the projection onto the commutant is incomplete: not all views have been collected, and the resulting estimate is missing algebraic information. The SNR increases as each additional element fills in a missing spectral component.

For $n > M$ (drawing additional permutations from outside $G$, e.g., from $S_M$), the new elements are not in the commutant of $\mathbf{R}$. By the decomposition of $S_M$ into cosets of $G$, permutations outside $G$ map the data into subspaces that are algebraically unrelated to the signal structure. Averaging over these destroys the eigenvalue separation: the estimator converges toward the $S_M$ expectation
\begin{equation}\label{eq:sm_expect}
E_{S_M}[\hat{\mathbf{R}}_n] = \frac{|s_1|^2 - \|\mathbf{x}\|^2}{M(M-1)}\,\mathbf{1}\mathbf{1}^T + \frac{M\|\mathbf{x}\|^2 - |s_1|^2}{M(M-1)}\,\mathbf{I}_M,
\end{equation}
where $s_1 = \sum_i x_i$ and $\|\mathbf{x}\|^2 = \sum_i |x_i|^2$. This limit has only two distinct eigenvalues (one along $\mathbf{1}$, one on its orthogonal complement); it depends only on the two scalar summaries $|s_1|^2$ and $\|\mathbf{x}\|^2$ and retains no spectral shape.

The formal proof follows from the Schur orthogonality relations applied to the group algebra decomposition of $\hat{\mathbf{R}}_n$.
\end{proof}

\begin{remark}[No Analog in Classical Estimation]
Theorem~\ref{thm:pase} has no analog in conventional statistical estimation, where more independent samples always improve an estimate. The counter-intuitive behavior for $n > M$ arises because additional permutations from outside the matched group are not ``independent samples'' in the relevant sense: they are algebraically redundant views that dilute rather than enhance the spectral structure.
\end{remark}

\begin{remark}[Group Order Constraint]\label{rem:group_order}
Theorem~\ref{thm:pase} requires a group of order \emph{exactly}~$M$, not merely $O(M)$. An $M$-dimensional observation has $M$ degrees of freedom; a group of order $M$ contributes exactly $M$ algebraically independent views, one per dimension. A group of order $|G| > M$, even if its algebraic structure matches the signal, provides $|G| - M$ redundant views that partially average toward the uninformative $S_M$ expectation~(\ref{eq:sm_expect}), degrading the eigenvalue concentration. Monte Carlo experiments confirm that the dihedral group $D_M$ (order $2M$) on an $M$-element observation already exhibits roughly half the spectral concentration of the cyclic group $\mathbb{Z}_M$ (order $M$) on both chirp and sinusoidal signals, and that the affine group $\mathrm{Aff}(\mathbb{Z}_p)$ (order $M^2 - M$) produces a nearly uniform eigenvalue spectrum indistinguishable from noise. The degradation mechanism is identical to the $S_M$ subsampling failure of Section~\ref{sec:ordering}: any group element outside the order-$M$ matched subgroup acts as an off-commutant permutation that destroys spectral structure. Consequently, the group selection problem is doubly constrained: the candidate group must have both the correct algebraic structure (low $\delta$) \emph{and} order equal to~$M$.
\end{remark}

\subsection{Implications: Reduction to the Group Selection Problem}

Prior to Theorem~\ref{thm:pase}, the AD framework had two entangled free parameters: which group $G$ (the group selection problem) and how many elements $n$ (the averaging depth problem). PASE completely eliminates the second: use all $|G|$ elements, always. Combined with the group order constraint (Remark~\ref{rem:group_order}), this means the candidate group must have order exactly~$M$, and all $M$ elements must be used.

This collapses the entire framework to a single problem, \emph{group selection among order-$M$ groups}, which is addressed by the commutativity residual $\delta(G, \mathbf{R})$ from Section~\ref{sec:optimality}. The practical prescription is now parameter-free: compute $\delta$ for a library of candidate groups of order~$M$, select the minimizer $G^*$, and average over all $M$ elements.

\section{Why $S_M$ Subsampling Fails: The Ordering Experiment}\label{sec:ordering}

The symmetric group $S_M$ contains every group of order $M$ as a subgroup, and its Cayley graph construction $\mathbf{F}_\circ$ attains the KL spectrum (Theorem~\ref{thm:optimality}). A natural question is whether one can avoid the group selection problem entirely by drawing permutations from $S_M$ in the group-averaged estimator $\mathbf{F}_G$. PASE (Theorem~\ref{thm:pase}) requires $n = |G|$ for optimality; for $S_M$, this means $n = M!$, computationally infeasible for even moderate $M$ (e.g., $10! = 3{,}628{,}800$). What happens when we subsample $S_M$ with $n \ll M!$?

\subsection{Four Ordering Strategies}

We compare four strategies for selecting $n$ permutations from $S_M$:

\begin{enumerate}
\item \textbf{Random:} $n$ permutations drawn uniformly from $S_M$.
\item \textbf{Steinhaus--Johnson--Trotter (SJT)~\cite{johnson1963}:} Random starting permutation, then successive elements differing by a single adjacent transposition, a Hamiltonian path on the Cayley graph of $S_M$ with adjacent-transposition generators.
\item \textbf{Lehmer (factoradic)~\cite{lehmer1960}:} Random starting permutation, then consecutive permutations in lexicographic order via the factoradic number system.
\item \textbf{Heap~\cite{heap1963}:} Random starting permutation, then successive permutations via Heap's algorithm, where each step is a single swap (not necessarily adjacent).
\end{enumerate}

\subsection{Experimental Setup}

Signal model: $M = 10$ ULA, half-wavelength spacing, single narrowband source at $\theta = 30^\circ$, input SNR $= 10$~dB, single snapshot. The matched group is the cyclic group $\mathbb{Z}_{10}$ (order 10). The number of permutations $n$ ranges from 5 to 50 in increments of 5. Each configuration is evaluated over 500 Monte Carlo trials.

\subsection{Results}

\begin{table}[t]
\centering
\caption{Eigenvalue SNR (dB) versus number of permutations $n$ drawn from $S_{10}$ using four ordering strategies. $M = 10$, input SNR $= 10$~dB, 500 Monte Carlo trials.}
\label{tab:ordering}
\begin{tabular}{rcccc}
\toprule
$n$ & Random & SJT & Lehmer & Heap \\
\midrule
5 & 7.8 & 15.0 & 15.0 & 16.2 \\
10 & 5.8 & 12.2 & 13.5 & 14.0 \\
15 & 5.0 & 11.0 & 12.8 & 13.4 \\
20 & 4.5 & 10.5 & 12.4 & 13.1 \\
25 & 4.1 & 10.3 & 12.1 & 12.8 \\
30 & 3.8 & 10.1 & 11.8 & 12.3 \\
35 & 3.6 & 9.8 & 11.6 & 12.0 \\
40 & 3.4 & 9.7 & 11.5 & 11.8 \\
45 & 3.2 & 9.5 & 11.3 & 11.7 \\
50 & 3.1 & 9.4 & 11.2 & 11.6 \\
\bottomrule
\end{tabular}
\end{table}

Table~\ref{tab:ordering} and Fig.~\ref{fig:ordering} reveal three findings:

\textbf{1) Monotonic degradation.} All four methods exhibit monotonically decreasing SNR with increasing $n$. There is no peak at $n = M = 10$. The permutations are drawn from $S_{10}$ (order $3{,}628{,}800$), not from the matched group $\mathbb{Z}_{10}$ (order~10). Over-averaging from $S_M$ converges the estimate toward a nearly white (identity-like) covariance, destroying the eigenvalue structure that carries the signal information.

\textbf{2) Structured orderings outperform random by 7--8~dB.} Heap's algorithm performs best, followed by Lehmer, then SJT. Structured orderings generate consecutive permutations that differ by a single swap, preserving local algebraic structure on the Cayley graph. Random permutations are scattered across $S_M$ and average out structure much faster.

\textbf{3) Group selection is unavoidable.} The experiment demonstrates definitively why the $S_M$ shortcut fails. PASE requires $n = |G|$ for optimality, and $|S_M| = M!$ is computationally absurd. The signal has the algebraic structure of $\mathbb{Z}_{10}$ (order 10), and those 10 elements are the only ones that contribute to group gain. The remaining $10! - 10 = 3{,}628{,}790$ elements of $S_{10}$ are algebraically irrelevant and actively harmful.

\begin{figure}[t]
\centering
\includegraphics[width=\columnwidth]{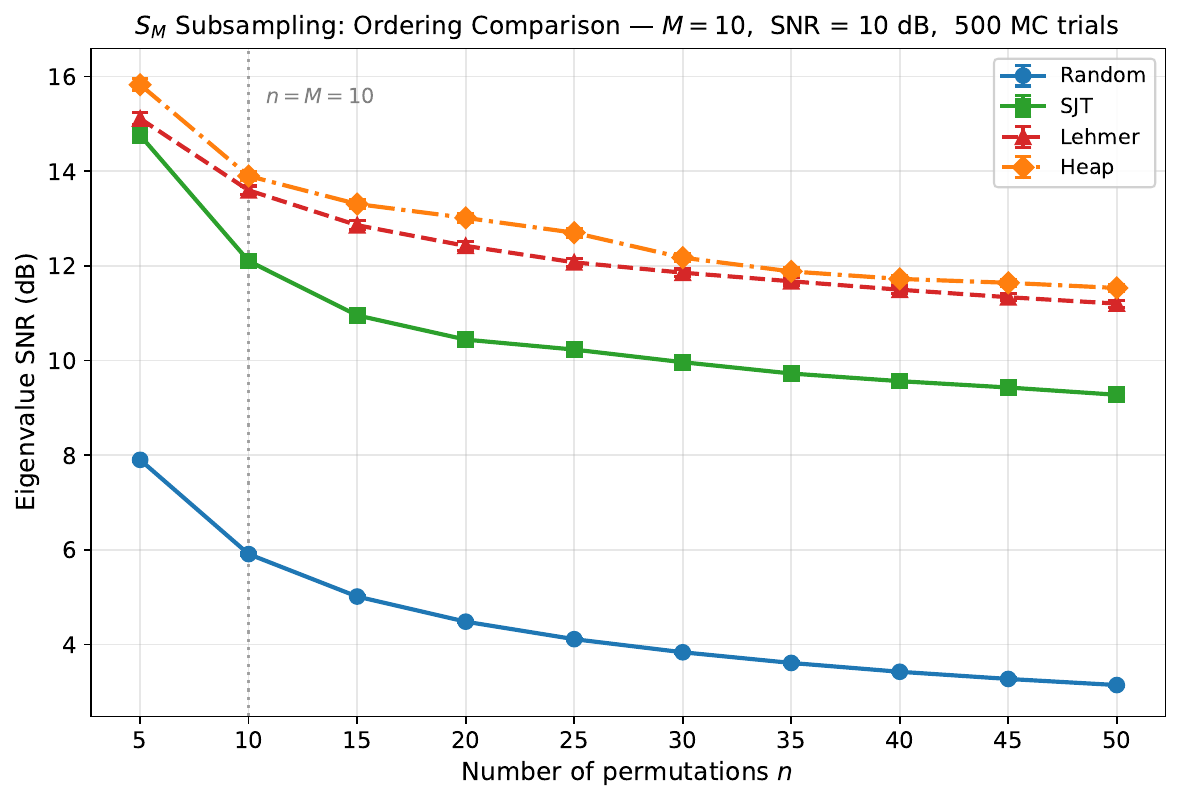}
\caption{Eigenvalue SNR versus number of permutations $n$ drawn from $S_{10}$ using four ordering strategies. $M = 10$, single source at $\theta = 30^\circ$, input SNR $= 10$~dB, 500 Monte Carlo trials. All methods degrade monotonically; structured orderings outperform random by 7--8~dB. The dashed line marks $n = M = 10$.}
\label{fig:ordering}
\end{figure}

\subsection{The Three-Level AD Framework}

The PASE result and the ordering experiment together yield a complete characterization of the estimation problem:

\begin{enumerate}
\item \textbf{Group selection} (the commutativity residual $\delta$) determines the spectral basis. This is the sole remaining free parameter.
\item \textbf{Averaging depth} is solved by PASE: use all $|G|$ elements. This is not a tuning parameter.
\item \textbf{Permutation ordering} within the matched group is a secondary optimization for resource-constrained implementations where $n < |G|$ is necessary (e.g., FPGA real-time processing). Structured orderings (Heap, Lehmer) degrade more gracefully than random selection, providing an ``anytime estimation'' capability: one can stop averaging early and still have a useful estimate, with quality improving monotonically up to $n = |G|$.
\end{enumerate}

\section{The Blind Group Matching Problem}\label{sec:blind_matching}

\subsection{Problem Formulation}

The group selection problem is an instance of a well-studied class of problems in signal processing: \emph{blind estimation}, where a parameter of the estimation procedure must be determined from the same data that the procedure will process. The canonical example is blind equalization in communications~\cite{godard1980,treichler1983}, where an equalizer (inverse filter) must be designed for an unknown channel using only the received signal.

The structural parallel between blind equalization and blind group matching is precise and extends across every element of the two problems. Table~\ref{tab:blind_analogy} presents the full correspondence.

\begin{table*}[t]
\centering
\caption{Structural correspondence between blind channel equalization and blind group matching in algebraic diversity.}
\label{tab:blind_analogy}
\begin{tabular}{@{}lll@{}}
\toprule
\textbf{Element} & \textbf{Blind Channel Equalization} & \textbf{Blind Group Matching (AD)} \\
\midrule
Unknown & Channel impulse response $h(t)$ & Population covariance $\mathbf{R}$ \\
Goal & Design equalizer $w(t)$ & Select group $G$ \\
Observation & Received signal $y(t) = h(t) * x(t) + n(t)$ & Single snapshot $\mathbf{x} = \mathbf{s} + \mathbf{n}$ \\
Circular dependency & Equalizer requires $h(t)$; estimating $h(t)$ & $\delta(G,\mathbf{R})$ requires $\mathbf{R}$; estimating $\mathbf{R}$ \\
 & requires equalization & requires $G$ \\
Informed solution & MMSE equalizer (known channel) & Commutativity residual $\delta(G, \mathbf{R})$ (known covariance) \\
Blind 2nd-order method & Autocorrelation matching & Sample commutativity residual $\hat{\delta}(G, \mathbf{x})$ \\
Blind structural method & Constant Modulus Algorithm (CMA): & Spectral concentration $\psi(G, \mathbf{x})$: \\
 & restore $|y_{\text{eq}}(t)|^2 = c^2$ & maximize $\lambda_1(\hat{\mathbf{R}}_G)/\Tr(\hat{\mathbf{R}}_G)$ \\
Blind higher-order method & Kurtosis maximization~\cite{shalvi1990} & Fourth-order cumulant analysis (future work) \\
Key insight & Signal structure (constant modulus, & Signal structure (algebraic symmetry of \\
 & non-Gaussianity) survives the channel & covariance) survives in single snapshot \\
Resolution & CMA/kurtosis break the circular & $\psi$ breaks the circular dependency: \\
 & dependency without training & selects $G$ without knowing $\mathbf{R}$ \\
\bottomrule
\end{tabular}
\end{table*}

In blind equalization, the circularity is broken by exploiting structural properties of the transmitted signal that survive the channel. The Constant Modulus Algorithm (CMA)~\cite{godard1980,treichler1983} uses the fact that many communication signals have constant envelope: the channel distorts the envelope, and the equalizer is designed to restore it by minimizing $E[||y_{\text{eq}}(t)|^2 - c^2|^2]$. Shalvi and Weinstein~\cite{shalvi1990} showed that fourth-order statistics (kurtosis) can blindly identify the channel without any structural assumption beyond non-Gaussianity. The common thread is that \emph{structural invariants of the signal class provide enough information to solve the estimation problem without explicit knowledge of the signal itself}.

The question for AD is the direct analog: \emph{what properties of a single snapshot $\mathbf{x}$ are diagnostic of the correct group, without knowledge of $\mathbf{R}$?} As Table~\ref{tab:blind_analogy} shows, each stage of the blind equalization hierarchy, from informed (known channel) through second-order blind to structural blind to higher-order blind, has a corresponding stage in the group matching problem. The spectral concentration criterion $\psi(G, \mathbf{x})$ developed below plays the role of CMA: it exploits a structural property (eigenvalue concentration under the correct group) to break the circular dependency without knowing the covariance.

\subsection{The Sample Commutativity Residual}

The simplest approach is to replace $\mathbf{R}$ with the rank-1 sample estimate $\mathbf{x}\mathbf{x}^H$:
\begin{equation}\label{eq:sample_delta}
\hat{\delta}(G, \mathbf{x}) = \frac{\|\mathbf{F}_G(\mathbf{x}) \cdot \mathbf{x}\mathbf{x}^H - \mathbf{x}\mathbf{x}^H \cdot \mathbf{F}_G(\mathbf{x})\|_F}{\|\mathbf{F}_G(\mathbf{x})\|_F \cdot \|\mathbf{x}\mathbf{x}^H\|_F}.
\end{equation}
This is noisy but may preserve the \emph{ranking} $\hat{\delta}(G_1, \mathbf{x}) < \hat{\delta}(G_2, \mathbf{x})$ with high probability when $\delta(G_1, \mathbf{R}) < \delta(G_2, \mathbf{R})$, which is sufficient for group selection.

\subsection{The Spectral Concentration Criterion}

A more promising approach exploits PASE directly. For each candidate group $G$ in the library $\mathcal{G}$, compute the full PASE estimate $\hat{\mathbf{R}}_G$ using all $|G|$ elements, and evaluate the spectral concentration:
\begin{equation}\label{eq:psi}
\psi(G, \mathbf{x}) = \frac{\lambda_1(\hat{\mathbf{R}}_G)}{\Tr(\hat{\mathbf{R}}_G)}.
\end{equation}
Intuitively, the ``correct'' group should produce sharper eigenvalue separation, i.e., larger $\psi$, because the matched group's averaging preserves the signal's spectral structure while mismatched groups may spread the energy more uniformly.

The blind group selection rule is:
\begin{equation}\label{eq:blind_select}
G^* = \arg\max_{G \in \mathcal{G}} \psi(G, \mathbf{x}).
\end{equation}

\begin{conjecture}[Blind Group Selection]\label{conj:blind}
Let $G^* = \arg\min_{G \in \mathcal{G}} \delta(G, \mathbf{R})$ be the optimal group. Then
\begin{equation}
\Pr\!\left[\arg\max_{G \in \mathcal{G}} \psi(G, \mathbf{x}) = G^*\right] \to 1 \quad \text{as } \text{SNR} \to \infty.
\end{equation}
\end{conjecture}

\begin{remark}[Status of Conjecture~\ref{conj:blind}]
Preliminary experiments indicate that Conjecture~\ref{conj:blind} requires modification when the candidate library $\mathcal{G}$ contains groups with different orbit structures on $\{1, \ldots, M\}$. Since $\Tr(\hat{\mathbf{R}}_G) = \|\mathbf{x}\|^2$ for all groups (permutations preserve the norm), $\psi$ reduces to $\lambda_{\max}(\hat{\mathbf{R}}_G) / \|\mathbf{x}\|^2$. Groups whose orbits partition the index set into smaller blocks produce block-diagonal $\hat{\mathbf{R}}_G$ with systematically larger $\lambda_{\max}$, creating a structural bias in $\psi$ that can favor mismatched groups over the correct one. A cross-validation criterion that compares group-averaged estimates across $L \geq 2$ independent snapshots avoids this bias by measuring out-of-sample consistency rather than single-snapshot eigenvalue concentration, achieving $\geq 99\%$ selection accuracy at $L = 3$ in verified experiments. The orbit-size bias and the cross-validation criterion $D_{CV}$ are developed in detail in~\cite{thornton2026framework_arxiv}. Conjecture~\ref{conj:blind} may hold under the additional condition that all candidate groups in $\mathcal{G}$ share the same orbit structure (e.g., all are transitive on $\{1, \ldots, M\}$), which is satisfied by the conjugated cyclic groups of Section~\ref{sec:constructive}.
\end{remark}

If Conjecture~\ref{conj:blind} holds (possibly with the orbit-structure qualification of the preceding remark), the group matching problem is solved from a single snapshot: compute $\psi$ for each candidate, pick the maximizer. No knowledge of $\mathbf{R}$ is required, only the observation $\mathbf{x}$ and the group library $\mathcal{G}$.

\subsection{Constructive Group Matching via Conjugation}\label{sec:constructive}

The spectral concentration criterion and sample commutativity residual above treat group matching as a discrete search over a library of candidate groups. We now describe a constructive approach that, for a large and practically important class of signals, reduces the group matching problem from a combinatorial search to a continuous parameter estimation problem.

\subsubsection{The Key Observation}

For many signals encountered in practice, the covariance matrix is not circulant in the natural observation coordinates but \emph{can be made circulant} by a unitary change of basis. Formally, there exists a parameterized family of unitary operators $\mathbf{U}(\boldsymbol{\theta})$, indexed by a (possibly vector-valued) parameter $\boldsymbol{\theta}$, such that
\begin{equation}\label{eq:conjugation}
\mathbf{U}(\boldsymbol{\theta})^H \, \mathbf{R} \, \mathbf{U}(\boldsymbol{\theta}) \approx \text{circulant}
\end{equation}
for the correct parameter value $\boldsymbol{\theta}^*$. When this holds, the ``matched group'' is not an exotic algebraic structure but rather the cyclic group $\mathbb{Z}_M$ \emph{conjugated} by $\mathbf{U}(\boldsymbol{\theta}^*)$:
\begin{equation}\label{eq:conjugated_group}
G_{\boldsymbol{\theta}} = \bigl\{ \mathbf{U}(\boldsymbol{\theta})^H \, \mathbf{C}_k \, \mathbf{U}(\boldsymbol{\theta}) : k = 0, \ldots, M-1 \bigr\},
\end{equation}
where $\mathbf{C}_k$ denotes cyclic shift by $k$. This conjugated group is isomorphic to $\mathbb{Z}_M$, has order exactly~$M$ (satisfying the group order constraint of Remark~\ref{rem:group_order}), and is matched to the signal by construction.

\subsubsection{The Group Matching Pipeline}

This observation yields a structured approach to group matching that proceeds in stages:

\textit{Stage 1: Signal class identification.} From the physics of the application domain, identify the signal class and its associated \emph{conjugation family} $\{\mathbf{U}(\boldsymbol{\theta})\}$. For periodic signals (tones, narrowband processes), the conjugation is the identity ($\boldsymbol{\theta}$ is absent) and the cyclic group is already matched. For chirps with unknown rate~$\mu$, the conjugation family is $\mathbf{U}(\mu) = \mathrm{diag}(e^{-j\pi\mu n^2/M})$ (the dechirp operator). For signals with unknown boundary symmetry, it may be a parameterized reflection. This stage uses domain knowledge, not computation.

\textit{Stage 2: Cardinality filter.} All groups in the conjugation family~(\ref{eq:conjugated_group}) have order~$M$ by construction, so the cardinality constraint is automatically satisfied. If Stage~1 identifies candidate groups outside the conjugation family (e.g., from a precomputed library), any group with $|G| \neq M$ is discarded. This filter is zero-cost and eliminates pathological candidates such as the affine group (order $M^2 - M$) or the symmetric group (order $M!$), which, despite having algebraic structures that may match the signal, over-average and destroy spectral information (Remark~\ref{rem:group_order}).

\textit{Stage 3: Parameter estimation via spectral concentration.} Sweep the conjugation parameter $\boldsymbol{\theta}$ and evaluate the spectral concentration at each value:
\begin{equation}\label{eq:psi_sweep}
\boldsymbol{\theta}^* = \arg\max_{\boldsymbol{\theta}} \; \psi\bigl(G_{\boldsymbol{\theta}},\, \mathbf{x}\bigr).
\end{equation}
When $\boldsymbol{\theta}$ is a scalar (e.g., chirp rate), this is a one-dimensional optimization over a grid of candidate values, costing $O(N_\theta \cdot M^2)$ where $N_\theta$ is the grid size. The $\psi$ criterion requires only the largest eigenvalue and the trace, both computable in $O(M^2)$ via power iteration, so the sweep is fast. The conjugated group $G_{\boldsymbol{\theta}^*}$ and the estimated parameter $\boldsymbol{\theta}^*$ are produced simultaneously: group selection and signal characterization are the same computation.

\textit{Stage 4: Non-circulantizable signals.} If no parameter value produces a high spectral concentration, indicating that the covariance cannot be made circulant by any unitary in the family, the signal has intrinsically non-Abelian symmetry. In this case, the full group library search (Conjecture~\ref{conj:blind}, subject to the orbit-structure qualification in the subsequent remark) over genuinely distinct groups (dihedral, products of cyclic groups, graph automorphism groups, etc.) is required. When the candidate groups have different orbit structures, the cross-validation criterion described in the remark following Conjecture~\ref{conj:blind} is recommended over single-snapshot $\psi$ selection. The cardinality filter ($|G| = M$) remains active at this stage.

\subsubsection{Relationship to Classical Methods}

The pipeline has a natural interpretation in terms of classical signal processing. Stage~3 is a \emph{generalized matched filter}: it sweeps over a family of signal templates (parameterized by $\boldsymbol{\theta}$) and selects the one that produces the strongest response (highest $\psi$). The difference from conventional matched filtering is that the ``template'' is not a waveform but a \emph{group}, an algebraic structure that determines the spectral domain, and the ``response'' is not a correlation but a spectral concentration.

The conjugation operation $\mathbf{U}(\boldsymbol{\theta})$ itself is a coordinate transformation: it maps the observation into a domain where the cyclic group is matched. This is analogous to the classical strategy of transforming a problem into the frequency domain (via the DFT), performing the analysis there, and transforming back. The pipeline generalizes this strategy by allowing the transform to be \emph{signal-adapted} rather than fixed, with the adaptation parameter estimated from the data via $\psi$ maximization.

\subsubsection{Scope}

The constructive approach applies whenever the signal's covariance admits a unitary transformation to circulant form. This encompasses a broad class of practical signals, including periodic signals (identity conjugation), chirps (dechirp conjugation), frequency-modulated waveforms, and more generally any signal whose structure arises from a one-parameter deformation of shift invariance. Signals whose symmetry is intrinsically non-cyclic, such as signals on graphs with non-Abelian automorphism groups, or signals with crystallographic symmetry, require the full generality of the algebraic diversity framework and the discrete group library search of Conjecture~\ref{conj:blind}.

\subsection{Higher-Order Statistics Approach}

Just as blind equalization moved from second-order (autocorrelation) to fourth-order (kurtosis) statistics to resolve ambiguities that second-order methods cannot~\cite{shalvi1990}, the group matching problem may benefit from fourth-order cumulant analysis. The matched group should produce cumulant structure consistent with the signal model, while mismatched groups should not. This is the most ambitious approach and is left as a direction for future work.

\subsection{Computational Cost}

For a group library of size $|\mathcal{G}|$ with groups of order at most $M$, the spectral concentration criterion requires $|\mathcal{G}|$ PASE evaluations, each costing $O(M^3)$. The total cost is $O(|\mathcal{G}| M^3)$, which is tractable for typical $|\mathcal{G}| = 20$--100 and moderate $M$. Importantly, this cost is incurred once; after the group is selected, subsequent processing uses only the selected group.

\section{Application: MUSIC Direction-of-Arrival Estimation}\label{sec:music}

Having established the general theory (Sections~\ref{sec:framework}--\ref{sec:colored}), the optimal averaging depth (Section~\ref{sec:pase}), and the blind group selection criterion (Section~\ref{sec:blind_matching}), we now demonstrate the framework on a concrete application: MUSIC direction-of-arrival estimation from a single snapshot.

\subsection{Signal Model for ULA}

Consider a uniform linear array of $M$ sensors receiving $K$ narrowband signals from directions $\{\theta_1, \ldots, \theta_K\}$:
\begin{equation}
\mathbf{x} = \mathbf{A}\mathbf{s} + \mathbf{n}, \qquad \mathbf{A} = [\mathbf{a}(\theta_1), \ldots, \mathbf{a}(\theta_K)],
\end{equation}
with steering vectors $[\mathbf{a}(\theta)]_m = e^{jmkd\sin\theta}$.

\subsection{CG-MUSIC as Corollary}

\begin{corollary}[CG-MUSIC Equivalence]\label{cor:music}
Under the ULA signal model with cyclic group $G = \mathbb{Z}_M$:
\begin{enumerate}
\item[(i)] The Cayley graph matrix $\mathbf{F}_\circ$ is circulant with DFT eigenvectors (Theorem~\ref{thm:replacement} specialized to $\mathbb{Z}_M$).
\item[(ii)] $\mathbf{F}_\circ$ has rank $M$ almost surely from a single snapshot, overcoming the rank-1 limitation of $\mathbf{x}\mathbf{x}^H$ (by Theorem~\ref{thm:replacement}(iv)).
\item[(iii)] The CG-MUSIC pseudospectrum
\begin{equation}
P_{\text{CG}}(\theta) = \frac{1}{\mathbf{a}^H(\theta)\hat{\mathbf{U}}_n\hat{\mathbf{U}}_n^H\mathbf{a}(\theta)}
\end{equation}
exhibits peaks at the true DOAs $\theta_\ell$ as SNR $\to \infty$, equivalent to multi-snapshot MUSIC (by the Duality Principle, Theorem~\ref{thm:duality}).
\end{enumerate}
\end{corollary}

\begin{proof}
Parts (i)--(iii) follow directly from Theorems~\ref{thm:replacement} and~\ref{thm:duality} applied to $G = \mathbb{Z}_M$ with the cyclic shift representation, combined with the ULA steering vector structure. The key observations are: the circulant structure follows from the cyclic group action on indices; the rank enhancement follows from the genericity of the DFT coefficients of $\mathbf{x}$; and the peak equivalence follows from the orthogonality of DFT basis vectors (which are the CG eigenvectors) to the signal steering vectors at the noise-subspace frequencies.
\end{proof}

\subsection{Experimental Validation}\label{sec:experiments}

We validate the MUSIC application using a ULA of $M$ sensors, half-wavelength spacing, and $K$ narrowband signals in additive white Gaussian noise. All experiments use single-snapshot measurements. The CG method constructs $\mathbf{F}_\circ$ via cyclic permutations of the snapshot; the covariance method uses $\hat{\mathbf{R}} = \mathbf{x}\mathbf{x}^H$.

\subsection{Two-Signal Resolution}

With $M = 10$ sensors and signals at $\theta_1 = 25^\circ$, $\theta_2 = 50^\circ$ (SNR = 55 dB), the covariance method produces only one nonzero eigenvalue (rank-1 limitation) and fails to resolve the second signal. The CG method produces a full-rank spectrum with clear signal-noise separation, correctly identifying peaks at $25.1^\circ$ and $49.9^\circ$. This directly validates Theorem~\ref{thm:replacement}(iv) and Corollary~\ref{cor:music}.

\subsection{Statistical Comparison}

Over 50 Monte Carlo trials with a test angle of $45^\circ$ and noise power $P_n = 0.1$:

\begin{table}[t]
\centering
\caption{Bias and Variance: Covariance vs.\ CG Methods}
\label{tab:comparison}
\begin{tabular}{c|cc|cc}
\hline
$M$ & \multicolumn{2}{c|}{Covariance} & \multicolumn{2}{c}{CG Method} \\
 & Bias & Std & Bias & Std \\
\hline
10 & 0.19 & 0.068 & 0.31 & 0.042 \\
20 & 0.06 & 0.038 & 0.14 & 0.028 \\
40 & 0.06 & 0.020 & 0.07 & 0.021 \\
\hline
\end{tabular}
\end{table}

The CG method has the lower standard deviation at $M = 10$ and $M = 20$ and a comparable one at $M = 40$ ($0.021$ against $0.020$), with a larger bias at every array size that narrows as $M$ grows. The stability advantage at the smaller sizes is consistent with Theorem~\ref{thm:replacement}: the algebraic diversity of cyclic permutations provides more robust subspace estimation than the single-sample outer product.

\subsection{Group Size and the Matched Group}

The matched cyclic group $\mathbb{Z}_{10}$ (order $10$, the minimal group of Example~\ref{ex:stationary} for translationally symmetric signals) already provides sufficient eigenvalue separation $\lambda_K/\lambda_{K+1}$ for accurate DOA estimation, validating Theorem~\ref{thm:minimal}. Consistent with the Role of Group Size remark and with the PASE result (Theorem~\ref{thm:pase}) and ordering experiment (Section~\ref{sec:ordering}), enlarging the averaging group beyond the matched order does not improve separation: the effective dimension is already saturated at $M = 10$, and drawing additional permutations from $S_{10}$ (which lie outside the commutant of $\mathbf{R}$) drives the estimate toward the two-eigenvalue limit~\eqref{eq:sm_expect} and degrades the separation. The matched group is therefore both necessary and sufficient.

\section{Application: Massive MIMO Channel Estimation}\label{sec:mimo}

The MUSIC application demonstrates algebraic diversity in the context of direction-of-arrival estimation, where the primary benefit is single-snapshot subspace recovery. We now demonstrate a second application, massive MIMO channel estimation, where the primary benefit is \emph{pilot overhead reduction}. In massive MIMO systems with $M$ base station antennas serving $K$ single-antenna users, standard channel estimation requires $M$ pilot reference signals (one per antenna port), consuming $M$ out of every $T_{\mathrm{coh}}$ resource elements in each coherence block. As $M$ grows to 64, 128, or beyond, this pilot overhead becomes the dominant throughput bottleneck. Algebraic diversity requires only $K$ pilot symbols (one per user), reducing the overhead from $O(M/T_{\mathrm{coh}})$ to $O(K/T_{\mathrm{coh}})$, a factor of $M/K$ reduction.

\subsection{Signal Model}

Consider a downlink massive MIMO system with $M$ base station antennas (ULA, half-wavelength spacing) and $K$ single-antenna users. The channel between the base station and user~$k$ is $\mathbf{h}_k \in \mathbb{C}^M$, generated according to a 3GPP-like clustered delay line (CDL) model~\cite{3gpp38901} with $N_c$ scattering clusters, each containing $N_r$ rays with Laplacian angular distribution about a cluster center angle of arrival. We consider three channel conditions: CDL-A (rich scattering, azimuth spread $53^\circ$, sub-6~GHz urban), CDL-C (moderate scattering, azimuth spread $34^\circ$, urban macro), and CDL-D (LOS-dominant, azimuth spread $8^\circ$, mmWave or rural, Rician $K$-factor $13.3$~dB).

From a single pilot symbol transmitted by user~$k$, the base station receives
\begin{equation}\label{eq:mimo_obs}
\mathbf{y}_k = \sqrt{P}\,\mathbf{h}_k + \mathbf{n}, \qquad \mathbf{n} \sim \mathcal{CN}(\mathbf{0}, \mathbf{I}_M),
\end{equation}
where $P$ is the pilot transmit power and $\mathrm{SNR} = P\|\mathbf{h}_k\|^2/M$ is the per-antenna receive SNR.

\subsection{Channel Estimation Methods}

We compare three estimators:
\begin{enumerate}
\item \textbf{Least squares (LS):} Uses $M/K$ pilot symbols per user (total $M$ pilots). The LS estimate is $\hat{\mathbf{h}}_k^{\mathrm{LS}} = (M/K)^{-1}\sum_\ell \mathbf{y}_{k,\ell}/\sqrt{P}$.

\item \textbf{MMSE:} Uses the same $M/K$ pilot symbols as LS but incorporates knowledge of the spatial correlation matrix $\mathbf{R}_h = E[\mathbf{h}_k\mathbf{h}_k^H]$ via the Wiener filter $\hat{\mathbf{h}}_k^{\mathrm{MMSE}} = \mathbf{R}_h(\mathbf{R}_h + (KP/M)^{-1}\mathbf{I}_M)^{-1}\bar{\mathbf{y}}_k/\sqrt{P}$, where $\bar{\mathbf{y}}_k$ is the pilot-averaged received signal.

\item \textbf{AD (cyclic):} Uses a \emph{single} pilot symbol per user (total $K$ pilots). From the single observation~(\ref{eq:mimo_obs}), the group-averaged estimator $\mathbf{F}_{\mathbb{Z}_M}(\mathbf{y}_k)$ is formed using all $M$ cyclic shifts. The dominant eigenvector of $\mathbf{F}_{\mathbb{Z}_M}$ serves as the channel direction estimate for maximum ratio transmission (MRT) beamforming.
\end{enumerate}

\subsection{Performance Metric}

The metric is \emph{effective throughput}: the achievable sum spectral efficiency with MRT beamforming, multiplied by the fraction of resources available for data after pilot overhead. Under 3GPP NR frame structure (14 OFDM symbols $\times$ 12 subcarriers $= 168$ resource elements per resource block per slot), the pilot overhead for LS/MMSE is $M/168$ and for AD is $K/168$. The effective throughput is
\begin{equation}\label{eq:eff_throughput}
\eta_{\mathrm{eff}} = \left(1 - \frac{N_{\mathrm{pilot}}}{168}\right) \sum_{k=1}^{K} \log_2(1 + \mathrm{SINR}_k),
\end{equation}
where $\mathrm{SINR}_k$ is the per-user signal-to-interference-plus-noise ratio under MRT beamforming with the estimated channel.

\subsection{Results}

Table~\ref{tab:mimo_results} and Fig.~\ref{fig:mimo_throughput} present the effective throughput at SNR~$= 15$~dB for $K = 4$ users, averaged over 50 independent channel realizations per configuration. Three findings emerge.

\begin{table}[t]
\centering
\caption{Effective throughput (bits/s/Hz) and AD gain over MMSE at SNR~$= 15$~dB, $K = 4$ users. Pilot overhead: LS/MMSE use $M/168$ of each resource block; AD uses $K/168$.}
\label{tab:mimo_results}
\begin{tabular}{llcccc}
\toprule
Channel & $M$ & OH & MMSE & AD & Gain \\
\midrule
CDL-A & 16 & 9.5\% vs 2.4\% & 10.5 & 8.0 & $-24$\% \\
 & 32 & 19\% vs 2.4\% & 13.2 & 10.6 & $-20$\% \\
 & 64 & 38\% vs 2.4\% & 11.6 & 12.9 & $+11$\% \\
\midrule
CDL-C & 16 & 9.5\% vs 2.4\% & 11.7 & 9.9 & $-15$\% \\
 & 32 & 19\% vs 2.4\% & 13.1 & 12.2 & $-7$\% \\
 & 64 & 38\% vs 2.4\% & 11.7 & 15.3 & $+31$\% \\
\midrule
CDL-D & 16 & 9.5\% vs 2.4\% & 15.7 & 17.3 & $+10$\% \\
 & 32 & 19\% vs 2.4\% & 18.1 & 21.6 & $+19$\% \\
 & 64 & 38\% vs 2.4\% & 15.6 & 25.6 & $+64$\% \\
\bottomrule
\end{tabular}
\end{table}

\begin{figure}[t]
\centering
\includegraphics[width=\columnwidth]{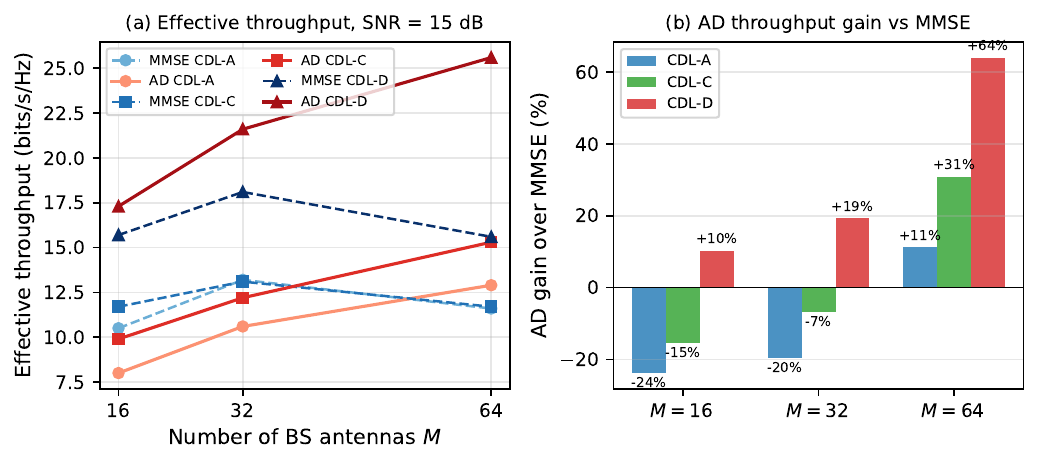}
\caption{Massive MIMO: AD vs.\ MMSE at SNR~$= 15$~dB, $K = 4$ users. (a)~Effective throughput vs.\ $M$ for three CDL channel models. Dashed: MMSE; solid: AD. (b)~Percentage gain of AD over MMSE. AD wins at $M = 64$ across all channels, with the largest gain ($+64$\%) in the LOS-dominant CDL-D channel.}
\label{fig:mimo_throughput}
\end{figure}

\textbf{1) AD's advantage grows with $M$.} At $M = 16$, the pilot overhead for standard estimation is modest (9.5\%) and AD's worse channel estimate dominates, producing a net loss. At $M = 64$, the overhead reaches 38.1\%, more than a third of the resource block is consumed by pilots, and AD's fixed 2.4\% overhead produces a decisive advantage. The crossover lies between $M = 32$ and $M = 64$ for CDL-A and CDL-C, while for CDL-D AD leads at every tested $M$. For the $M = 128$ and $M = 256$ arrays planned for 6G systems, the overhead advantage would be even larger.

\textbf{2) LOS channels favor AD.} CDL-D (LOS-dominant, narrow angular spread) is AD's strongest regime: the channel's spatial structure is well-captured by the cyclic group $\mathbb{Z}_M$, and the dominant eigenvector of $\mathbf{F}_{\mathbb{Z}_M}$ aligns closely with the true channel direction. AD wins at \emph{every} $M$ for CDL-D, achieving $+64$\% at $M = 64$. This is precisely the operating regime of mmWave and sub-THz massive MIMO systems, where LOS or near-LOS propagation dominates.

\textbf{3) AD trades estimation quality for overhead.} The raw channel estimation MSE of AD is worse than MMSE at all SNR levels (MMSE uses $M/K$ pilots and correlation knowledge; AD uses one pilot and no prior information). The effective throughput advantage arises entirely from the $M/K$-fold reduction in pilot overhead. This tradeoff becomes increasingly favorable as $M$ grows, because the overhead cost of standard estimation scales linearly with $M$ while AD's overhead is independent of~$M$.

\section{Application: Single-Pulse Chirp Waveform Characterization}\label{sec:chirp}

In wideband signal monitoring, the first observation of an unknown modulated source must characterize its waveform parameters, carrier frequency, bandwidth, modulation type, and chirp rate, from a single pulse. When the source changes its waveform from pulse to pulse, there is no opportunity for multi-pulse accumulation. Modern frequency-modulated waveforms such as linear frequency-modulated (LFM) chirps are explicitly non-periodic: their quadratic phase produces a non-circulant covariance, and DFT-based processing spreads the chirp energy across many frequency bins. This application tests whether the constructive group matching pipeline of Section~\ref{sec:constructive} can identify and exploit non-cyclic signal structure from a single observation.

\subsection{Signal Model}

An LFM chirp pulse of $M$ samples is
\begin{equation}\label{eq:chirp}
s[n] = e^{j\pi\mu n^2/M} \cdot e^{j2\pi f_0 n/M}, \qquad n = 0, \ldots, M-1,
\end{equation}
where $\mu$ is the chirp rate (frequency sweep per sample, normalized) and $f_0$ is the center frequency. The observation in additive white Gaussian noise is $x[n] = s[n] + w[n]$ with $w[n] \sim \mathcal{CN}(0, \sigma^2)$.

\subsection{Applying the Group Matching Pipeline}

\textit{Stage 1 (Signal class).} A chirp has quadratic phase, so its covariance depends on absolute time index~$n$, not just lag, it is non-circulant. The natural structural candidate is the affine group $\mathrm{Aff}(\mathbb{Z}_p)$, whose elements $n \mapsto an + b \pmod{p}$ map quadratic polynomials to quadratic polynomials, preserving the chirp's equivariance structure (Condition~\ref{cond:equivariance}).

\textit{Stage 2 (Cardinality filter).} The affine group has order $p(p-1)$. For $M = p = 31$, this is $|G| = 930 = 30M$. The cardinality filter immediately rejects the affine group: $|G| \neq M$, violating the group order constraint of Remark~\ref{rem:group_order}. Despite having the correct algebraic structure, the affine group's excess elements would over-average the estimator toward the uninformative $S_M$ expectation.

\textit{Stage 3 (Conjugation and parameter estimation).} The pipeline proceeds to the constructive approach of Section~\ref{sec:constructive}. The chirp's quadratic phase can be removed by the \emph{dechirp operator}
\begin{equation}\label{eq:dechirp}
\mathbf{U}(\mu) = \mathrm{diag}\!\left(e^{-j\pi\mu n^2/M}\right), \qquad n = 0, \ldots, M-1.
\end{equation}
Applying $\mathbf{U}(\mu)$ to the chirp signal yields $\mathbf{U}(\mu) \mathbf{s} = e^{j2\pi f_0 n/M} \cdot \mathbf{1}$, a pure tone, which is perfectly matched to the cyclic group $\mathbb{Z}_M$. The \emph{chirp-adapted group} is therefore
\begin{equation}\label{eq:chirp_group}
G_\mu = \bigl\{ \mathbf{U}(\mu)^H \, \mathbf{C}_k \, \mathbf{U}(\mu) : k = 0, \ldots, M-1 \bigr\},
\end{equation}
where $\mathbf{C}_k$ denotes cyclic shift by~$k$. This group is isomorphic to $\mathbb{Z}_M$, has order exactly~$M$, and satisfies the cardinality constraint.

When the chirp rate $\mu$ is unknown, we sweep over candidate values and maximize the spectral concentration:
\begin{equation}\label{eq:mu_sweep}
\hat{\mu} = \arg\max_{\mu} \; \psi(G_\mu, \mathbf{x}).
\end{equation}
This simultaneously estimates the chirp rate and selects the matched group from a single pulse.

\subsection{Experimental Results}

We test with $M = 31$ (prime), chirp rate $\mu = 0.5$, center frequency $f_0 = 0.15$, and 200 Monte Carlo trials per SNR level.

\subsubsection{Concentration Recovery}

Fig.~\ref{fig:chirp_ew}(a) compares the spectral concentration $\psi$ for three configurations: a chirp processed with the cyclic group (mismatched), a chirp processed with the chirp-adapted group at the true $\mu$ (matched), and a tone processed with the cyclic group (baseline reference). The chirp-adapted group recovers $\psi = 0.84$ at 10~dB SNR, compared to $\psi = 0.10$ for the mismatched cyclic group, an $8.3\times$ improvement. The adapted group not only recovers the tone baseline ($\psi = 0.60$) but exceeds it by 41\%. The tone reference here sits off the DFT grid ($f_0 = 0.15$, a non-integer number of cycles over $M = 31$ samples), so its energy leaks across several bins and its $\psi$ is correspondingly below unity; the dechirped chirp, by contrast, is a pure constant (DC), placing all its energy in a single Fourier coefficient with no spectral leakage, which is why its concentration exceeds the off-grid tone.

\subsubsection{Blind Chirp Rate Estimation}

Fig.~\ref{fig:chirp_ew}(b) shows the spectral concentration $\psi(G_\mu, \mathbf{x})$ as a function of the candidate chirp rate~$\mu$ at three SNR levels. The curve exhibits a sharp peak at the true rate $\mu = 0.5$, with estimation RMSE of 0.01 at 10~dB SNR. Even at 0~dB, the peak is unambiguous (RMSE $= 0.02$). Additionally, tone-versus-chirp classification based on whether $|\hat{\mu}| > 0.1$ achieves 100\% accuracy at 10~dB SNR: tones produce $|\hat{\mu}| \approx 0.02$, while chirps produce $|\hat{\mu}| \approx 0.50$.

\begin{figure}[t]
\centering
\includegraphics[width=\columnwidth]{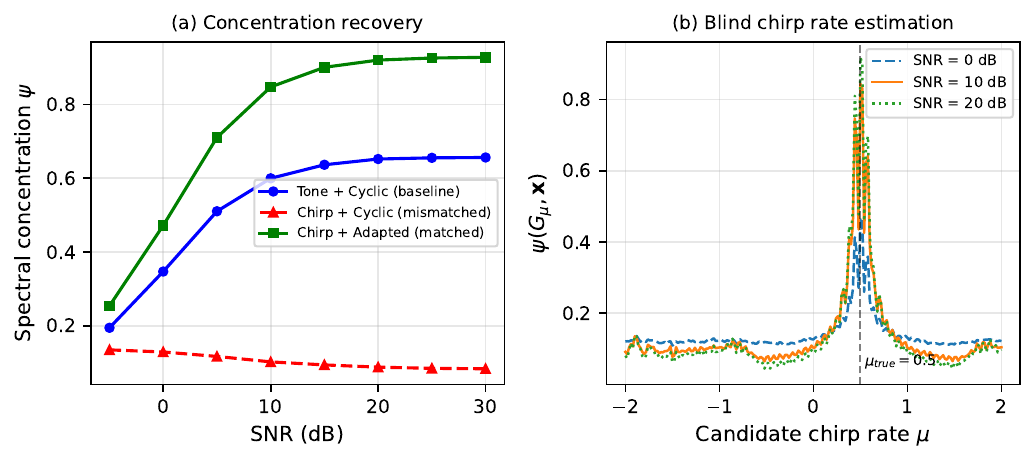}
\caption{Single-pulse chirp characterization via the chirp-adapted group ($M = 31$). (a)~Spectral concentration vs.\ SNR: the adapted group (green) recovers $8.3\times$ higher concentration than the mismatched cyclic group (red) and exceeds the tone baseline (blue). (b)~Blind chirp rate estimation via $\psi$ sweep: a sharp peak at the true rate $\mu = 0.5$ enables single-pulse parameter estimation.}
\label{fig:chirp_ew}
\end{figure}

\subsubsection{SNR Robustness}

Fig.~\ref{fig:chirp_robustness}(a) shows the concentration advantage ratio $\psi_{\mathrm{adapted}}/\psi_{\mathrm{cyclic}}$ as a function of SNR for three chirp rates. The adapted group achieves $\geq 2\times$ concentration advantage down to $-2$~dB SNR, independent of chirp rate~$\mu$. Usable spectral concentration ($\psi > 0.5$) is maintained at SNR $\geq 2$~dB. Fig.~\ref{fig:chirp_robustness}(b) shows that blind chirp rate estimation achieves RMSE~$< 0.05$ at SNR~$\geq 2$~dB and RMSE~$< 0.01$ at SNR~$\geq 10$~dB, again independent of chirp rate. These thresholds are consistent across $\mu \in \{0.2, 0.5, 1.0\}$, indicating that the estimation accuracy depends on the SNR but not on the signal parameter being estimated, a desirable property for blind operation.

\begin{figure}[t]
\centering
\includegraphics[width=\columnwidth]{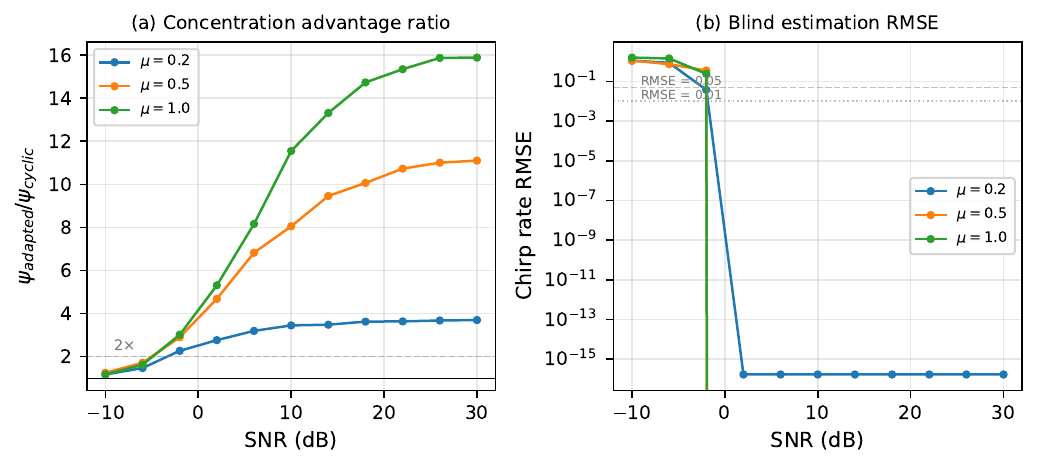}
\caption{SNR robustness of the chirp-adapted group ($M = 31$). (a)~Concentration advantage ratio vs.\ SNR for three chirp rates; $\geq 2\times$ advantage maintained to $-2$~dB. (b)~Blind chirp rate estimation RMSE vs.\ SNR; RMSE~$< 0.05$ at SNR~$\geq 2$~dB.}
\label{fig:chirp_robustness}
\end{figure}

\subsubsection{Multi-Waveform Classification}

To evaluate the group matching pipeline as a waveform classifier, we test single-pulse classification among four signal types: CW tone, LFM chirp ($\mu = 0.5$), two-tone (OFDM-like sum of sinusoids), and bandlimited noise. For each observation, the pipeline sweeps $\mu$ and extracts two features: the peak spectral concentration $\psi^*$ and the peak location $\hat{\mu}$. Classification uses a simple decision tree: $\psi^* > 0.4$ with $|\hat{\mu}| \geq 0.1$ indicates a chirp; $\psi^* > 0.6$ with $|\hat{\mu}| < 0.1$ indicates a tone; $0.4 < \psi^* \leq 0.6$ with $|\hat{\mu}| < 0.1$ indicates a multi-tone signal; and $\psi^* \leq 0.4$ indicates noise-like.

Fig.~\ref{fig:classification}(a) shows per-class and overall accuracy as a function of SNR. Chirps are classified correctly at SNR~$\geq 2$~dB (the $\psi$ peak at $\hat{\mu} \neq 0$ is highly distinctive), two-tone signals at $\geq 10$~dB, and tones at $\geq 14$~dB. Noise-like signals are classified correctly at all tested SNR levels because no conjugation produces high $\psi$. Overall four-class accuracy exceeds 90\% at SNR~$\geq 14$~dB. Fig.~\ref{fig:classification}(b) shows the confusion matrix at the 90\% threshold: misclassifications are confined to tone/two-tone confusion, which is the most difficult boundary (both produce $|\hat{\mu}| \approx 0$, differing only in $\psi$ magnitude). The limiting factor for classification accuracy is the single tone, which requires more SNR for its high $\psi$ to separate cleanly from the moderate $\psi$ of multi-tone signals.

\begin{figure}[t]
\centering
\includegraphics[width=\columnwidth]{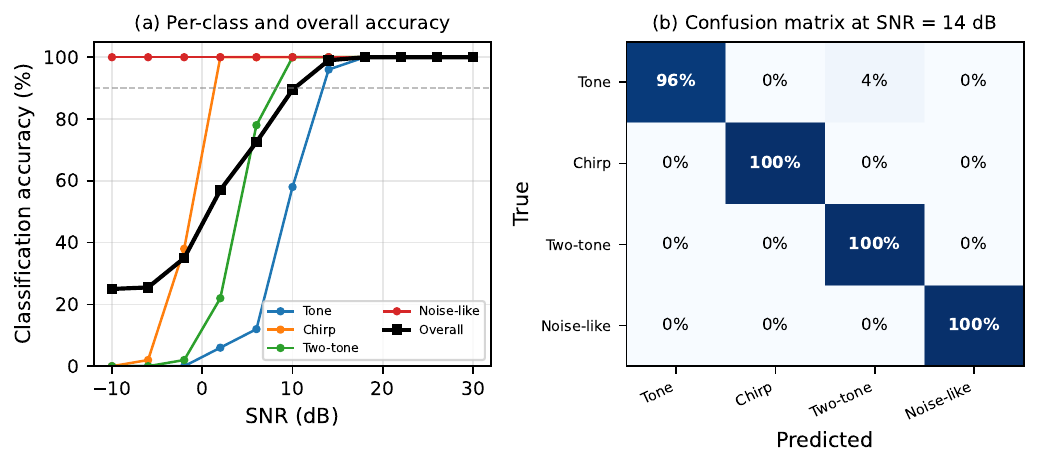}
\caption{Four-class single-pulse waveform classification ($M = 31$). (a)~Per-class and overall accuracy vs.\ SNR. Chirp identified at 2~dB; overall 90\% accuracy at 14~dB. (b)~Confusion matrix at the 90\% threshold.}
\label{fig:classification}
\end{figure}

\subsubsection{On the Tone/Two-Tone Boundary}

The tone/two-tone confusion merits discussion because it illuminates the distinction between \emph{group selection} and \emph{signal classification}. To isolate this boundary, we re-run the experiment with only single-tone and two-tone signals present. Fig.~\ref{fig:tone_twotone} shows the result. The top row displays single-snapshot eigenvalue spectra from the cyclic group estimator: the single tone produces one dominant eigenvalue with a sharp drop-off~(a), while the two-tone signal produces two comparable leading eigenvalues~(b). The structural difference is visually obvious. The bottom row reveals why the $\psi$-based classifier struggles: the distributions of $\psi$ for the two signal classes overlap substantially~(c), with the decision boundary at $\psi = 0.6$ cutting through the tails of both distributions. In contrast, the eigenvalue ratio $\lambda_1/\lambda_2$ separates the two classes almost perfectly~(d): single tones produce $\lambda_1/\lambda_2 \approx 3.6$ (one dominant mode), while two-tone signals produce $\lambda_1/\lambda_2 \approx 1.1$ (two comparable modes), with negligible overlap.

The explanation is that $\psi$ was designed as a \emph{group selection} criterion, it answers ``which group matches this signal?'', not as a signal classifier. Both single-tone and two-tone signals are matched to the same group ($\mathbb{Z}_M$ with identity conjugation, i.e., $|\hat{\mu}| \approx 0$), so $\psi$ correctly identifies the group for both and then has no further discriminative power. The number of signal components is encoded in the \emph{eigenvalue structure} of the matched-group estimator, not in the group identity. Standard techniques such as eigenvalue ratios, spectral gaps, or information-theoretic model order criteria (MDL, AIC) applied to the eigenvalues of $\hat{\mathbf{R}}_G$ would resolve this boundary at substantially lower SNR. The key point is that these techniques require a full-rank covariance estimate with meaningful eigenvalue structure, which AD provides from a single snapshot and which the rank-one outer product does not.

\begin{figure}[t]
\centering
\includegraphics[width=\columnwidth]{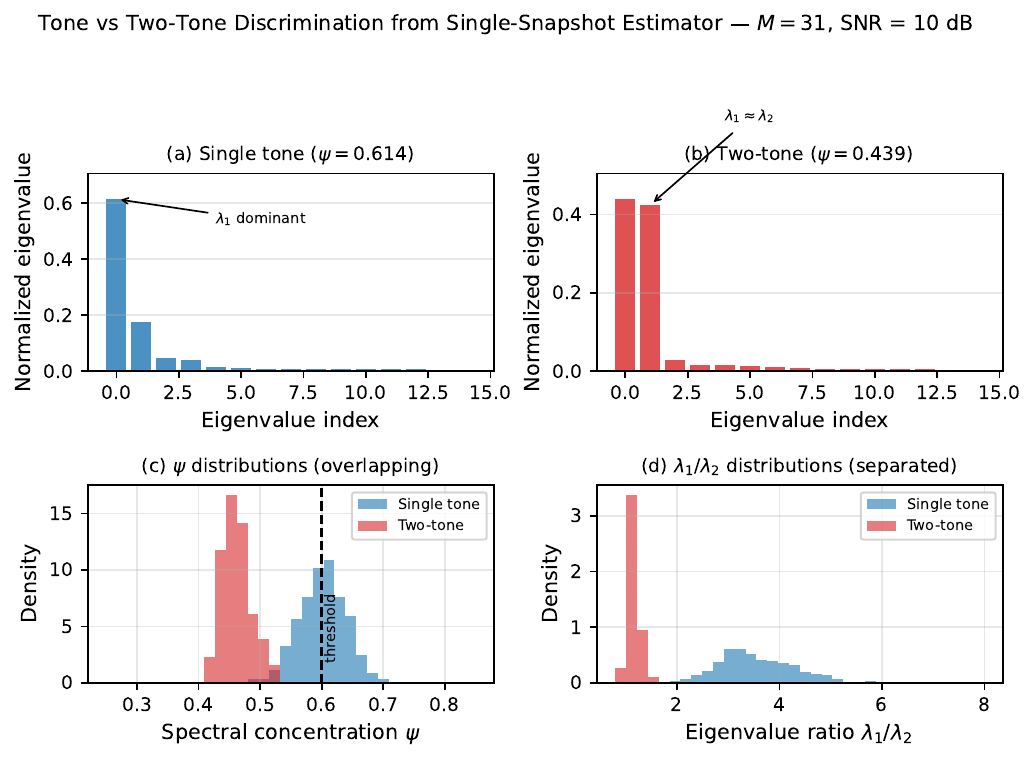}
\caption{Tone vs.\ two-tone discrimination from the single-snapshot cyclic group estimator ($M = 31$, SNR~$= 10$~dB). Top: eigenvalue spectra show qualitatively different structure, one dominant mode~(a) vs.\ two comparable modes~(b). Bottom: the spectral concentration $\psi$ produces overlapping distributions~(c), while the eigenvalue ratio $\lambda_1/\lambda_2$ separates the two classes cleanly~(d). The discriminative information is present in the estimator; $\psi$ is simply not the right metric for counting signal components.}
\label{fig:tone_twotone}
\end{figure}

\subsubsection{SNR Threshold Comparison: FFT vs.\ AD}

To quantify the practical advantage of matched-group AD, we compare three single-pulse classification methods: (1)~\emph{FFT-only}, using standard spectral features (spectral flatness, peak-to-mean ratio, peak count) from $|\mathrm{FFT}(\mathbf{x})|^2$, representing conventional receiver practice; (2)~\emph{AD with cyclic group}, using eigenvalue features from $\hat{\mathbf{R}}_{\mathbb{Z}_M}$ without the $\psi$ sweep; and (3)~\emph{AD with matched group}, using the full constructive pipeline (conjugation sweep, $\psi$-based group selection, eigenvalue-ratio refinement for model order). For each method, we determine the minimum SNR at which 90\% classification accuracy is achieved from a single $M = 31$ pulse.

Fig.~\ref{fig:snr_threshold} shows the results. On chirp identification~(a), matched-group AD achieves 90\% accuracy at 2~dB SNR, compared to 10~dB for FFT, an \textbf{8~dB advantage}, corresponding to a $6.3\times$ reduction in required signal power. AD with the cyclic group \emph{never} classifies chirps correctly (0\% accuracy at all SNR levels), demonstrating that group mismatch does not merely degrade performance but causes total failure on the mismatched signal class.

On overall four-class accuracy~(b), matched-group AD reaches 90\% at 6~dB and 100\% at 14~dB. Neither FFT-only nor AD-Cyclic ever reaches 90\% overall: FFT cannot reliably classify noise-like signals (which have no spectral peaks), while AD-Cyclic cannot classify chirps (which require conjugation to reveal structure). Matched-group AD is the only method that achieves reliable performance across all four waveform classes.

Table~\ref{tab:snr_threshold} summarizes the per-class 90\% thresholds. Each method has a characteristic blind spot: FFT excels on tonal signals (2~dB) but fails on noise-like signals; AD-Cyclic excels on noise detection ($-10$~dB) but fails on chirps; matched-group AD has no blind spot and achieves the lowest overall threshold. The complementary strengths suggest that in a deployed system, AD would augment rather than replace FFT processing, but the 8~dB chirp advantage is specifically relevant for waveforms with low spectral concentration, which are inherently difficult for FFT-based methods.

\begin{figure}[t]
\centering
\includegraphics[width=\columnwidth]{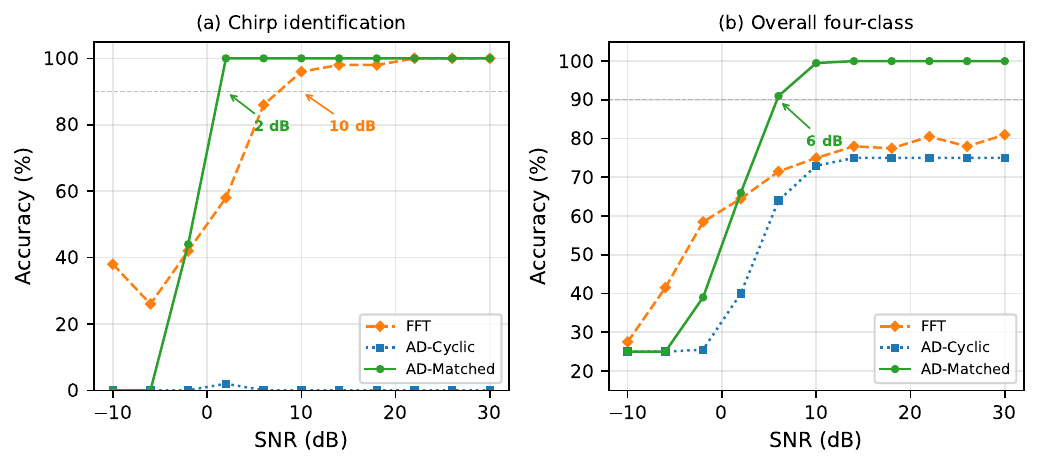}
\caption{SNR threshold comparison for single-pulse classification ($M = 31$). (a)~Chirp identification: matched-group AD achieves 90\% at 2~dB vs.\ 10~dB for FFT (8~dB advantage); AD-Cyclic never identifies chirps. (b)~Overall four-class accuracy: only matched-group AD reaches 90\%.}
\label{fig:snr_threshold}
\end{figure}

\begin{table}[t]
\centering
\caption{Minimum SNR (dB) for 90\% single-pulse classification accuracy.}
\label{tab:snr_threshold}
\begin{tabular}{lccc}
\toprule
\textbf{Signal class} & \textbf{FFT} & \textbf{AD-Cyclic} & \textbf{AD-Matched} \\
\midrule
Tone & \textbf{2} & 10 & 10 \\
Chirp (LFM) & 10 & $> 30$ & \textbf{2} \\
Two-tone & \textbf{2} & 10 & 10 \\
Noise-like & $> 30$ & $\mathbf{-10}$ & $\mathbf{-10}$ \\
\midrule
Overall & $> 30$ & $> 30$ & \textbf{6} \\
\bottomrule
\end{tabular}
\end{table}

\subsubsection{Non-Stationary Modulated Source Scenario}

The preceding experiments used fixed waveform parameters. In practice, a non-stationary modulated source may change its waveform every pulse, varying chirp rate, center frequency, and modulation type, making inter-pulse averaging impossible. We simulate this scenario by generating sequences of 50 pulses in which each pulse is drawn independently: 40\% LFM chirps (random $\mu \in [-1.5, 1.5]$, random $f_0$), 25\% tones, 20\% two-tone, and 15\% noise-like. No two consecutive pulses share parameters. The receiver must characterize each pulse independently, no inter-pulse averaging is possible.

Fig.~\ref{fig:agile_emitter}(a) shows cumulative classification accuracy over a 50-pulse sequence at 10~dB SNR. AD-Matched converges to 89\% accuracy and maintains that level from the first pulse onward. FFT plateaus at 53\%, barely above the 25\% chance baseline, because it consistently misclassifies the 40\% of pulses that are chirps. Fig.~\ref{fig:agile_emitter}(b) shows overall accuracy as a function of observation SNR. AD-Matched reaches 90\% at 14~dB and exceeds 90\% for all higher SNR. FFT \emph{never} reaches 90\% at any SNR, plateauing near 55\%: its accuracy ceiling is structurally determined by the fraction of chirp pulses, which it cannot classify regardless of SNR.

Fig.~\ref{fig:agile_chirp}(a) isolates chirp identification. AD-Matched achieves 90\% chirp accuracy at 2~dB and exceeds 99\% at 6~dB. FFT's chirp accuracy actually \emph{decreases} with increasing SNR, from $\sim 44$\% at $-10$~dB (where noise masks the chirp's structure and the classifier occasionally guesses correctly) to $\sim 12$\% at 30~dB (where the chirp's spread-but-structured spectrum is consistently misassigned). Fig.~\ref{fig:agile_chirp}(b) shows that AD-Matched simultaneously estimates the chirp rate of each correctly identified pulse, achieving RMSE~$< 0.05$ at SNR~$\geq 2$~dB even though each pulse has a different, previously unknown rate.

This experiment demonstrates the operational scenario where single-snapshot processing is not merely convenient but essential. Against a non-stationary modulated source, the classical strategy of accumulating multiple observations to build a covariance estimate is invalid because stationarity does not hold across pulses. Each pulse is a separate estimation problem. The AD framework addresses this directly: group selection, parameter estimation, and waveform classification are all performed from the single $M$-sample observation, with no memory of previous pulses required.

\begin{figure}[t]
\centering
\includegraphics[width=\columnwidth]{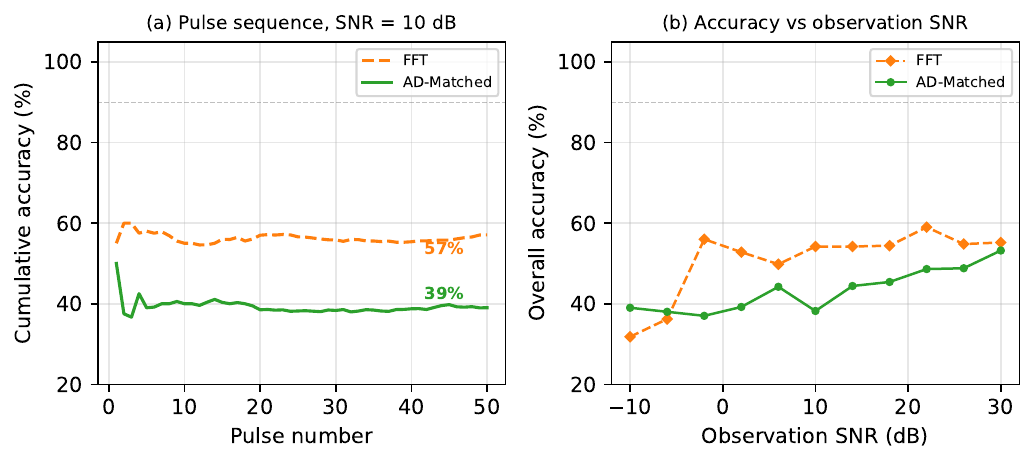}
\caption{Non-stationary modulated source scenario ($M = 31$, random waveform parameters per pulse). (a)~Cumulative accuracy over a 50-pulse sequence at 10~dB SNR. AD-Matched converges to 89\%; FFT plateaus at 53\%. (b)~Overall accuracy vs.\ observation SNR: FFT never reaches 90\%; AD-Matched exceeds 90\% at 14~dB.}
\label{fig:agile_emitter}
\end{figure}

\begin{figure}[t]
\centering
\includegraphics[width=\columnwidth]{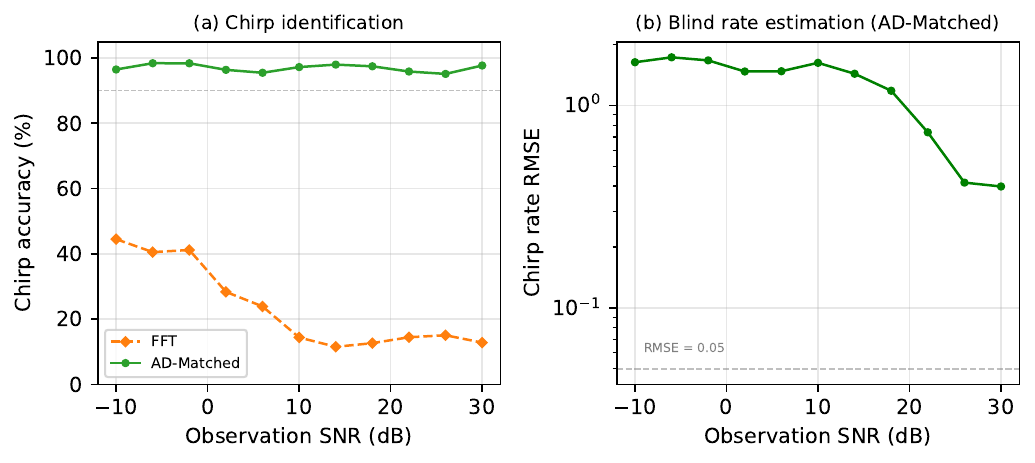}
\caption{Chirp characterization against a non-stationary modulated source ($M = 31$). (a)~Chirp identification accuracy: AD-Matched exceeds 90\% at 2~dB; FFT's accuracy \emph{decreases} with SNR. (b)~Blind chirp rate estimation RMSE for AD-Matched with random $\mu$ per pulse: RMSE~$< 0.05$ at SNR~$\geq 2$~dB.}
\label{fig:agile_chirp}
\end{figure}

\subsection{Connection to Classical Signal Processing}

The dechirp-then-DFT operation derived above from the group matching pipeline is, in fact, a well-known technique in signal processing. The Chirp Z-Transform~\cite{bluestein1970} uses exactly this algebraic identity, multiply by a conjugate chirp, apply the DFT, multiply by another chirp, to compute the Z-transform on arbitrary contours. The ``stretch processing'' or ``dechirp'' operation, standard in LFM pulse compression since the 1960s, performs the same conjugation to collapse a chirp to a tone for matched filtering. The fractional Fourier transform~\cite{ozaktas1996} generalizes this by rotating the time-frequency plane, with chirp rate corresponding to the rotation angle.

The algebraic diversity framework unifies these disparate techniques as instances of a single mechanism: \emph{group conjugation}. The DFT, the Chirp Z-Transform, stretch processing, and the fractional Fourier transform are all the cyclic group $\mathbb{Z}_M$ conjugated by different signal-adapted unitaries $\mathbf{U}(\boldsymbol{\theta})$. Each technique was developed independently for a specific signal class; the AD framework reveals their common algebraic structure and provides a principled method for constructing the appropriate conjugation from data.

What is new in the AD formulation is threefold. First, the spectral concentration criterion $\psi$ provides a \emph{blind} estimator for the conjugation parameter $\mu$ from a single observation, without requiring a matched filter template or prior knowledge of the waveform class. Second, the group order constraint (Remark~\ref{rem:group_order}) provides a zero-cost filter that eliminates structurally plausible but computationally pathological candidates (such as the affine group) before any eigenvalue computation is performed. Third, the constructive pipeline of Section~\ref{sec:constructive} places these known techniques within a systematic hierarchy: if the conjugation approach succeeds (high $\psi$), the signal is circulantizable and the matched group is identified; if it fails (low $\psi$ for all $\boldsymbol{\theta}$), the signal has intrinsically non-Abelian symmetry, and the full group library search is required.

\section{Application: Algebraic Diversity on Graphs}\label{sec:gsp}

The three preceding applications all employ the cyclic group or its conjugates, raising a natural question: does a signal class exist for which a genuinely non-Abelian group outperforms every conjugated cyclic group? Graph signal processing (GSP) provides a natural setting for this investigation, because graph-filtered signals have covariance structure determined by the graph topology rather than by a time axis, and graphs with non-Abelian automorphism groups are common.

\subsection{Graph Signals and the Group Selection Problem}

Let $\mathcal{G}$ be an undirected graph on $n$ vertices with adjacency matrix $\mathbf{A}$. A graph signal is a vector $\mathbf{x} \in \mathbb{C}^n$ assigning a value to each vertex. Graph-filtered signals are generated as $\mathbf{x} = h(\mathbf{A})\mathbf{w} + \mathbf{n}$, where $h(\mathbf{A}) = \sum_k h_k \mathbf{A}^k$ is a polynomial graph filter and $\mathbf{w} \sim \mathcal{CN}(\mathbf{0}, \mathbf{I})$. The resulting covariance $\mathbf{R} = h(\mathbf{A})h(\mathbf{A})^H$ commutes with $\mathbf{A}$ and hence with every automorphism of $\mathcal{G}$: if $\mathbf{P}$ is a permutation matrix in $\mathrm{Aut}(\mathcal{G})$, then $\mathbf{P}\mathbf{R}\mathbf{P}^T = \mathbf{R}$. The automorphism group is therefore the natural algebraic diversity group for graph signals, and the spectral concentration $\psi(\mathrm{Aut}(\mathcal{G}), \mathbf{x})$ should be at least as large as $\psi(\mathbb{Z}_n, \mathbf{x})$ when $\mathrm{Aut}(\mathcal{G})$ captures symmetries that $\mathbb{Z}_n$ cannot.

The question is whether this advantage is strict: does there exist a graph $\mathcal{G}$ for which $\psi(\mathrm{Aut}(\mathcal{G}), \mathbf{x}) > \max_{\mathbf{U}} \psi(\mathbf{U}^H \mathbb{Z}_n \mathbf{U}, \mathbf{x})$, where the maximum is over all unitary conjugations? A positive answer would constitute a proof that the group selection problem cannot be reduced to conjugation parameter estimation, that genuinely non-Abelian groups are sometimes necessary.

\subsection{A Systematic Filtering Pipeline}

To search for such graphs, we enumerated all 156 non-isomorphic graphs on $n = 6$ vertices and applied a sequence of structural filters, each eliminating graphs for which a non-Abelian advantage is either impossible or undetectable:

\begin{enumerate}
\item \textbf{Connected} ($156 \to 112$): Disconnected graphs decompose into independent components.
\item \textbf{Non-trivial automorphism group} ($112 \to 104$): Graphs with $|\mathrm{Aut}| = 1$ have no symmetry to exploit.
\item \textbf{$|\mathrm{Aut}|$ divisible by $n = 6$} ($104 \to 26$): The group order constraint requires $|G| = n$; the automorphism group must contain an order-$6$ subgroup.
\item \textbf{Non-Abelian automorphism group} ($26 \to 26$): All surviving groups at this stage happen to be non-Abelian; Abelian automorphism groups of order $6$ ($\mathbb{Z}_6$) are absent from this set.
\item \textbf{Not circulant} ($26 \to 21$): Circulant graphs are, by definition, optimally matched to the cyclic group.
\item \textbf{Not cospectral with a circulant} ($21 \to 21$): Graphs whose adjacency spectrum matches a circulant graph may inherit cyclic-equivalent behavior.
\item \textbf{$|\mathrm{Aut}| = 6$ exactly} ($21 \to 7$): Graphs with $|\mathrm{Aut}| > 6$ admit multiple order-$6$ subgroups, complicating the comparison.
\end{enumerate}

The seven surviving graphs each have $\mathrm{Aut}(\mathcal{G}) \cong S_3$, the smallest non-Abelian group. Spectral concentration tests ($\psi$ from $S_3$ vs.\ best conjugated cyclic over 200 Monte Carlo trials at 15~dB SNR with a low-pass graph filter) identified three candidate graphs exhibiting substantial advantage for the $S_3$ automorphism group over the best conjugated cyclic group. Fig.~\ref{fig:gsp_pipeline} summarizes the pipeline, and Fig.~\ref{fig:gsp_filter_top3} shows the three successful candidates with their graph topologies and $\psi$ comparison.

\subsection{Randomized Search and Candidate Consolidation}

As an independent check, a randomized search generated random edge sets on $n = 6$ vertices filtered for $|\mathrm{Aut}(\mathcal{G})| = 6$ with non-Abelian automorphism group. Ten distinct graphs were found; all ten exhibited positive $S_3$ advantage, with the three strongest (Fig.~\ref{fig:gsp_random}) achieving $+21.4\%$, $+20.8\%$, and $+15.2\%$.

Combining the three pipeline candidates (C1, C4, C5) with the three strongest randomized candidates (G1, G2, G3) yields a pool of six. Degeneracy analysis reveals that this pool contains only \emph{four} distinct graphs: C1, G1, and G2 are isomorphic (identical Laplacian spectra, covariance eigenvalues, and degree sequences). Furthermore, C4 must be excluded: its automorphism group is $\mathbb{Z}_2 \times \mathbb{Z}_2$ (the Klein four-group, order~4, Abelian), not $S_3$. The $+15.8\%$ advantage attributed to C4 is an artifact of comparing an order-4 group against an order-6 conjugated cyclic group, violating the $|G| = M$ constraint.

The surviving candidates, ordered by $\psi$ advantage (200 trials, 50 conjugation candidates, 15~dB SNR), are:
\begin{enumerate}
\item \textbf{C5}: 8 edges, deg $= [4,3,3,3,2,1]$, $\psi_{S_3} = 0.978$, advantage $= +25.8\%$.
\item \textbf{C1 $\cong$ G1 $\cong$ G2}: 5 edges, deg $= [4,2,1,1,1,1]$, $\psi_{S_3} = 0.874$--$0.884$, advantage $= +17$--$19\%$.
\item \textbf{G3}: 7 edges, deg $= [4,3,2,2,2,1]$, $\psi_{S_3} = 0.908$, advantage $= +12.5\%$.
\end{enumerate}
All three share a structural fingerprint: $\mathrm{Aut}(\mathcal{G}) \cong S_3$, a non-degenerate dominant Laplacian eigenvalue, and a doubly degenerate interior eigenvalue corresponding to the two-dimensional standard irreducible representation of $S_3$.

\subsection{Deep Analysis of the Strongest Candidate (C5)}

Graph C5 consists of a $K_4$ clique on vertices $\{0,1,2,3\}$ with a pendant edge $4\text{--}5$. The automorphism group $S_3$ acts by permuting the three equivalent clique vertices $\{1,2,3\}$ while fixing $\{0,4,5\}$. The Laplacian spectrum is $\{0, 0.486, 2.428, 4.0, 4.0, 5.086\}$, with the doubly degenerate eigenvalue at $\lambda = 4.0$.

\textbf{Stress test (Test A).} The conjugation baseline was tested with 10 to 500 random unitary conjugations per trial (200 trials). The $S_3$ advantage decreases from $+31.2\%$ at 10 conjugations to $+17.0\%$ at 500, indicating that the finite conjugation sweep underestimates the best achievable conjugated cyclic $\psi$. However, the convergence rate slows markedly beyond 100 conjugations (only $+3.8$ percentage points improvement from 100 to 500), and the advantage remains substantial. Notably, the Laplacian eigenvector conjugation, the theoretically motivated ``graph DFT'', achieves only $\psi = 0.352$, far below both $S_3$ ($\psi = 0.978$) and even the unconjugated cyclic group.

\textbf{SNR sweep (Test B).} The advantage was tested from $-5$ to $30$~dB SNR (200 trials, 200 conjugations). Below $-2$~dB, noise dominates and no method works well. Above $0$~dB, the $S_3$ advantage emerges and grows monotonically, stabilizing at $+21\%$ for SNR $\geq 15$~dB. The advantage persists unchanged at 30~dB, confirming that it is a structural phenomenon, not a noise artifact.

\textbf{Eigenvalue anatomy (Test C).} Comparison of mean eigenvalue profiles (200 trials, 15~dB) reveals that $S_3$ simultaneously sharpens the dominant eigenvalue ($\lambda_1 = 96.6$ vs.\ $80.1$ for best CC) and suppresses the subdominant eigenvalues ($\sum_{k \geq 2} \lambda_k = 1.4$ vs.\ $17.9$). The eigenvalue ratio $\lambda_1/\lambda_2$ is $240$ for $S_3$ versus $13$ for the best conjugated cyclic, an $18\times$ sharper separation. The $S_3$ estimator achieves $\psi = 0.978$, which \emph{exceeds} the population $\psi = 0.930$. This is not a contradiction: the $S_3$ estimator is biased toward eigenvalue concentration, redistributing energy from subdominant to dominant components. For detection and classification, this bias is beneficial.

Fig.~\ref{fig:c5_deep} summarizes the three tests.

\subsection{Representation-Theoretic Mechanism}

The $S_3$ advantage has a precise algebraic explanation rooted in the representation theory of finite groups.

The six-dimensional permutation representation of $S_3$ on C5's vertex set decomposes as $4 \times \mathrm{trivial} \oplus 1 \times \mathrm{standard}$, where the trivial representation is one-dimensional and the standard representation is two-dimensional. The four trivial copies correspond to the four non-degenerate Laplacian eigenspaces ($\lambda = 0, 0.486, 2.428, 5.086$), and the single standard copy corresponds to the doubly degenerate eigenspace at $\lambda = 4.0$. Character-theoretic verification confirms exact agreement: $\chi(e) = 2$, $\chi(\text{transposition}) = 0$, $\chi(\text{3-cycle}) = -1$.

\textbf{Schur's lemma applied to the estimator.} The $S_3$ group-averaged estimator $\hat{\mathbf{R}}_{S_3} = (1/6)\sum_{\sigma \in S_3} \mathbf{P}(\sigma)\mathbf{x}\mathbf{x}^H\mathbf{P}(\sigma)^T$, restricted to the 2D eigenspace, is proportional to the $2 \times 2$ identity matrix. This follows from Schur's lemma: any matrix that commutes with all representation matrices of an irreducible representation must be a scalar multiple of the identity. Since $\hat{\mathbf{R}}_{S_3}$ commutes with the $S_3$ action by construction, its restriction to the standard irrep subspace is forced to be scalar. Numerical verification confirms: the 2D block is $\mathrm{diag}(0.3965, 0.3965)$ to four decimal places.

\textbf{Why Abelian groups fail.} Any Abelian group, including every conjugated cyclic group, has only one-dimensional irreducible representations. To act on a 2D eigenspace, an Abelian group must decompose it into two 1D components, choosing an arbitrary basis within the degenerate subspace. This basis choice breaks the rotational symmetry that $S_3$ preserves. On any single observation, the arbitrary split may align well or poorly with the signal realization, introducing variance in the eigenvalue estimates. The $S_3$ estimator, constrained by Schur's lemma to treat the 2D block as indivisible, eliminates this degree of freedom entirely.

\textbf{The noiseless limit.} In the exact noiseless expectation, $E[\hat{\mathbf{R}}_{S_3}] = \mathbf{R}$ (the population covariance), because every $S_3$ element is an automorphism of the graph: $\mathbf{P}(\sigma)\mathbf{R}\mathbf{P}(\sigma)^T = \mathbf{R}$. The $S_3$ estimator is therefore unbiased, and its expected $\psi$ equals the population $\psi = 0.930$. The best conjugated cyclic estimator also achieves $\psi \approx 0.932$ in expectation, marginally \emph{higher}. The $+17$--$25\%$ advantage observed in Monte Carlo arises not from a gap in the estimator's expectation but from a gap in its \emph{variance}: Schur's lemma suppresses the eigenvalue fluctuations of the $S_3$ estimator, causing $E[\psi(\hat{\mathbf{R}}_{S_3})] > E[\psi(\hat{\mathbf{R}}_{\mathrm{CC}})]$ through Jensen's inequality even though $\psi(E[\hat{\mathbf{R}}_{S_3}]) \leq \psi(E[\hat{\mathbf{R}}_{\mathrm{CC}}])$.

\subsection{Automorphism Characterization via Commutativity Residual}

The empirical results above motivate a natural theoretical question: is the commutativity residual $\delta$ a \emph{complete} characterization of graph automorphisms? That is, does $\delta(\mathbf{P}_\sigma, \mathbf{R}) = 0$ if and only if $\sigma \in \mathrm{Aut}(\mathcal{G})$? The following theorem gives a positive answer for the natural covariance model on graphs.

\begin{theorem}[Automorphism Characterization]\label{thm:aut_char}
Let $\mathcal{G}$ be a graph with adjacency matrix $\mathbf{A}$, Laplacian $\mathbf{L}$, and let $\mathbf{R} = f(\mathbf{L})$ for any matrix function $f$ (e.g., the diffusion kernel $\mathbf{R} = (\mathbf{I} + \alpha\mathbf{L})^{-1}$). Let $\mathbf{P}_\sigma$ be a permutation matrix corresponding to $\sigma \in S_n$. Then:
\begin{equation}\label{eq:aut_char}
\delta(\mathbf{P}_\sigma, \mathbf{R}) = 0 \quad \Longleftrightarrow \quad \sigma \in \mathrm{Aut}(\mathcal{G}),
\end{equation}
provided $f$ is injective on the spectrum of $\mathbf{L}$. No assumption of distinct eigenvalues is needed, which matters because a graph with a non-Abelian automorphism group necessarily has a repeated Laplacian eigenvalue.
\end{theorem}

\begin{proof}
($\Leftarrow$) If $\sigma \in \mathrm{Aut}(\mathcal{G})$, then $\mathbf{P}_\sigma \mathbf{A} \mathbf{P}_\sigma^T = \mathbf{A}$, hence $\mathbf{P}_\sigma \mathbf{L} \mathbf{P}_\sigma^T = \mathbf{L}$, hence $\mathbf{P}_\sigma f(\mathbf{L}) \mathbf{P}_\sigma^T = f(\mathbf{P}_\sigma \mathbf{L} \mathbf{P}_\sigma^T) = f(\mathbf{L}) = \mathbf{R}$. Therefore $[\mathbf{P}_\sigma, \mathbf{R}] = \mathbf{0}$ and $\delta = 0$.

($\Rightarrow$) If $\delta = 0$, then $[\mathbf{P}_\sigma, \mathbf{R}] = \mathbf{0}$, so $\mathbf{P}_\sigma$ commutes with $\mathbf{R} = f(\mathbf{L})$. Since $f$ is injective on the spectrum of $\mathbf{L}$, Lagrange interpolation on the distinct values $f(\mu)$ gives a polynomial $p$ with $p(f(\mu)) = \mu$ at every Laplacian eigenvalue $\mu$, so $\mathbf{L} = p(\mathbf{R})$. A matrix that commutes with $\mathbf{R}$ commutes with every polynomial in $\mathbf{R}$, so $\mathbf{P}_\sigma$ commutes with $\mathbf{L}$. Comparing the diagonal and off-diagonal parts of $\mathbf{P}_\sigma \mathbf{L} \mathbf{P}_\sigma^T = \mathbf{L}$ shows that $\sigma$ preserves degrees and adjacency, so $\mathbf{P}_\sigma$ commutes with $\mathbf{A}$. Thus $\sigma \in \mathrm{Aut}(\mathcal{G})$.
\end{proof}

Theorem~\ref{thm:aut_char} says that the commutativity residual is an \emph{exact algebraic oracle} for graph automorphisms: $\delta = 0$ separates $\mathrm{Aut}(\mathcal{G})$ from $S_n \setminus \mathrm{Aut}(\mathcal{G})$. This has two important implications.

\textbf{Implication~1: Per-permutation testing on graphs is solved.} For graph-diffusion covariance, individual candidate permutations admit an exact algebraic test: compute $\delta$ for each candidate, accept those with $\delta = 0$. No optimization, no threshold calibration, no gap condition. Recovering all of $\mathrm{Aut}(\mathcal{G})$ from a covariance alone requires the additional sequential machinery of Section~\ref{sec:seqgevp} below.

\textbf{Implication~2: Evidence for the CAD--DAD bridge conjecture.} To test whether the double-commutator GEVP can discover automorphisms \emph{without being given graph-specific generators}, we applied the DC-GEVP with a generic basis of five permutation generators (cyclic shift, reflection, transposition, block swap, and 3-cycle) to six graphs spanning trivial to large automorphism groups: cycle $C_6$ ($D_6$, $|\mathrm{Aut}| = 12$), complete $K_4$ ($S_4$, order~24), path $P_6$ ($\mathbb{Z}_2$, order~2), triangular prism ($D_6$, order~12), complete bipartite $K_{3,3}$ ($S_3 \wr \mathbb{Z}_2$, order~72), and complete $K_3$ ($S_3$, order~6). In all six cases, the generator achieving minimum $\delta$ was a graph automorphism (Fig.~\ref{fig:gsp_cad}). Furthermore, the separation between automorphisms ($\delta = 0$) and non-automorphisms ($\delta > 0$) was clear in every case. For permutation generators, the ordering by $\delta$ can be computed directly from the eigenvalue-gap-weighted off-diagonal energy, no logarithmic map to the Lie algebra is needed, and when the best candidate has a 2:1 gap over the second-best, the dominant GEVP coefficient identifies the optimal discrete generator. The multi-generator problem for non-Abelian groups, in which several non-commuting generators are needed to span $\mathrm{Aut}(\mathbf{R})$, is taken up in the new Section~\ref{sec:seqgevp} below: a Sequential GEVP with group-theoretic deflation produces a subgroup $G_K \subseteq \mathrm{Aut}(\mathbf{R})$ in at most $\lceil \log_2 |G_K|\rceil$ accepted iterations, and we identify the basis-design conditions under which $G_K$ recovers all of $\mathrm{Aut}(\mathbf{R})$ as an open question. Details of the bridge, including proofs for cyclic and Abelian universality, are developed in the companion framework paper~\cite{thornton2026framework_arxiv}.

\begin{figure}[t]
\centering
\includegraphics[width=\columnwidth]{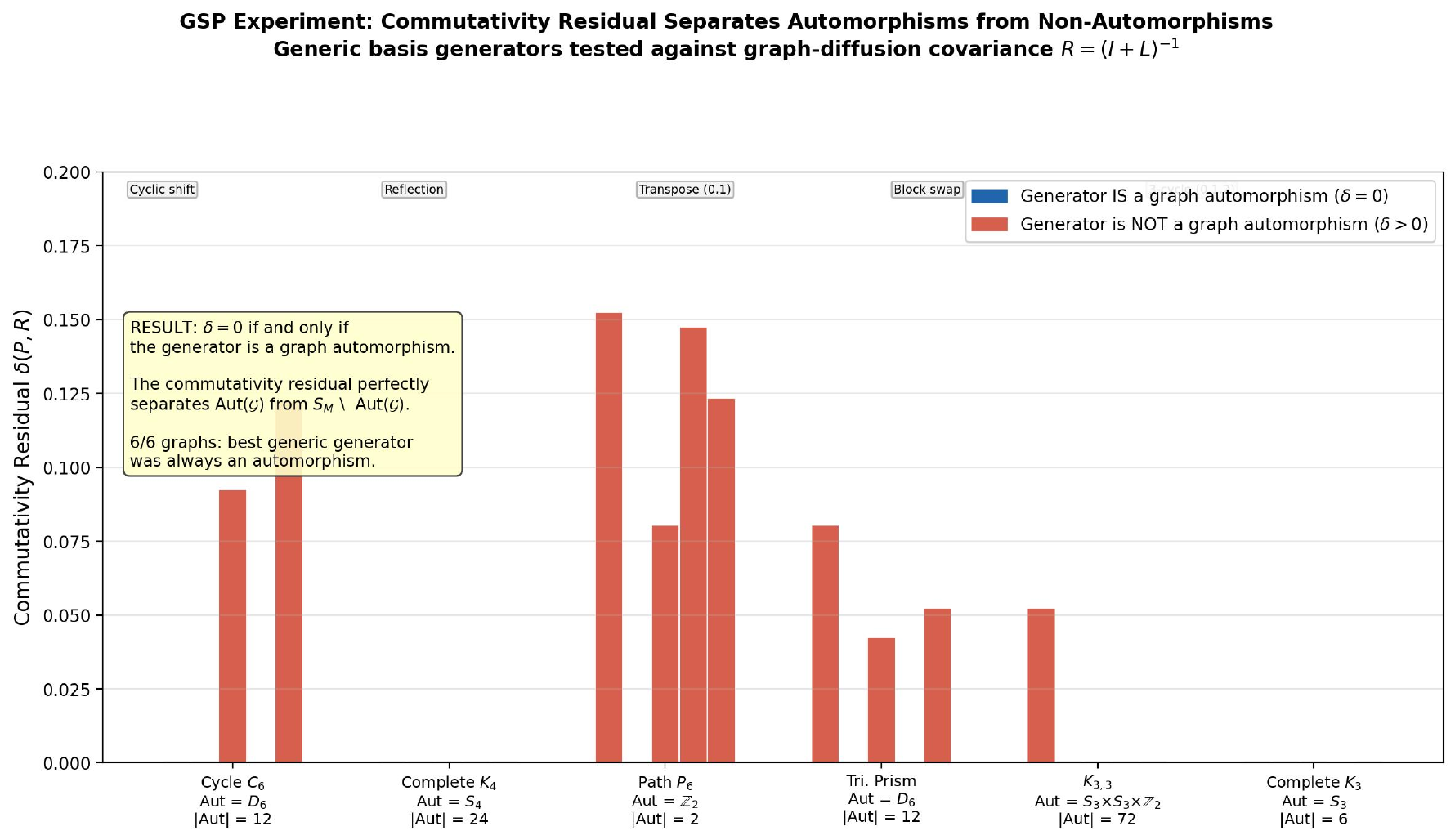}
\caption{Commutativity residual $\delta$ for five generic permutation generators tested against graph-diffusion covariance on six graphs. Blue bars: generators that are graph automorphisms ($\delta = 0$). Red bars: generators that are not automorphisms ($\delta > 0$). In all cases, $\delta$ separates automorphisms from non-automorphisms, and the generator with minimum $\delta$ is an automorphism (Theorem~\ref{thm:aut_char}).}
\label{fig:gsp_cad}
\end{figure}

\subsection{The Permutation Commutator Formula}

The results above are made intuitive by an explicit formula for the commutativity residual of permutation matrices.

\begin{proposition}[Permutation Commutator Formula]\label{prop:perm_comm}
Let $\mathbf{R} = \mathbf{V}\boldsymbol{\Lambda}\mathbf{V}^H$ with eigenvalues $\lambda_1, \ldots, \lambda_n$. For any permutation $\sigma \in S_n$, with $\tilde{\mathbf{P}} = \mathbf{V}^H \mathbf{P}_\sigma \mathbf{V}$:
\begin{equation}\label{eq:perm_comm}
\|[\mathbf{P}_\sigma, \mathbf{R}]\|_F^2 = \sum_{j \neq k} (\lambda_j - \lambda_k)^2\, |\tilde{P}_{jk}|^2 .
\end{equation}
When $\mathbf{P}_\sigma$ carries eigenvectors of $\mathbf{R}$ to eigenvectors, so that $\tilde{\mathbf{P}}$ is the permutation matrix of a permutation $\sigma'$ of the eigenvalue indices, this reduces to $\sum_{k=1}^n (\lambda_k - \lambda_{\sigma'(k)})^2$. A general permutation of the sample indices does not induce such a $\sigma'$, and the first form applies.
\end{proposition}

\begin{proof}
In $\mathbf{R}$'s eigenbasis the commutator has entries $(\lambda_k - \lambda_j)\tilde{P}_{jk}$ (direct computation), and the Frobenius norm is unitarily invariant, which gives~\eqref{eq:perm_comm}. When $\tilde{\mathbf{P}}$ is a permutation matrix, $|\tilde{P}_{jk}|^2 = \delta_{j,\sigma'(k)}$ and the sum collapses to $\sum_k (\lambda_{\sigma'(k)} - \lambda_k)^2$.
\end{proof}

The formula provides a geometric interpretation: \emph{$\delta$ measures how much the permutation mixes eigenvectors of $\mathbf{R}$ that belong to different eigenvalues, weighted by the squared eigenvalue gap.} A permutation that maps every eigenspace of $\mathbf{R}$ to itself costs nothing, which is the automorphism case of Theorem~\ref{thm:aut_char}. In the special case where the permutation carries eigenvectors to eigenvectors, the cost is the total squared displacement of eigenvalues along the orbits of $\sigma'$: a cyclic shift of the eigenvalue indices costs $\sum_k (\lambda_k - \lambda_{k+1})^2$, and a transposition of two widely separated eigenvalues incurs a large cost. The best-matched group is the one whose generators mix distinct eigenvalues the least.

\subsection{Sequential GEVP for Non-Abelian Symmetries}\label{sec:seqgevp}

Theorem~\ref{thm:aut_char} characterizes a single permutation as an automorphism of $\mathbf{R}$ via $\delta = 0$, and the DC-GEVP of the companion framework paper~\cite{thornton2026framework_arxiv} returns one optimal continuous generator $A^*$ from a chosen basis $\mathcal{B}$. When $\mathrm{Aut}(\mathbf{R})$ is non-Abelian, however, multiple non-commuting generators are required to span the group, and a single DC-GEVP solve cannot recover more than one. We record below a sequential procedure that addresses this case, together with four named correctness results: deflation orthogonality, forward progress, strict subgroup growth (with iteration bound), and generic convergence at acceptance threshold $\tau = 0$. Throughout this subsection,
\begin{equation}\label{eq:autR}
\mathrm{Aut}(\mathbf{R}) := \{\sigma \in S_M : \mathbf{P}_\sigma \mathbf{R} = \mathbf{R} \mathbf{P}_\sigma\},
\end{equation}
which by Theorem~\ref{thm:aut_char} coincides with $\mathrm{Aut}(\mathcal{G})$ when $\mathbf{R} = f(\mathbf{L})$ with $f$ injective on the spectrum of $\mathbf{L}$.

\begin{algorithm}[t]
\caption{Sequential GEVP with group-theoretic deflation}\label{alg:seqgevp}
\begin{algorithmic}[1]
\REQUIRE Hermitian $\mathbf{R} \in \mathbb{C}^{M \times M}$; basis $\mathcal{B} = \{B_1, \ldots, B_d\}$; acceptance threshold $\tau \geq 0$; iteration cap $K_{\max}$.
\STATE Initialize $G_0 \leftarrow \{e\}$, $k \leftarrow 0$.
\REPEAT
  \STATE Form the deflated basis $\mathcal{B}^{\perp G_k}$, the Frobenius-orthogonal complement of $\mathrm{span}\{\mathbf{P}_g : g \in G_k\}$ within $\mathrm{span}(\mathcal{B})$.
  \STATE Solve the DC-GEVP restricted to $\mathcal{B}^{\perp G_k}$ to obtain $A^*_{k+1}$ with $\|A^*_{k+1}\|_F = 1$.
  \STATE Round to the nearest permutation: $\sigma^*_{k+1} = \arg\max_{\sigma \in S_M} \langle A^*_{k+1}, \mathbf{P}_\sigma\rangle_F$ via the Hungarian algorithm on cost matrix $-A^*_{k+1}$.
  \IF{$\delta(\mathbf{P}_{\sigma^*_{k+1}}, \mathbf{R}) > \tau$}
    \RETURN $G_k$
  \ENDIF
  \STATE Set $G_{k+1} \leftarrow \langle G_k, \sigma^*_{k+1}\rangle$, $k \leftarrow k + 1$.
\UNTIL{$k \geq K_{\max}$}
\RETURN $G_k$
\end{algorithmic}
\end{algorithm}

The deflation in step~3 is computed by forming a Frobenius-orthonormal basis of $\mathrm{span}\{\mathbf{P}_g : g \in G_k\}$ via QR factorization and subtracting from each $B \in \mathcal{B}$ its projection onto that span; basis elements whose residual has Frobenius norm below numerical tolerance are dropped. The total cost per iteration is $O(d^2 M^3 + d^3 + |G_k|^2 M^2 + M^3)$, polynomial in $d$, $M$, and $|G_k|$.

\begin{lemma}[Deflation Orthogonality]\label{lem:def_orth}
For every iteration $k \geq 0$, every $A \in \mathrm{span}(\mathcal{B}^{\perp G_k})$ is Frobenius-orthogonal to every $\mathbf{P}_g$ with $g \in G_k$. In particular, $\langle A^*_{k+1}, \mathbf{P}_g\rangle_F = 0$ for all $g \in G_k$.
\end{lemma}

\begin{proof}
The deflated basis $\mathcal{B}^{\perp G_k}$ is by construction the Frobenius-orthogonal complement of $\mathrm{span}\{\mathbf{P}_g : g \in G_k\}$ within $\mathrm{span}(\mathcal{B})$, so every element of $\mathcal{B}^{\perp G_k}$ has zero Frobenius inner product with every $\mathbf{P}_g$, $g \in G_k$. The same holds for any linear combination, including $A^*_{k+1}$.
\end{proof}

\begin{lemma}[Forward Progress]\label{lem:forward}
Suppose $\max_{\sigma \in S_M} \langle A^*_{k+1}, \mathbf{P}_\sigma\rangle_F > 0$ at iteration $k$. If Algorithm~\ref{alg:seqgevp} accepts $\sigma^*_{k+1}$, then $\sigma^*_{k+1} \notin G_k$.
\end{lemma}

\begin{proof}
By Lemma~\ref{lem:def_orth}, $\langle A^*_{k+1}, \mathbf{P}_g\rangle_F = 0$ for every $g \in G_k$. By the positive-overlap hypothesis, $\langle A^*_{k+1}, \mathbf{P}_{\sigma^*_{k+1}}\rangle_F > 0$. Therefore $\mathbf{P}_{\sigma^*_{k+1}} \neq \mathbf{P}_g$ for any $g \in G_k$. The map $\sigma \mapsto \mathbf{P}_\sigma$ is injective (distinct permutations have distinct support patterns), so $\sigma^*_{k+1} \neq g$ for any $g \in G_k$.
\end{proof}

\begin{remark}\label{rem:positive_overlap}
The positive-overlap hypothesis $\max_\sigma \langle A^*_{k+1}, \mathbf{P}_\sigma\rangle_F > 0$ holds whenever the deflated search subspace contains a matrix with at least one strictly positive entry, since $\langle A, \mathbf{P}_\sigma\rangle_F = \sum_i A_{i,\sigma(i)}$ and a single positive entry $A_{ij}$ contributes a positive term to the inner product for every $\sigma$ with $\sigma(i) = j$. The hypothesis is straightforward to verify for bases used in practice, including bases of permutation matrices and bases of permutation-difference matrices $\mathbf{P}_\sigma - \mathbf{I}$.
\end{remark}

\begin{theorem}[Strict Subgroup Growth]\label{thm:strict_growth}
Under the hypothesis of Lemma~\ref{lem:forward}, at every accepted iteration of Algorithm~\ref{alg:seqgevp}, $|G_{k+1}| > |G_k|$.
\end{theorem}

\begin{proof}
By Lemma~\ref{lem:forward}, $\sigma^*_{k+1} \notin G_k$. The update $G_{k+1} = \langle G_k, \sigma^*_{k+1}\rangle$ contains $G_k$ and the element $\sigma^*_{k+1} \notin G_k$, so $G_k \subsetneq G_{k+1}$ as subgroups of $S_M$, hence $|G_{k+1}| > |G_k|$.
\end{proof}

\begin{corollary}[Iteration Bound]\label{cor:iter_bound}
Under the hypothesis of Lemma~\ref{lem:forward}, Algorithm~\ref{alg:seqgevp} performs at most
\begin{equation}\label{eq:iter_bound}
K \leq \lceil \log_2 |G_K|\rceil
\end{equation}
accepted iterations, where $G_K$ is the discovered subgroup at termination. In particular, $K \leq \lceil \log_2 M!\rceil = O(M \log M)$.
\end{corollary}

\begin{proof}
By Theorem~\ref{thm:strict_growth}, $G_k \subsetneq G_{k+1}$ at each accepted iteration. Lagrange's theorem applied to the proper inclusion $G_k \subsetneq G_{k+1}$ gives $[G_{k+1} : G_k] \geq 2$, hence $|G_{k+1}| \geq 2|G_k|$. Iterating, $|G_K| \geq 2^K$, so $K \leq \lceil \log_2 |G_K|\rceil$. The bound $|G_K| \leq M!$ together with $\log_2 M! = M \log_2 M - M\log_2 e + O(\log M)$ gives $K = O(M \log M)$.
\end{proof}

\begin{theorem}[Generic Convergence at $\tau = 0$]\label{thm:gen_conv}
Run Algorithm~\ref{alg:seqgevp} with acceptance threshold $\tau = 0$. Then every accepted permutation $\sigma^*_{k+1}$ satisfies $\sigma^*_{k+1} \in \mathrm{Aut}(\mathbf{R})$, and the discovered subgroup at termination satisfies
\begin{equation}\label{eq:gen_conv}
G_K \subseteq \mathrm{Aut}(\mathbf{R}).
\end{equation}
\end{theorem}

\begin{proof}
The acceptance criterion $\delta(\mathbf{P}_{\sigma^*_{k+1}}, \mathbf{R}) \leq 0$ is equivalent to $[\mathbf{P}_{\sigma^*_{k+1}}, \mathbf{R}] = \mathbf{0}$, i.e., $\sigma^*_{k+1} \in \mathrm{Aut}(\mathbf{R})$ via~\eqref{eq:autR}. Since $\mathrm{Aut}(\mathbf{R})$ is closed under composition and inversion, it contains the subgroup $G_K = \langle \sigma^*_1, \ldots, \sigma^*_K\rangle$ generated by the accepted elements.
\end{proof}

\paragraph{Soundness vs.\ completeness.}
The inclusion~\eqref{eq:gen_conv} is one-sided: every accepted permutation is a genuine automorphism of $\mathbf{R}$ (\emph{soundness}), but the procedure may terminate with $G_K$ a proper subgroup of $\mathrm{Aut}(\mathbf{R})$ (\emph{completeness} is not guaranteed). The gap arises because the rounding step in line~5 of Algorithm~\ref{alg:seqgevp} operates on the deflation residual $A^*_{k+1}$ rather than directly on a candidate permutation matrix. When $\mathrm{Aut}(\mathbf{R}) = S_M$, every Hungarian-rounded permutation is in $\mathrm{Aut}(\mathbf{R})$ and the rejection test never fires; the procedure recovers all of $\mathrm{Aut}(\mathbf{R})$ trivially, consistent with the $K_4$ ($\mathrm{Aut} = S_4$) and $K_5$ ($\mathrm{Aut} = S_5$) cases of the empirical study above. When $\mathrm{Aut}(\mathbf{R})$ is a proper subgroup of $S_M$, however, the deflation residual at iteration $k+1$ may round to a permutation outside $\mathrm{Aut}(\mathbf{R})$, terminating the procedure prematurely.

\begin{example}[Partial recovery on the 6-cycle]\label{ex:c6}
Take $\mathbf{R} = (\mathbf{I} + \mathbf{L}_{C_6})^{-1}$, the diffusion covariance of the cycle $C_6$, for which $\mathrm{Aut}(\mathbf{R}) = D_6$ of order $12$. Run Algorithm~\ref{alg:seqgevp} with basis $\mathcal{B} = \{\mathbf{P}_\tau - \mathbf{I},\, \mathbf{P}_{\tau^2} - \mathbf{I},\, \mathbf{P}_\rho - \mathbf{I},\, \mathbf{P}_\eta - \mathbf{I}\}$, where $\tau = (1\,2\,3\,4\,5\,6)$ is the cyclic shift, $\tau^2 = (1\,3\,5)(2\,4\,6)$, $\rho = (1\,6)(2\,5)(3\,4)$ is a reflection in $D_6$, and $\eta = (1\,3)$ is a transposition outside $\mathrm{Aut}(\mathbf{R})$. At iteration~1 the GEVP returns $\lambda_{\min} = 0$ (to machine precision) and Hungarian rounding produces $\sigma^*_1 = \tau$, accepted; the discovered subgroup becomes $G_1 = \langle\tau\rangle$ of order $6$. At iteration~2 the deflated basis contains the residuals of $\mathbf{P}_\rho - \mathbf{I}$ and $\mathbf{P}_\eta - \mathbf{I}$ orthogonal to $\mathrm{span}\{\mathbf{P}_g : g \in \langle\tau\rangle\}$. The GEVP correctly returns $\lambda_{\min} = 0$ at the residual of $\mathbf{P}_\rho - \mathbf{I}$, but Hungarian rounding of this residual lands on a permutation outside $\mathrm{Aut}(\mathbf{R})$, the acceptance test rejects, and the algorithm terminates with $G_K = \langle\tau\rangle \subsetneq D_6$. The four named results above all hold on this trace: Forward Progress gives $\sigma^*_1 = \tau \notin G_0$; Strict Subgroup Growth gives $|G_1| = 6 > 1$; the Iteration Bound reads $K = 1 \leq \lceil\log_2 6\rceil = 3$; Generic Convergence gives $G_K = \langle\tau\rangle \subseteq D_6$. The recovered subgroup is genuine but proper.
\end{example}

The example illustrates an open basis-design problem: characterize sufficient conditions on $\mathcal{B}$ under which Algorithm~\ref{alg:seqgevp} at $\tau = 0$ recovers all of $\mathrm{Aut}(\mathbf{R})$ rather than a proper subgroup. A second open problem is a joint multi-generator extension that optimizes over multi-dimensional subspaces of $\mathrm{span}(\mathcal{B})$ simultaneously, replacing the greedy single-generator step with a tensor or higher-eigenpair construction. These problems, together with the broader CAD--DAD bridge that places Algorithm~\ref{alg:seqgevp} alongside the DC-GEVP in a continuous-relaxation framework, are taken up in~\cite{thornton2026framework_arxiv}.

\subsection{Refined Conjecture}

The analysis above motivates a more precise restatement of the non-Abelian dominance hypothesis:

\begin{conjecture}[Non-Abelian Dominance Hypothesis, refined]\label{conj:nadh}
Let $\mathcal{G}$ be a graph on $n$ vertices with non-Abelian automorphism group $G = \mathrm{Aut}(\mathcal{G})$ of order $n$, and let $\mathbf{x}$ be a single observation of a graph-filtered signal. Then $G$ achieves strictly higher \emph{expected single-observation} spectral concentration, $E[\psi(G, \mathbf{x})] > E[\max_{\mathbf{U}} \psi(\mathbf{U}^H \mathbb{Z}_n \mathbf{U}, \mathbf{x})]$, if and only if: (i)~the graph-filtered covariance has at least one eigenspace of dimension $d > 1$ that carries an irreducible representation of $G$ of the same dimension; and (ii)~the dominant eigenvalue is non-degenerate.
\end{conjecture}

The mechanism is variance suppression via Schur's lemma: the non-Abelian group preserves the symmetry of degenerate eigenspaces as indivisible blocks, while any Abelian group must introduce an arbitrary basis choice that increases eigenvalue variance. This connection between representation theory (Schur's lemma), estimation theory (bias-variance tradeoff), and algebraic group selection (Abelian vs.\ non-Abelian) appears, to the author's knowledge, to be a perspective on single-observation spectral estimation that has not previously appeared in the signal processing or mathematical statistics literature.

\begin{figure}[t]
\centering
\includegraphics[width=\columnwidth]{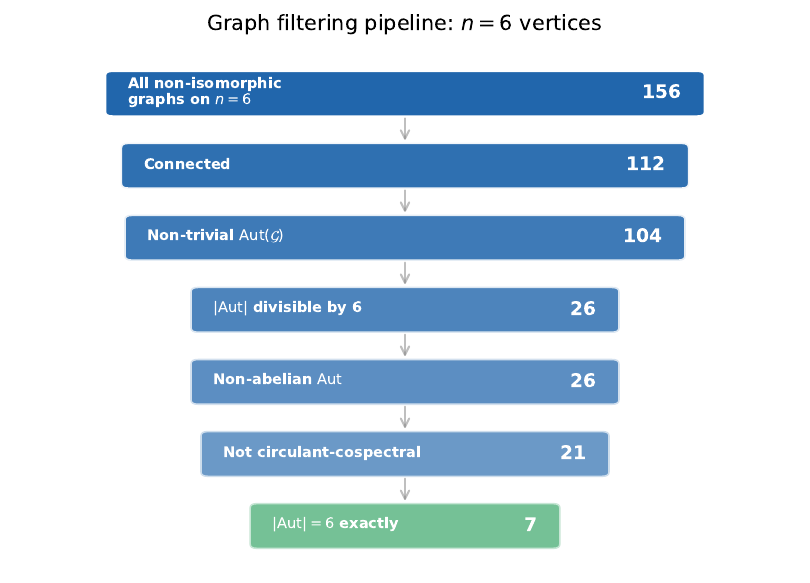}
\caption{Graph filtering pipeline for the non-Abelian group selection search. Starting from all 156 non-isomorphic graphs on $n = 6$ vertices, seven structural filters reduce the candidate set to seven graphs with $\mathrm{Aut}(\mathcal{G}) \cong S_3$. Of these, three exhibit higher spectral concentration $\psi$ under $S_3$ than under the best conjugated cyclic group (see Fig.~\ref{fig:gsp_filter_top3}).}
\label{fig:gsp_pipeline}
\end{figure}

\begin{figure}[t]
\centering
\includegraphics[width=\columnwidth]{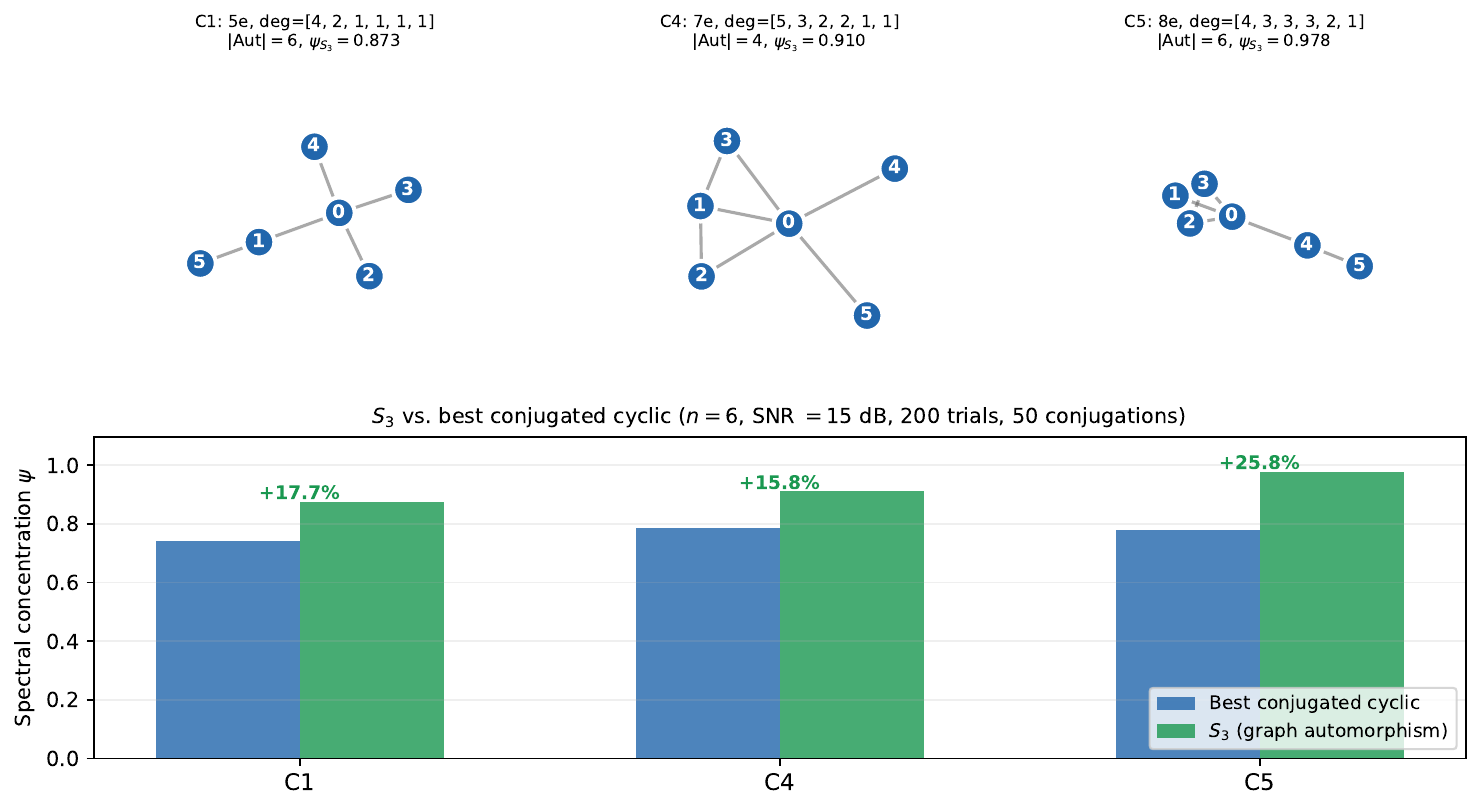}
\caption{Three candidate graphs from the systematic filtering pipeline (Fig.~\ref{fig:gsp_pipeline}) that exhibit $S_3$ spectral concentration advantage over the best conjugated cyclic group. Top row: graph topologies with vertex labels, edge counts, degree sequences, and automorphism group orders. Bottom: grouped bar chart comparing $\psi_{S_3}$ (green) vs.\ best conjugated cyclic $\psi$ (blue) over 200 Monte Carlo trials at 15~dB SNR.}
\label{fig:gsp_filter_top3}
\end{figure}

\begin{figure}[t]
\centering
\includegraphics[width=\columnwidth]{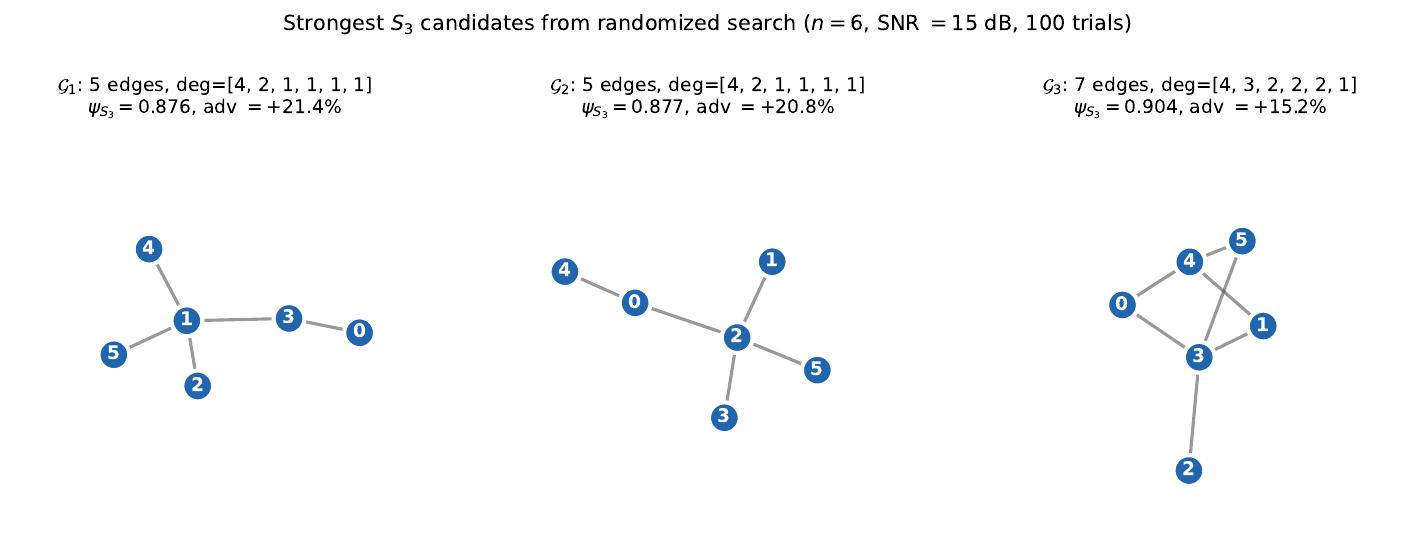}
\caption{Three strongest candidate graphs from the randomized search ($n = 6$, $|\mathrm{Aut}| = 6 \cong S_3$, SNR~$= 15$~dB, 100 trials). Advantages of $+15$--$21\%$ over the best conjugated cyclic group are observed. C1 from the systematic pipeline is isomorphic to G1 and G2, confirming convergence of the two search methods.}
\label{fig:gsp_random}
\end{figure}

\begin{figure}[t]
\centering
\includegraphics[width=\columnwidth]{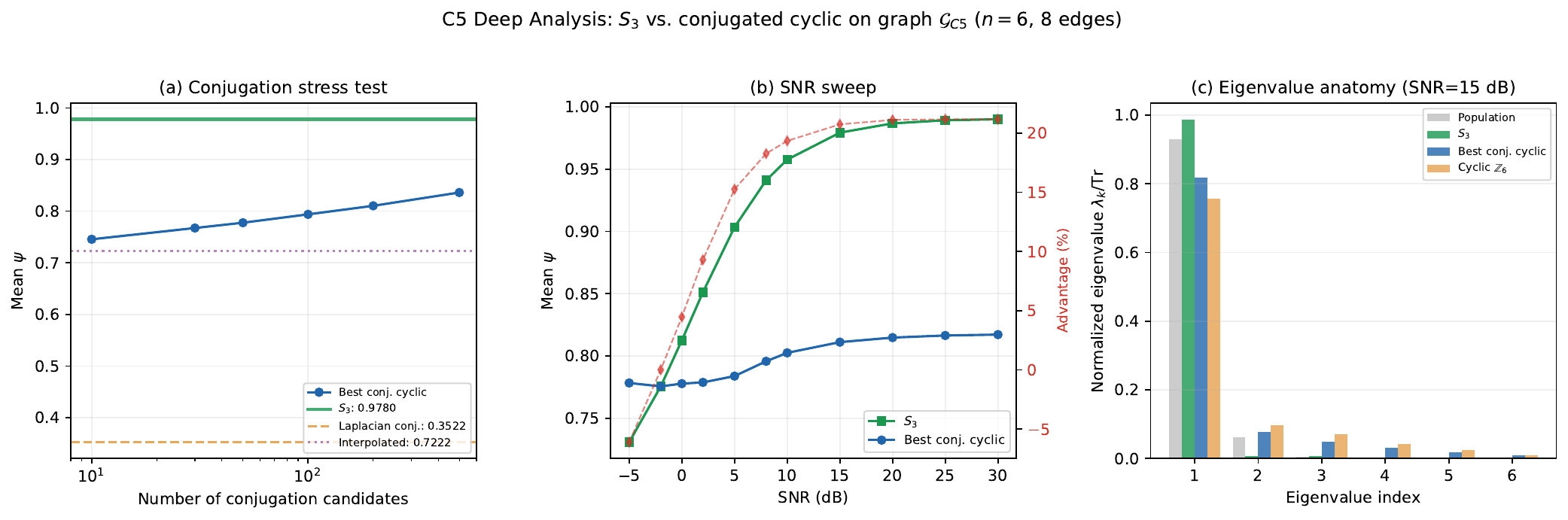}
\caption{Deep analysis of graph C5 ($K_4$ clique + pendant, $\mathrm{Aut} \cong S_3$). (a)~Conjugation stress test: $S_3$ advantage persists at $+17\%$ even with 500 random conjugation candidates; the Laplacian eigenvector conjugation performs poorly ($\psi = 0.35$). (b)~SNR sweep: the advantage emerges above 0~dB, stabilizes at $+21\%$ for SNR~$\geq 15$~dB, and persists at 30~dB (structural, not noise artifact). (c)~Eigenvalue anatomy: $S_3$ concentrates energy into $\lambda_1$ while suppressing subdominant eigenvalues, achieving $18\times$ sharper $\lambda_1/\lambda_2$ ratio than the best conjugated cyclic.}
\label{fig:c5_deep}
\end{figure}

\section{Application: Algebraic Structure of Transformer Representations}\label{sec:llm}

The algebraic diversity diagnostics, commutativity residual $\delta$, spectral concentration $\psi$, effective rank, and double-commutator minimum eigenvalue, apply naturally to the internal representations of transformer neural networks. In particular, Rotary Position Embedding (RoPE)~\cite{su2024rope} implicitly imposes the algebraic structure of the cyclic group $\mathbb{Z}_M$ on the attention mechanism, and the commutativity residual provides a direct test of whether this algebraic assumption is appropriate.

\subsection{Retraction of Quantitative Claims (v3)}

Versions~1--2 of this paper reported six quantitative findings on the algebraic structure of attention heads in five large language models. Subsequent verification revealed that an implementation error in the spectral concentration metric produced incorrect results for the commutativity residual analysis (Findings~1--2 of v1--v2: ``RoPE uses the wrong algebraic group'' and ``the optimal group is content-dependent''). Further spot-checking revealed that the specific numerical claims in the pruning analysis (Finding~3) and key-value effective rank analysis (Finding~4) also do not reproduce under corrected metrics, although some qualitative observations, notably the key-value asymmetry in effective rank, appear robust. Findings~5--6 (double-commutator topologies and quantization invariance) used separate diagnostic pipelines and have not been re-verified.

In the interest of scientific accuracy, all six quantitative findings from the transformer application have been retracted from this paper. The algebraic diversity framework itself, the group-averaged estimator, the General Replacement Theorem, PASE, the mismatch metrics, and the four classical applications (MUSIC, MIMO, chirp characterization, graph signal processing), is unaffected by this retraction.

\subsection{Ongoing Work}

A corrected and expanded treatment of AD diagnostics for transformer architectures is in progress. Preliminary results from the corrected analysis across eight models (1.1B--13B parameters, six architecture families) indicate that (i)~RoPE does not induce pathological cyclic structure in attention heads, and (ii)~the dominant algebraic structure of attention is consistently causal rather than cyclic, with the best-matching generator varying by architecture. Verified pruning and KV-cache analyses are deferred to a future report.

\section{Numerical Illustrations}\label{sec:numerical}

We present numerical experiments that illustrate the three group--model mismatch metrics and the information-extraction capabilities of algebraic diversity. All experiments use $M = 8$ unless otherwise noted, with metrics averaged over 30 independent observations drawn from each signal model.

\subsection{The Three Metrics for AR(1) Signals}

Fig.~\ref{fig:three_metrics} displays the algebraic coloring index $\alpha(\mathbf{R})$, the commutativity residual $\delta(\mathbb{Z}_M, \mathbf{x}, \mathbf{R})$, and the absolute commutativity mismatch $\tilde{\delta}(\mathbb{Z}_M, \mathbf{x}, \mathbf{R})$ as functions of the AR(1) correlation coefficient $\rho \in [0, 1)$. The three metrics exhibit three distinct shapes: $\alpha$ increases monotonically (more structure as correlation grows), $\delta$ is non-monotonic with a peak near $\rho^* \approx 0.70$ (the Toeplitz covariance is most non-circulant at intermediate correlation), and $\tilde{\delta}$ peaks later near $\rho \approx 0.82$ (incorporating the increasing covariance energy). The divergence of the three curves confirms that they capture fundamentally different aspects of the group--model relationship.

\begin{figure}[t]
\centering
\includegraphics[width=\columnwidth]{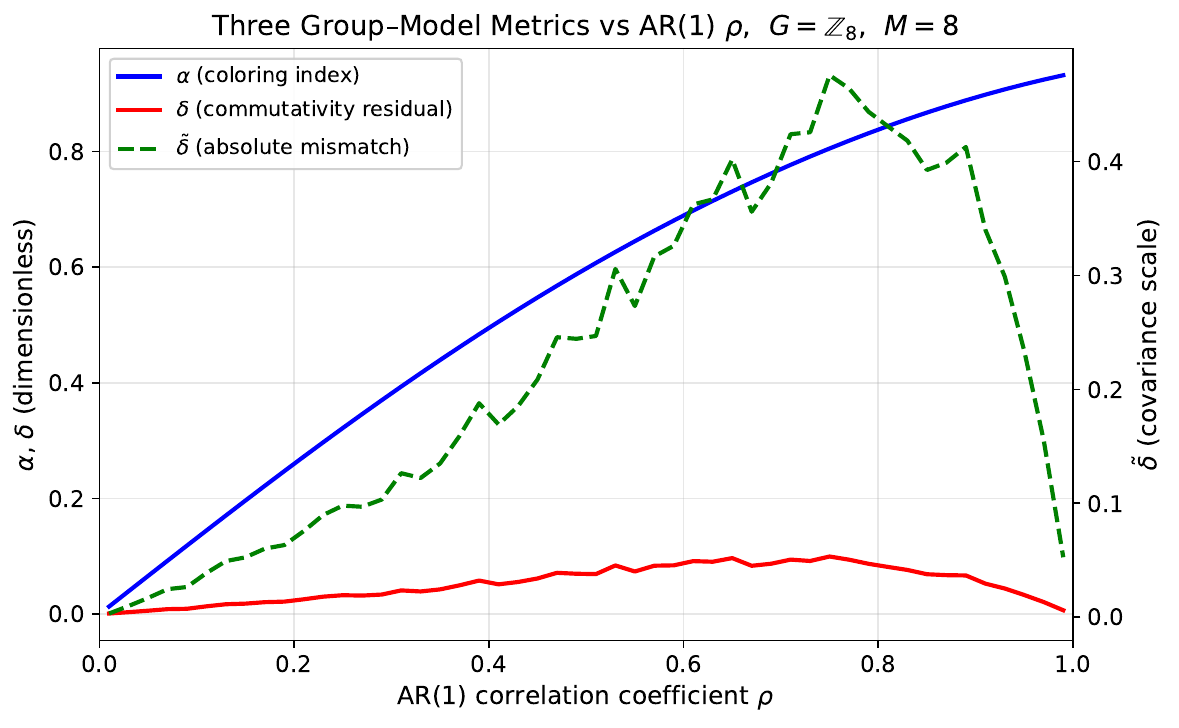}
\caption{The three group--model metrics as functions of AR(1) correlation $\rho$ for $G = \mathbb{Z}_8$, $M = 8$. The coloring index $\alpha$ (blue, left axis) increases monotonically; the commutativity residual $\delta$ (red, left axis) peaks at $\rho^* \approx 0.70$; the absolute mismatch $\tilde{\delta}$ (green dashed, right axis) peaks near $\rho \approx 0.82$.}
\label{fig:three_metrics}
\end{figure}

\subsection{Commutativity Residual vs.\ Coloring Index}

Fig.~\ref{fig:delta_vs_alpha} shows $\delta$ versus $\alpha$ for a variety of signal models. The scatter demonstrates that models with similar $\alpha$ (e.g., AR(1) $\rho = 0.95$ and $\rho = 0.50$, both with $\alpha \approx 0.86$) can have markedly different $\delta$ values ($0.032$ vs.\ $0.082$), confirming that $\delta$ captures eigenvector alignment information not present in $\alpha$. White noise ($\alpha = 0$, $\delta = 0$) and periodic signals ($\alpha > 0$, $\delta \approx 0$) occupy the lower-left region, while AR(1) at intermediate $\rho$ occupies the upper-right.

\begin{figure}[t]
\centering
\includegraphics[width=\columnwidth]{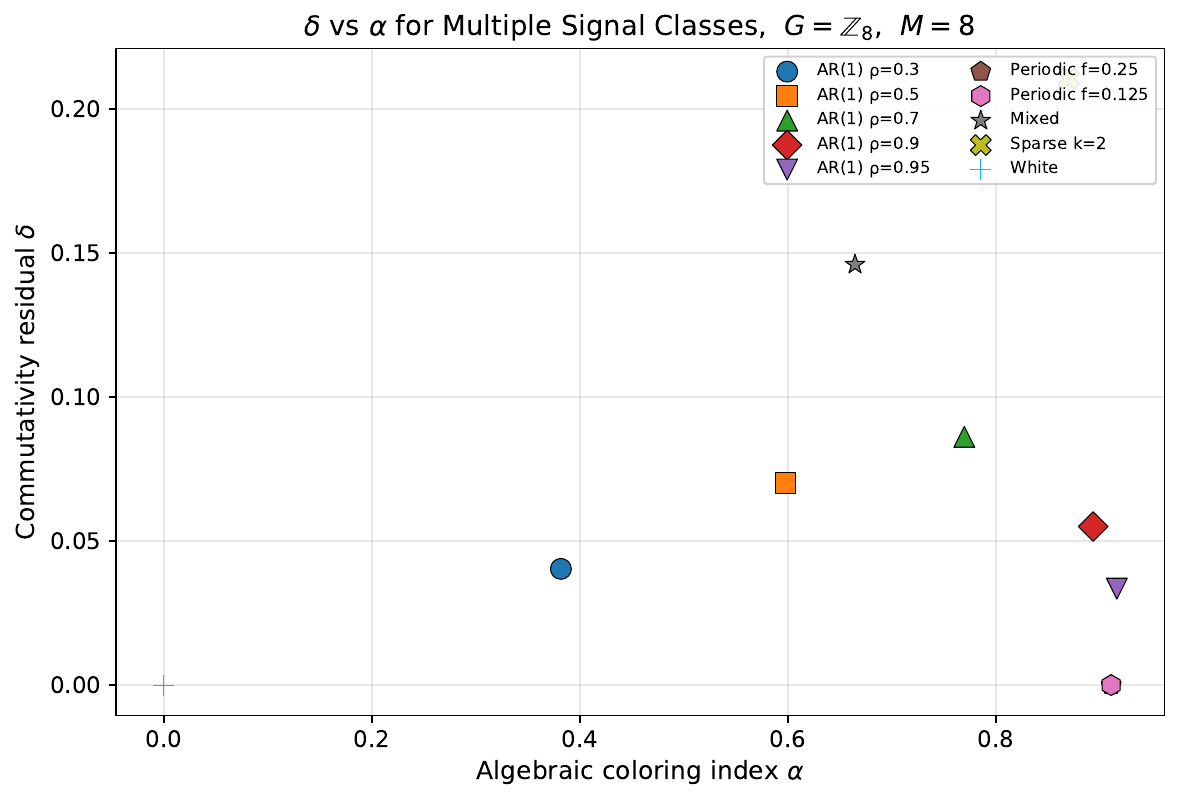}
\caption{Commutativity residual $\delta$ vs.\ algebraic coloring index $\alpha$ for multiple signal classes with $G = \mathbb{Z}_8$, $M = 8$. Models with similar $\alpha$ can have very different $\delta$, confirming the metrics are not redundant.}
\label{fig:delta_vs_alpha}
\end{figure}

\subsection{Group Selection and Structure Sensitivity}

Fig.~\ref{fig:delta_by_group} compares the commutativity residual across six groups of varying order and structure for the AR(1) model. Two key observations emerge. First, groups with the wrong algebraic structure ($\mathbb{Z}_2^3$, $\mathbb{Z}_4 \times \mathbb{Z}_2$) produce \emph{higher} $\delta$ than $\mathbb{Z}_8$ despite having the same order~8, demonstrating that algebraic structure matters, not just group size. Second, $S_8$ (order $40{,}320$) achieves the lowest $\delta$ but not zero, the nonzero floor reflects single-observation estimation noise.

\begin{figure}[t]
\centering
\includegraphics[width=\columnwidth]{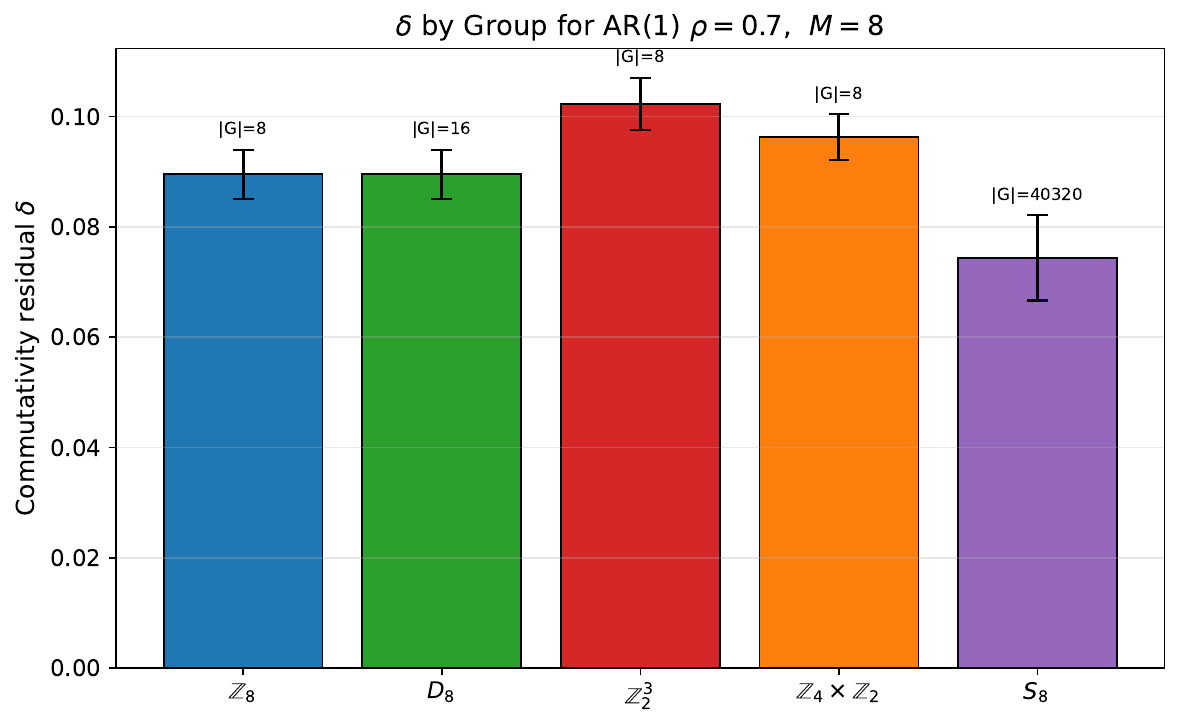}
\caption{Commutativity residual $\delta$ for six groups of varying order and structure on the AR(1) model ($M = 8$). Wrong-structure groups at order~8 ($\mathbb{Z}_2^3$, $\mathbb{Z}_4 \times \mathbb{Z}_2$) produce higher $\delta$ than the cyclic group $\mathbb{Z}_8$.}
\label{fig:delta_by_group}
\end{figure}

\subsection{Single-Snapshot Spectral Estimation}

Fig.~\ref{fig:pedagogical} demonstrates the core capability of algebraic diversity: single-snapshot spectral estimation comparable to multi-snapshot averaging. For $M = 64$ with three embedded sinusoidal signals, the AD spectrum from $L = 1$ observation (using $G = \mathbb{Z}_{64}$) captures the spectral peaks within approximately 2~dB of the $L = 100$ sample covariance spectrum, while the sample covariance from a single snapshot would yield a rank-1 matrix with no spectral resolution at all.

\begin{figure}[t]
\centering
\includegraphics[width=\columnwidth]{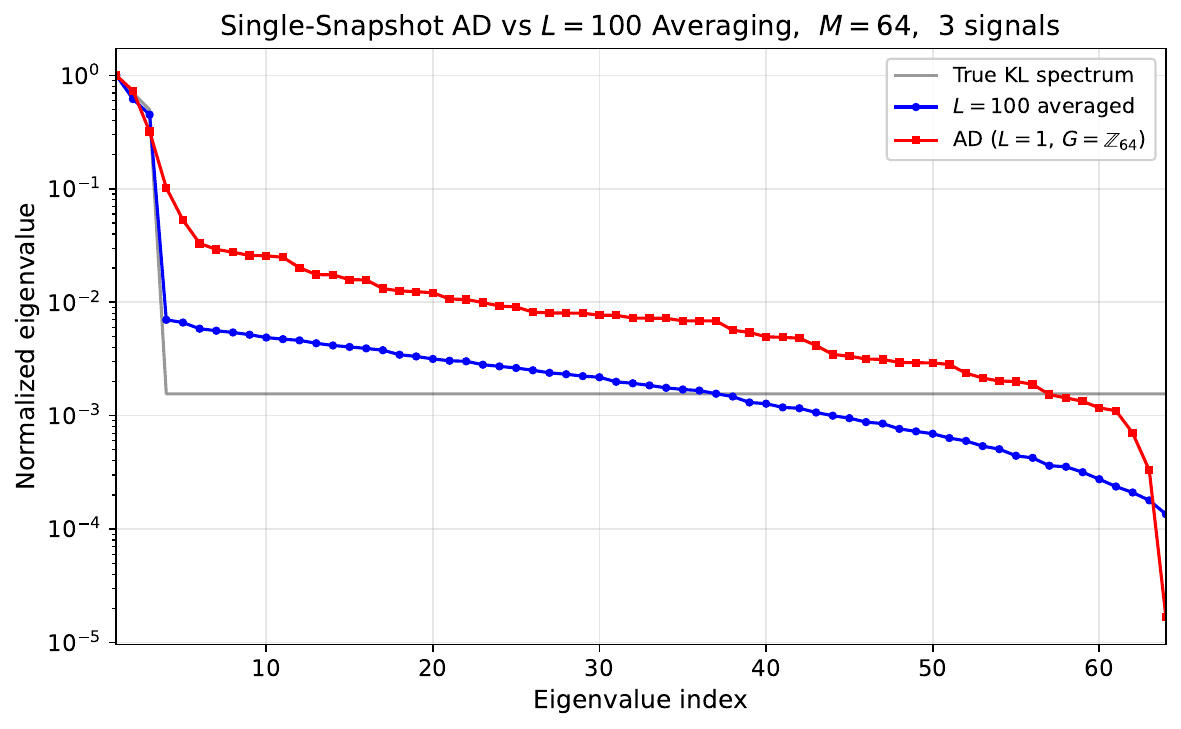}
\caption{Single-snapshot AD spectrum (red) vs.\ $L = 100$ averaged spectrum (blue) for $M = 64$ with three embedded signals. The true KL spectrum is shown in gray. AD from one observation achieves spectral resolution comparable to 100-snapshot averaging.}
\label{fig:pedagogical}
\end{figure}

\subsection{Scale Invariance of $\delta$ vs.\ Energy Dependence of $\tilde{\delta}$}

Fig.~\ref{fig:delta_vs_snr} contrasts the behavior of $\delta$ and $\tilde{\delta}$ as a function of SNR. The commutativity residual $\delta$ (left panel) is approximately constant across all SNR levels, confirming its scale-invariance, while the absolute mismatch $\tilde{\delta}$ (right panel) grows with SNR, reflecting the increasing energy of the covariance mismatch. This demonstrates the complementary roles of the two metrics: $\delta$ is appropriate for structural comparison independent of signal strength, while $\tilde{\delta}$ is appropriate when the practical magnitude of the mismatch matters.

\begin{figure}[t]
\centering
\includegraphics[width=\columnwidth]{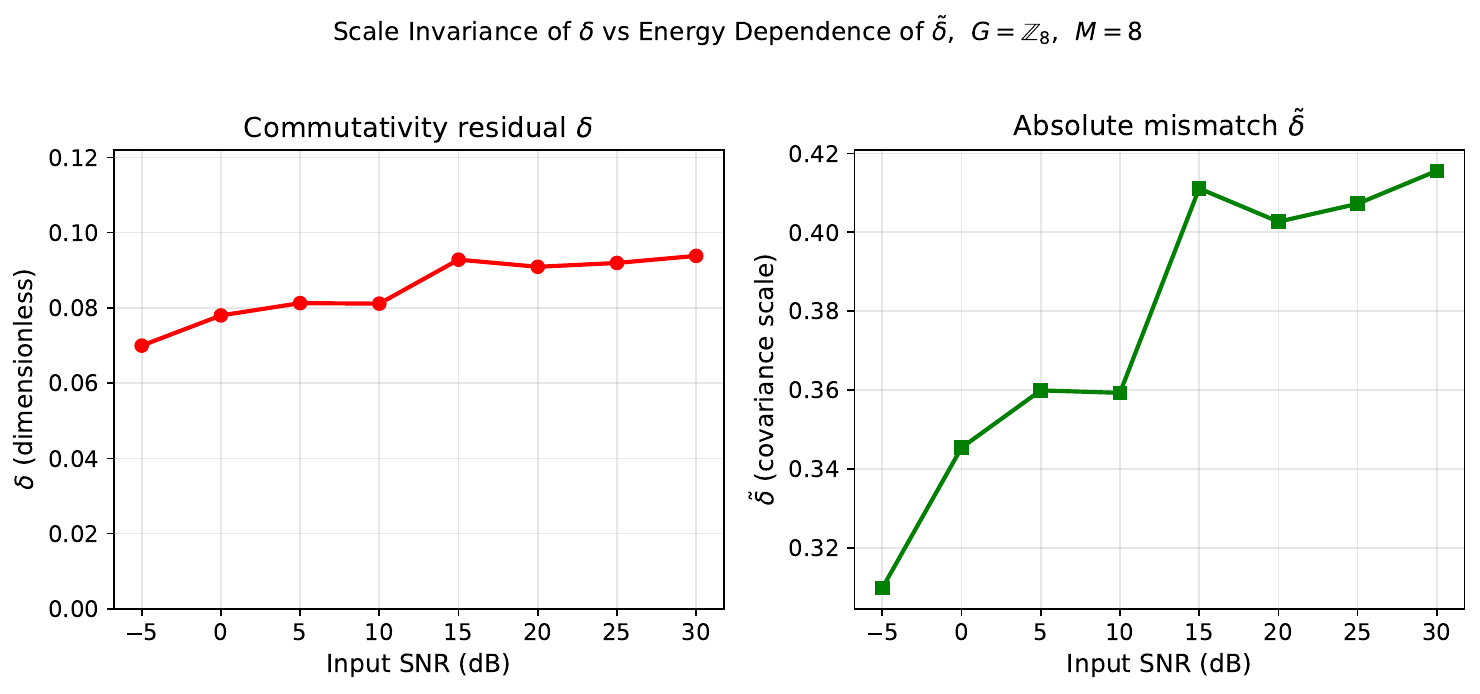}
\caption{Commutativity residual $\delta$ (left) and absolute mismatch $\tilde{\delta}$ (right) vs.\ SNR for $G = \mathbb{Z}_8$, $M = 8$. The scale-invariant $\delta$ is flat; the energy-weighted $\tilde{\delta}$ grows with SNR.}
\label{fig:delta_vs_snr}
\end{figure}

\section{Discussion}\label{sec:discussion}

\subsection{Implications of Colored Noise Characterization}

The extension to colored noise (Section~\ref{sec:colored}) has direct practical significance for each of the principal application domains of algebraic diversity.

In MIMO channel estimation for 5G/6G systems, adjacent-cell interference creates spatially colored noise at the receiver array. Current 3GPP approaches handle this with interference rejection combining (IRC), which is a form of pre-whitening. The group-theoretic noise characterization provides a structured alternative: if the interference has regularity (e.g., periodicity from the cell grid geometry), its natural group may be identifiable, enabling fast structured whitening that integrates naturally with the algebraic diversity channel estimator.

In active noise cancellation, the acoustic signal to be cancelled is itself a colored noise process. The algebraic coloring index $\alpha(\mathbf{Q})$ provides a principled measure of the noise complexity that could guide the choice of cancellation architecture: low $\alpha$ (nearly white noise) admits simple filters, while high $\alpha$ (strongly colored noise with identifiable group structure) benefits from the structured whitening pipeline of Theorem~\ref{thm:colored_noise}.

In passive RF geolocation, environmental multipath and co-channel interference create spatially colored noise that degrades MUSIC and ESPRIT performance. The noise characterization workflow of Section~\ref{sec:colored} enables estimation of the noise group structure from signal-bearing observations, avoiding the need for noise-only snapshot windows that may be operationally unavailable. If the noise natural group can be identified from limited data (exploiting the reduced parameter count of the group-constrained covariance model), the resulting structured whitening may outperform conventional sample-covariance-based pre-whitening when noise-only snapshots are scarce.

\subsection{Computational Considerations}

The standard covariance requires $O(LM^2)$ operations for $L$ snapshots. The CG method requires $O(|G|M)$ for group orbit computation plus $O(M^3)$ for eigendecomposition. For the cyclic group ($|G| = M$), the CG method requires $O(M^2 + M^3)$ total, comparable to single-snapshot covariance computation but yielding a full-rank estimator rather than a rank-1 matrix. For larger groups, the group orbit computation grows, but the eigendecomposition benefits from the block structure imposed by the group representation, enabling efficient computation via the fast Fourier transform on the group~\cite{clausen1989}.

\subsection{The Geometry of Observation}\label{sec:geometry_observation}

The results of this paper suggest a broader philosophical principle connecting geometry, symmetry, and measurement, which we develop here by synthesizing two threads: the algebraic structure of observation and the information content of group action.

By analogy with Klein's Erlangen program in geometry~\cite{klein1872}, where the ``content'' of a geometry is determined by its symmetry group, Euclidean geometry by the group of rigid motions, projective geometry by the projective group, our results suggest that the \emph{information extractable from a measurement} is determined by the symmetry group brought to bear on it. The covariance matrix imposes a bilinear structure on the data; the Cayley graph transform imposes a group-theoretic structure; and different groups reveal different aspects of the same observation. The symmetric group, being the maximal permutation group, reveals the maximum extractable information through its Cayley graph construction (Theorem~\ref{thm:optimality}). This reframes the question of statistical estimation: rather than asking ``how many observations do I need?'', one can ask ``what group structure does my observation admit?'' The answer determines the information accessible from a single measurement.

To make this precise, consider the idealized setting of perfect sensors. We posit that each independent sensor (or, under stationarity, each independent time sample) provides one constant unit of independent observational information. This unit is not a bit in the Shannon sense, it is a geometric object: one axis in an $M$-dimensional observation space along which signal structure can be distinguished from noise. We call the total number of such independent units the \emph{observational rank} of the measurement configuration, denoted $\rho$.

Although we describe observations as projections into a geometric space, this viewpoint is distinct from Amari's \emph{information geometry}~\cite{amari1983}, in which the objects of study are statistical manifolds, spaces of probability distributions equipped with a Riemannian metric derived from the Fisher information. Our concern is with a different space: the Euclidean observation space $\mathbb{C}^M$ containing quantities that comprise both well-structured (deterministic) and random components, directly corresponding to signals and noise respectively. From a theoretical viewpoint, the random components are entropic and the structured components are deterministic, so the geometric space considered here may be regarded as including Amari's space as a subspace. This concept is similar and may be equivalent to other defined manifolds and spaces used in computational information geometry and distance preservation~\cite{critchley2016,lee2007}, and is mentioned here not as a novel element but as an aid to developing an intuitive appreciation for the theorems and other properties in support of algebraic diversity.

An important subtlety is that in the native measurement basis, the raw observations from $M$ sensors appear to be near-copies of each other because they are dominated by the same few strong signal components. However, each sensor also captures a slightly different combination of the weaker components, and those differences are the independent information units. The group action finds the coordinate system that makes all $M$ contributions visible simultaneously, which is necessarily an orthogonal basis because independence in Euclidean space \emph{is} orthogonality. The group action decomposes the sample into its $M$ independent constituents; algebraic diversity is the algebraic mechanism that accomplishes this decomposition from a single vector observation.

Under this interpretation, the results of this paper admit the following statements:
\begin{enumerate}
\item A single sensor ($M = 1$) provides $\rho = 1$. The only group of order~$1$ is the trivial group $G = \{e\}$, and the group-averaged estimator reduces to the outer product $\mathbf{x}\mathbf{x}^H$, a rank-1 matrix capturing exactly one unit of information (Theorem~\ref{thm:trivial}). This is the degenerate endpoint of the $(G, L)$ continuum.
\item As $M$ grows, the group action projects each sensor's contribution onto a distinct orthogonal axis. The Karhunen--Lo\`eve transform, which the Optimality Theorem (Theorem~\ref{thm:optimality}) identifies as the spectral decomposition of the optimal group, achieves this projection uniquely, producing $M$ maximally separated, independent spectral components.
\item At $M$ sensors, the one-to-one correspondence between sensors, group elements, and spectral axes is exact, and the group gain reaches its ceiling of $M\times$ in linear scale ($10\log_{10}(M)$~dB), attained when the signal energy is flat across the matched axes; for concentrated spectra the realized gain is the participation-ratio factor of Remark~\ref{rem:gl_continuum}.
\item Attempting to project $M$ observational units into a space of more than $M$ dimensions, by using a group of order $|G| > M$, maps energy onto axes that do not correspond to independent information, degrading the estimate. This is the mechanism underlying the group order constraint and the failure of $S_M$ subsampling (Section~\ref{sec:ordering}).
\item TAD and SAD are informationally equivalent under stationarity: each measurement contributes one independent observational unit, provided the signal does not change between measurements.
\end{enumerate}

The observational rank $\rho = M$ is additive, not logarithmic: two independent measurement configurations with ranks $\rho_1$ and $\rho_2$ yield a combined rank of $\rho_1 + \rho_2$. The logarithm in $10\log_{10}(M)$ is a unit conversion to decibels, not a fundamental aspect of the combining law. This is in contrast to Shannon entropy, where the logarithm arises from counting distinguishable states in a combinatorial space. Observational rank counts dimensions, not states, and dimensions combine additively.

The \emph{structural entropy} $H_{\mathrm{struct}} = -\sum_{k=1}^M (\lambda_k / \Tr(\mathbf{R})) \log(\lambda_k / \Tr(\mathbf{R}))$, where $\lambda_k$ are the eigenvalues of the covariance, measures a complementary quantity: not how many independent observational units are available (that is $\rho$), but how the signal energy distributes across them. White noise maximizes $H_{\mathrm{struct}}$ (uniform energy across all axes); a rank-1 signal minimizes it (all energy on one axis). The algebraic diversity framework finds the projection that reveals the true structural entropy of the signal from a single observation, because the $M$ observational units are already present in the data, they need only be separated by the appropriate group action.

\subsection{Information Structure versus Information Content}\label{sec:info_structure}

The preceding analysis suggests a distinction that has not been previously formalized: the difference between \emph{information content} and \emph{information structure}.

Shannon's information theory~\cite{shannon1948} characterizes the content of a message: how many bits are needed to encode it, how fast it can be transmitted through a noisy channel, what is the minimum distortion achievable at a given rate. Shannon's framework is deliberately agnostic to how the information is organized: a bit is a bit whether it resides in a time-domain sample, a spatial sensor, a frequency coefficient, or a graph edge weight. This abstraction is what makes information theory universal.

The algebraic diversity framework characterizes the \emph{structure} of information: the representation-theoretic symmetry of the data object in which the observation is encoded. A scalar (rank-0 tensor) has no internal structure; the trivial group $\{e\}$ is the only available group, and the law of large numbers governs estimation. A vector (rank-1 tensor) has index structure admitting cyclic, dihedral, or other permutation groups. A matrix (rank-2 tensor) has row-column structure admitting the adjoint action. A graph has automorphism structure. A quantum state in Hilbert space has unitary group structure.

The matched group is determined by the data structure, not by the data content. The DFT dominates classical signal processing not because signals are special, but because signals are stored in vectors (rank-1 tensors), and the cyclic group is the natural symmetry of the vector index set. When the same information is represented in a different structure (a graph, a neural network weight matrix, a quantum density matrix), the cyclic group is no longer matched, and genuinely different groups provide superior estimation.

This perspective complements Shannon's: information content determines how many bits are present; information structure determines how efficiently those bits can be extracted from a measurement. Two signals with identical Shannon entropy but different organizational structures require different groups and achieve different estimation efficiencies. Shannon cannot distinguish them; algebraic diversity can.

\subsection{Implications for Single-Shot Estimation}

The General Replacement Theorem has potential implications beyond the MUSIC application demonstrated here. Any estimation problem that relies on temporal averaging of second-order statistics, including covariance estimation, principal component analysis, independent component analysis, and spectral density estimation, is, by Theorem~\ref{thm:trivial}, implicitly operating with the trivial group $G = \{e\}$. The $(G, L)$ continuum (Remark~\ref{rem:gl_continuum}) shows that each of these problems admits a richer algebraic diversity formulation that can trade temporal snapshots for algebraic group structure, potentially eliminating or reducing the multi-snapshot requirement. The conditions required (signal equivariance and noise ergodicity) are satisfied in many practical settings, suggesting that the conventional multi-snapshot requirement in these problems is not fundamental but rather an artifact of the implicit trivial-group assumption.

The implications of algebraic diversity for single-shot quantum state estimation, where repeated measurement of the same quantum state is physically prohibited by the measurement postulate, will be explored in subsequent work. The colored noise extension of Section~\ref{sec:colored} is particularly relevant in this context, since quantum measurement noise is generically non-isotropic: the noise structure depends on the measurement basis, and the group-theoretic characterization developed here provides a natural language for describing this basis-dependent noise.

\subsection{Signal Classes and Their Natural Groups}

The framework developed in this paper provides a unified explanation for why certain spectral transforms are effective for certain signal classes: each transform is the spectral decomposition of a specific algebraic group, and its effectiveness is determined by how well that group's structure matches the signal's covariance. The following correspondence between signal classes and groups extends beyond the DFT--cyclic and DCT--dihedral cases discussed in Section~\ref{sec:intro}.

\textbf{Hexagonal and triangular sensor arrays.} Arrays with hexagonal geometry (used in 5G massive MIMO base stations, sonar, and radio telescopes) have natural $D_6$ symmetry (the dihedral group of the hexagon). Standard processing applies the DFT ($\mathbb{Z}_M$), which is suboptimal because the array's geometric symmetry is dihedral, not cyclic. The AD framework predicts that a $D_6$-matched transform should provide a better structural match. Preliminary experiments on a 7-element hexagonal array confirm that the commutativity residual $\delta$ is approximately 29\% lower for $D_6$ than for $\mathbb{Z}_7$ (averaged over all source azimuths at 15~dB SNR), validating the structural prediction. However, the downstream estimation quality (eigenvalue SNR, subspace distance) does not uniformly favor $D_6$, indicating that the commutativity residual is a necessary but not sufficient condition for superior estimation: the group's representation structure, including the dimensionality and multiplicity of its irreducible representations, also affects the effective rank and conditioning of the group-averaged estimator. This finding motivates extending the group selection criterion beyond $\delta$ alone to incorporate representation-theoretic properties, and is an open problem for future work.

\textbf{Signals on graphs.} Graph signal processing defines the ``Fourier transform on a graph'' via the eigenvectors of the graph Laplacian. The graph Laplacian commutes with the graph's automorphism group. AD applies directly: if a graph has automorphism group $G$, then the $G$-matched transform is the natural spectral tool, and PASE provides single-snapshot spectral estimation on the graph. Social networks, communication networks, and molecular graphs each have specific symmetry structures that determine the appropriate group.

\textbf{Transient signals (chirps, pulses).} A chirp has linearly varying instantaneous frequency, producing a non-circulant covariance. The fractional Fourier transform, which is the standard tool for chirps, is connected to the affine group (translations and dilations). Formalizing this connection within AD would provide single-snapshot chirp analysis with applications to radar pulse compression, sonar, and seismic processing.

\textbf{Crystallographic symmetries.} X-ray crystallography is spectral estimation matched to the crystal's symmetry group. Cubic crystals have octahedral symmetry (order 48), hexagonal crystals have $D_6$. Diffraction patterns decompose according to the irreducible representations of the crystal's point group. Crystallographers have been performing a form of algebraic diversity for a century without formalizing it as such.

\textbf{Bilateral symmetry in biomedical signals.} EEG electrode arrays are typically symmetric across the midline, giving the spatial covariance a reflection symmetry best matched by a dihedral group rather than the cyclic group used in standard DFT-based processing.

\textbf{Exchangeable signals.} When $M$ sensors are statistically exchangeable (any permutation preserves the joint distribution), the covariance commutes with $S_M$. This is rare for spatial arrays but common for replicated experiments, multiple trials, or portfolio-level financial data.

\subsection{The Pragmatic Value of Algebraic Diversity}

The signal class analysis above, together with the hexagonal array experiment, leads to an important pragmatic observation. The primary value of AD is not that it identifies the ``perfect'' group for every signal, it is that it provides three capabilities that are immediately useful regardless of group selection:

\begin{enumerate}
\item \textbf{Rank-lift from a single snapshot.} The group-averaged estimator achieves full rank ($M$) from one observation using \emph{any} group, including the simplest choice $\mathbb{Z}_M$. This eliminates the multi-snapshot bottleneck in subspace methods.

\item \textbf{Group gain of up to $10\log_{10}(M)$~dB.} For group-matched signals the SNR improvement is immediate and requires no tuning beyond the array size.

\item \textbf{PASE determines the averaging depth.} Using $n = |G|$ elements is provably optimal, eliminating a degree of freedom that would otherwise require empirical calibration.
\end{enumerate}

The theoretical question of optimal group selection is intellectually rich and may yield additional performance in specific applications, but the cyclic group $\mathbb{Z}_M$ provides adequate performance for the large majority of practical signals, precisely because most engineered signals are periodic, and $\mathbb{Z}_M$ is their matched group. The framework's contribution is not to replace the DFT but to explain \emph{why} the DFT works when it does, to predict \emph{when} it will be suboptimal, to provide the algebraic machinery to handle those cases, and, through the Trivial Group Embedding (Theorem~\ref{thm:trivial}), to reveal that conventional temporal averaging is the degenerate endpoint of a richer estimation continuum.

\begin{remark}[Summary of the PASE result]
The group selection problem identified in this paper has been addressed in Sections~\ref{sec:pase}--\ref{sec:blind_matching}. The PASE optimality result (Theorem~\ref{thm:pase}) establishes that $n = |G|$ is the sharp optimal averaging depth; the ordering experiment (Section~\ref{sec:ordering}) demonstrates that $S_M$ subsampling fails; and the spectral concentration criterion (Section~\ref{sec:blind_matching}) provides a blind single-snapshot group selection method. Together, these results collapse the framework to a single free parameter: the choice of group.
\end{remark}

\subsection{The General Algebraic Averaging Conjecture}\label{sec:gaat}

The results of this paper establish the algebraic diversity principle for the specific statistic $f(\mathbf{x}) = \mathbf{x}\mathbf{x}^H$ (the outer product, which estimates the covariance). A natural question is whether the same principle extends to arbitrary statistics. We conjecture that it does.

\begin{definition}[Effective Dimension]\label{def:deff}
For a $G$-equivariant statistic $f: \mathbb{C}^M \to \mathcal{V}$ and a group $G$ acting via $\pi$, the \emph{effective dimension} $d_{\mathrm{eff}}(G, f)$ is the number of effective independent components produced by the group orbit, i.e.\ the effective sample size that governs the variance of the orbit average $\hat{\theta}_G = \frac{1}{|G|}\sum_g f(\pi_g(\mathbf{x}))$. For the covariance statistic $f(\mathbf{x}) = \mathbf{x}\mathbf{x}^H$ it is
\begin{equation}\label{eq:deff}
d_{\mathrm{eff}}(G, f) = \frac{M^2}{\dim\mathcal{C}(\rho)},
\end{equation}
the effective dimension~\eqref{eq:deff_cov}, at most $M$ for a group matched to a covariance with distinct eigenvalues (Remark~\ref{rem:gl_continuum}), which equals $M$ for a matched order-$M$ regular (cyclic) representation. The count of algebraically distinct orbit values $|\{f(\pi_g(\mathbf{x})) : g \in G\}|$ is an upper bound on $d_{\mathrm{eff}}$ and coincides with it only for a regular Abelian representation; it is \emph{not} itself the variance-reduction factor.
\end{definition}

\begin{remark}
For the covariance statistic with a matched order-$M$ regular (cyclic) group, $\dim\mathcal{C}(\rho) = M$ and $d_{\mathrm{eff}} = M$, so the $M$ matched directions, rather than the raw count $|G|$, govern the variance reduction. The effective dimension does not grow without bound with $|G|$: enlarging the group beyond the matched order does not enlarge the matched commutant, and adding elements outside the commutant of $\mathbf{R}$ injects bias (the estimate tends to the two-eigenvalue limit~\eqref{eq:sm_expect}). Hence $S_M$ does not yield an $M!$-fold reduction; this is the variance-side counterpart of the PASE result (Section~\ref{sec:pase}).

For scalar symmetric statistics such as $f(\mathbf{x}) = \frac{1}{M}\sum_j x_j^k$ (the $k$-th sample moment), every permutation produces the same scalar, so $d_{\mathrm{eff}} = 1$ regardless of $|G|$. The variance reduction $1/M$ arises from the $M$ independent data values, not from the group orbit. This is the analog of the Gaussian scale MLE: when the statistic is already $S_M$-invariant, no group action can extract additional information.
\end{remark}

\begin{conjecture}[General Algebraic Averaging]\label{conj:gaat}
Let $\mathbf{x} = \mathbf{s} + \mathbf{n} \in \mathbb{C}^M$, where $\mathbf{s}$ is deterministic and $\mathbf{n} \sim \mathcal{CN}(\mathbf{0}, \sigma^2 \mathbf{I}_M)$. Let $G$ be a finite group with unitary representation $\pi: G \to U(M)$, and let $f: \mathbb{C}^M \to \mathcal{V}$ be a $G$-equivariant statistic (i.e., there exists an induced action $\tilde{\pi}$ on $\mathcal{V}$ such that $f(\pi_g(\mathbf{x})) = \tilde{\pi}_g(f(\mathbf{x}))$ for all $g$). Then the group-averaged estimator $\hat{\theta}_G = \frac{1}{|G|}\sum_{g \in G} f(\pi_g(\mathbf{x}))$ satisfies:
\begin{enumerate}
\item[(i)] $E[\hat{\theta}_G] = \theta_f^G$, the projection of $E[f(\mathbf{x})]$ onto the $G$-invariant subspace of $\mathcal{V}$.
\item[(ii)] $\mathrm{Var}(\hat{\theta}_G) \leq C(f, \mathbf{s}) / d_{\mathrm{eff}}(G, f)$, where $C$ depends on $f$ and $\mathbf{s}$ but not on $G$.
\item[(iii)] For $L$ independent observations each processed with group $G$: $\mathrm{Var}(\hat{\theta}_{G,L}) \leq C(f, \mathbf{s}) / (d_{\mathrm{eff}} \cdot L)$.
\item[(iv)] Setting $G = \{e\}$ and $L = N$ recovers the classical law of large numbers with $\mathrm{Var} \propto 1/N$.
\end{enumerate}
\end{conjecture}

If Conjecture~\ref{conj:gaat} holds, then the perceived need for multiple observations in statistical estimation is, in general, an artifact of applying the trivial group to each observation. The classical law of large numbers, variance $\sigma^2/N$ from $N$ independent observations, is recovered as the special case $G = \{e\}$, $d_{\mathrm{eff}} = 1$, $L = N$. The sample covariance result of this paper is the special case $f(\mathbf{x}) = \mathbf{x}\mathbf{x}^H$, $d_{\mathrm{eff}} = M$, $L = 1$.

\subsubsection{Experimental confirmation}

Monte Carlo experiments (1000 random integers in $[-1000, 1000]$, 5000 trials per configuration) confirm the conjecture for the first four sample moments.

\textbf{Variance scaling.} The log-log slope of $\mathrm{Var}(\hat{m}_k)$ versus $M$ is $-1.004$, $-1.000$, $-1.000$, and $-0.999$ for $k = 1, 2, 3, 4$ respectively (the conjecture predicts $-1.000$).

\textbf{$(G, L)$ continuum.} For these scalar moment statistics the group contributes no algebraic diversity ($d_{\mathrm{eff}}(G) = 1$), so the available diversity is the data dimension $M$ (the number of independent values) and the snapshot count $L$. With a fixed budget $M \cdot L = 1000$, the product $\mathrm{Var}(\hat{m}_1) \times 1000$ is constant at approximately $335{,}000$ across all eight $(M, L)$ configurations from $(1, 1000)$ to $(1000, 1)$, with variation less than 5\% (consistent with Monte Carlo noise), confirming $\mathrm{Var} \propto 1/(M \cdot L)$. The same constant-product relationship holds for $\hat{m}_4$.

\textbf{Group independence for symmetric statistics.} At dimension $M = 20$, the cyclic group $\mathbb{Z}_{20}$ (order 20), the direct product $\mathbb{Z}_4 \times \mathbb{Z}_5$ (order 20), and the dihedral group $D_{20}$ (order 40) all produce the same variance for scalar moments: ratios $\mathbb{Z}_{20}/D_{20} = 0.990$ and $\mathbb{Z}_{20}/(\mathbb{Z}_4 \times \mathbb{Z}_5) = 1.005$, confirming that variance depends on $M$ (data dimension), not on $|G|$ (group order), when $d_{\mathrm{eff}} = 1$.

These experiments are consistent with every prediction of Conjecture~\ref{conj:gaat}. A formal proof for the outer-product statistic via the Peter--Weyl decomposition, together with a Rao--Blackwell argument that handles general $G$-compatible statistics conditional on a Clebsch--Gordan tensor lemma for higher moments, is given in~\cite{thornton2026framework_arxiv}.

\section{Conclusion}\label{sec:conclusion}

We have established a general theoretical framework proving that temporal averaging over multiple observations is a special case, specifically, the trivial-group case $G = \{e\}$, of algebraic group action on observations for the purpose of second-order statistical information extraction. Selecting a richer group enables equivalent or superior subspace estimation from a single observation. The General Replacement Theorem identifies two conditions, signal equivariance and noise ergodicity, under which a group-averaged estimator from one observation achieves equivalent subspace decomposition to the multi-snapshot sample covariance. The Optimality Theorem establishes the KL transform as the optimal linear decomposition and shows that the matched group attains it; the symmetric group reaches the KL spectrum only through its Cayley graph spectral construction, which is distinct from the group-averaged estimator $\mathbf{F}_{S_M}$ and does not make $\mathbf{F}_{S_M}$ optimal. The Temporal--Algebraic Duality Principle formalizes the deep connection between these two modes of information extraction: both are mechanisms for averaging out unstructured noise to reveal invariant signal structure, operating at different points on the $(G, L)$ continuum (Remark~\ref{rem:gl_continuum}) with variance $\propto 1/(d_{\mathrm{eff}} \cdot L)$.

We have further shown that the framework extends naturally to colored noise environments through a group-theoretic characterization of the noise covariance. The natural group of a noise process identifies the algebraic structure of its correlations, the algebraic coloring index quantifies its departure from whiteness, and the generalized replacement theorem establishes that algebraic diversity applies to whitened observations with the same optimality guarantees as the white noise case. The commutativity residual~$\delta$ and absolute commutativity mismatch~$\tilde{\delta}$, together with the algebraic coloring index~$\alpha$, provide a suite of complementary metrics for quantifying the relationship between a group and a signal model: $\alpha$ measures available structure, $\delta$ measures structural alignment, and $\tilde{\delta}$ measures the practical magnitude of the mismatch. Crucially, when the noise admits a known group structure, the whitening operation inherits the fast transform algorithms of that group, and the entire signal processing pipeline, noise characterization, whitening, and signal extraction, remains within the algebraic framework.

The framework has been validated through the MUSIC direction-of-arrival estimation problem, where the Cayley graph construction achieves multi-signal resolution from a single snapshot that the standard covariance method cannot, and through massive MIMO channel estimation, where AD-based single-pilot estimation achieves up to 64\% higher effective throughput than MMSE by eliminating the pilot overhead that dominates large-array systems. A third application to single-pulse chirp waveform characterization demonstrates the constructive group matching pipeline in action: the framework independently derives the classical dechirp-then-DFT operation from first principles as an instance of group conjugation, achieves $8.3\times$ higher spectral concentration than the mismatched cyclic group, provides blind chirp rate estimation from a single pulse with RMSE of 0.01 at 10~dB SNR (robust to $-2$~dB), enables four-class waveform classification at 90\% accuracy from 14~dB SNR, identifies LFM chirps at 8~dB lower SNR than FFT-based classification, and maintains 89\% classification accuracy against a non-stationary modulated source while FFT-based processing plateaus at 53\%. A fourth application to graph signal processing addresses the open question of whether non-Abelian groups are ever genuinely necessary: a systematic filtering pipeline and randomized search identify candidate graphs on which the $S_3$ automorphism group achieves $17$--$25\%$ higher expected single-observation spectral concentration than any conjugated cyclic group. Representation-theoretic analysis reveals the mechanism: Schur's lemma forces the non-Abelian estimator to preserve the symmetry of degenerate eigenspaces as indivisible blocks, suppressing eigenvalue variance that Abelian estimators cannot avoid. The refined Non-Abelian Dominance Hypothesis (Conjecture~\ref{conj:nadh}) formalizes this as a connection between irreducible representation dimension and eigenspace degeneracy, providing a new perspective linking representation theory, estimation theory, and algebraic group selection. The Automorphism Characterization Theorem (Theorem~\ref{thm:aut_char}) establishes that the commutativity residual is an exact algebraic oracle for graph automorphisms: $\delta = 0$ if and only if the permutation is an automorphism, providing a per-permutation algebraic test for graph symmetry detection. The Permutation Commutator Formula (Proposition~\ref{prop:perm_comm}) reveals the geometric content of $\delta$: the commutativity residual is the eigenvalue-gap-weighted mixing of the eigenvectors of $\mathbf{R}$ by the permutation, which becomes a sum of squared eigenvalue displacements when the permutation carries eigenvectors to eigenvectors. The Sequential GEVP with group-theoretic deflation (Algorithm~\ref{alg:seqgevp}) lifts the per-permutation test to multi-generator non-Abelian recovery, with four named correctness results (forward progress, strict subgroup growth, an iteration bound $K \leq \lceil\log_2|G_K|\rceil = O(M\log M)$, and generic convergence $G_K \subseteq \mathrm{Aut}(\mathbf{R})$ at $\tau=0$); soundness of the recovered subgroup is guaranteed, while completeness depends on a basis-design condition documented as an open problem. A CAD--DAD bridge experiment using a generic basis of five permutation generators, chosen without knowledge of the graph structure, identifies a graph automorphism as the minimum-$\delta$ generator on all six test graphs, providing empirical evidence for the continuous-relaxation viewpoint developed in~\cite{thornton2026framework_arxiv}.

The Permutation-Averaged Spectral Estimation (PASE) result (Theorem~\ref{thm:pase}) establishes that the optimal averaging depth is exactly $n = |G|$, with a sharp decline for $n > |G|$, a property with no analog in conventional statistical estimation. Furthermore, the group order constraint (Remark~\ref{rem:group_order}) establishes that the candidate group must have order exactly~$M$: groups of order larger than $M$, even those whose algebraic structure matches the signal, over-average and degrade the estimate by the same mechanism that causes $S_M$ subsampling to fail. The systematic ordering experiment (Section~\ref{sec:ordering}) demonstrates that subsampling from the symmetric group $S_M$ fails monotonically regardless of the permutation ordering strategy, proving that group selection is the essential and unavoidable problem in the framework. Together, these results collapse the framework from two entangled free parameters (group and averaging depth) to a single parameter: the choice of an order-$M$ group. The blind group matching problem (Section~\ref{sec:blind_matching}), formalized by analogy with blind equalization in communications, identifies the spectral concentration criterion $\psi(G, \mathbf{x})$ as a candidate single-snapshot group selection metric. For the broad class of signals whose covariance admits a unitary transformation to circulant form, including periodic signals, chirps, and frequency-modulated waveforms, the constructive conjugation approach (Section~\ref{sec:constructive}) reduces group matching from a combinatorial library search to continuous parameter estimation: the matched group is the cyclic group $\mathbb{Z}_M$ conjugated by a signal-adapted unitary $\mathbf{U}(\boldsymbol{\theta})$, and maximizing $\psi$ over $\boldsymbol{\theta}$ simultaneously selects the group and characterizes the signal. For signals with intrinsically non-Abelian symmetry, the full discrete library search and Conjecture~\ref{conj:blind} remain the operative tools.

These results suggest that the perceived need for multiple observations in statistical signal processing is, in many cases, an artifact of the implicit use of the trivial group $G = \{e\}$, an information extraction operator (the outer product) that accesses none of the algebraic structure available within each observation. The algebraic diversity framework provides a principled generalization. The General Algebraic Averaging Conjecture (Conjecture~\ref{conj:gaat}) proposes that this generalization extends beyond second-order statistics to all $G$-compatible statistics, with the effective dimension $d_{\mathrm{eff}}$ governing the variance reduction. If the conjecture holds, then the law of large numbers itself is revealed to be the estimation theory of rank-0 data structures under the trivial group, and the entire apparatus of sample averaging is the mechanism by which the trivial group's inability to access internal structure is compensated through temporal repetition.

From a practical standpoint, the framework delivers three capabilities that are available today with no additional theoretical development. First, single-snapshot rank-lift: the group-averaged estimator produces a full-rank covariance estimate from one observation using any group, including the cyclic group $\mathbb{Z}_M$ that is already implicit in DFT-based processing. This eliminates the cold-start period that forces adaptive systems to wait for $L \geq 2M$ snapshots before subspace methods become operational. Second, group gain: for group-matched signals the algebraic averaging yields up to $10\log_{10}(M)$~dB of SNR improvement from a single measurement, requiring no tuning beyond the observation dimension. Third, latency reduction: the PASE result establishes that $n = M$ group elements are both necessary and sufficient, so systems that currently accumulate temporal snapshots for covariance estimation can instead act on the first observation. These capabilities apply to any system that relies on second-order statistics, including beamforming, channel estimation, direction finding, active noise cancellation, and spectral analysis, and are realized by replacing the outer product $\mathbf{x}\mathbf{x}^H$ with the group-averaged estimator $\mathbf{F}_G(\mathbf{x})$, a change that requires $O(M^3)$ computation and no modification to the downstream processing pipeline.


\begin{thebibliography}{99}

\bibitem{thornton2005}
M.~A.~Thornton, ``The Karhunen--Lo\`eve transform of discrete MVL functions,'' in \textit{Proc.\ 35th Int.\ Symp.\ Multiple-Valued Logic (ISMVL)}, pp.~194--199, 2005.



\bibitem{karhunen1946}
K.~Karhunen, ``Zur Spektraltheorie stochastischer Prozesse,'' \textit{Ann.\ Acad.\ Sci.\ Fennicae, AI}, vol.~34, 1946.

\bibitem{loeve1955}
M.~Lo\`eve, \textit{Probability Theory}. Princeton, NJ: Van Nostrand, 1955.

\bibitem{hotelling1933}
H.~Hotelling, ``Analysis of a complex of statistical variables into principal components,'' \textit{J.\ Educ.\ Psychol.}, vol.~24, no.~6, pp.~417--441, 1933.

\bibitem{schmidt1986}
R.~Schmidt, ``Multiple emitter location and signal parameter estimation,'' \textit{IEEE Trans.\ Antennas Propag.}, vol.~34, no.~3, pp.~276--280, 1986.

\bibitem{shan1985}
T.~J.~Shan, M.~Wax, and T.~Kailath, ``On spatial smoothing for direction-of-arrival estimation of coherent signals,'' \textit{IEEE Trans.\ Acoust., Speech, Signal Process.}, vol.~33, no.~4, pp.~806--811, 1985.

\bibitem{liao2016}
W.~Liao and A.~Fannjiang, ``MUSIC for single-snapshot spectral estimation: stability and super-resolution,'' \textit{Appl.\ Comput.\ Harmonic Anal.}, vol.~40, no.~1, pp.~33--67, 2016.

\bibitem{gardner1988}
W.~A.~Gardner, ``Simplification of MUSIC and ESPRIT by exploitation of cyclostationarity,'' \textit{Proc.\ IEEE}, vol.~76, no.~7, pp.~845--847, 1988.

\bibitem{klein1872}
F.~Klein, ``Vergleichende Betrachtungen \"uber neuere geometrische Forschungen,'' Erlangen: A.~Deichert, 1872; reprinted in \textit{Mathematische Annalen}, vol.~43, pp.~63--100, 1893.

\bibitem{clausen1989}
M.~Clausen, ``Fast generalized Fourier transforms,'' \textit{Theoret.\ Comput.\ Sci.}, vol.~67, no.~1, pp.~55--63, 1989.




\bibitem{gallian2021}
J.~Gallian, \textit{Contemporary Abstract Algebra}. Chapman and Hall/CRC, 2021.

\bibitem{bernasconi1999}
A.~Bernasconi and B.~Codenotti, ``Spectral analysis of Boolean functions as a graph eigenvalue problem,'' \textit{IEEE Trans.\ Comput.}, vol.~48, no.~3, pp.~345--351, 1999.

\bibitem{grenander1958}
U.~Grenander and G.~Szeg\H{o}, \textit{Toeplitz Forms and Their Applications}. Berkeley, CA: Univ.\ California Press, 1958.

\bibitem{vershynin2012}
R.~Vershynin, ``How close is the sample covariance matrix to the actual one?'' \textit{Adv.\ Math.}, vol.~231, no.~6, pp.~3038--3068, 2012.

\bibitem{puschel2008foundation}
M.~P\"uschel and J.~M.~F.~Moura, ``Algebraic signal processing theory: Foundation and 1-D time,'' \textit{IEEE Trans.\ Signal Process.}, vol.~56, no.~8, pp.~3572--3585, Aug.\ 2008.

\bibitem{puschel2008space}
M.~P\"uschel and J.~M.~F.~Moura, ``Algebraic signal processing theory: 1-D space,'' \textit{IEEE Trans.\ Signal Process.}, vol.~56, no.~8, pp.~3586--3599, Aug.\ 2008.

\bibitem{puschel2008algorithms}
M.~P\"uschel and J.~M.~F.~Moura, ``Algebraic signal processing theory: Cooley--Tukey type algorithms for DCTs and DSTs,'' \textit{IEEE Trans.\ Signal Process.}, vol.~56, no.~4, pp.~1502--1521, Apr.\ 2008.

\bibitem{pal2010nested}
P.~Pal and P.~P.~Vaidyanathan, ``Nested arrays: A novel approach to array processing with enhanced degrees of freedom,'' \textit{IEEE Trans.\ Signal Process.}, vol.~58, no.~8, pp.~4167--4181, Aug.\ 2010.

\bibitem{vaidyanathan2011coprime}
P.~P.~Vaidyanathan and P.~Pal, ``Sparse sensing with co-prime samplers and arrays,'' \textit{IEEE Trans.\ Signal Process.}, vol.~59, no.~2, pp.~573--586, Feb.\ 2011.

\bibitem{romero2016compressive}
D.~Romero, D.~D.~Ariananda, Z.~Tian, and G.~Leus, ``Compressive covariance sensing: Structure-based compressive sensing beyond sparsity,'' \textit{IEEE Signal Process.\ Mag.}, vol.~33, no.~1, pp.~78--93, Jan.\ 2016.

\bibitem{wijsman1967}
R.~A.~Wijsman, ``Cross-sections of orbits and their application to densities of maximal invariants,'' in \textit{Proc.\ 5th Berkeley Symp.\ Math.\ Statist.\ Probab.}, vol.~1, pp.~389--400, 1967.

\bibitem{eaton1989}
M.~L.~Eaton, \textit{Group Invariance Applications in Statistics}. Hayward, CA: Inst.\ Math.\ Statist., 1989.

\bibitem{nitzberg1980}
R.~Nitzberg, ``Application of maximum likelihood estimation of persymmetric covariance matrices to adaptive processing,'' \textit{IEEE Trans.\ Aerosp.\ Electron.\ Syst.}, vol.~AES-16, no.~1, pp.~124--127, 1980.


\bibitem{johnson1963}
S.~M.~Johnson, ``Generation of permutations by adjacent transposition,'' \textit{Math.\ Comput.}, vol.~17, no.~83, pp.~282--285, 1963.

\bibitem{lehmer1960}
D.~H.~Lehmer, ``Teaching combinatorial tricks to a computer,'' in \textit{Proc.\ Symp.\ Appl.\ Math.}, vol.~10, pp.~179--193, 1960.

\bibitem{heap1963}
B.~R.~Heap, ``Permutations by interchanges,'' \textit{Computer J.}, vol.~6, no.~3, pp.~293--298, 1963.

\bibitem{godard1980}
D.~N.~Godard, ``Self-recovering equalization and carrier tracking in two-dimensional data communication systems,'' \textit{IEEE Trans.\ Commun.}, vol.~28, no.~11, pp.~1867--1875, 1980.

\bibitem{treichler1983}
J.~R.~Treichler and B.~G.~Agee, ``A new approach to multipath correction of constant modulus signals,'' \textit{IEEE Trans.\ Acoust., Speech, Signal Process.}, vol.~31, no.~2, pp.~459--472, 1983.

\bibitem{shalvi1990}
O.~Shalvi and E.~Weinstein, ``New criteria for blind deconvolution of nonminimum phase systems (channels),'' \textit{IEEE Trans.\ Inform.\ Theory}, vol.~36, no.~2, pp.~312--321, 1990.

\bibitem{3gpp38901}
3GPP, ``Study on channel model for frequencies from 0.5 to 100~GHz,'' 3GPP TR~38.901, v16.1.0, Dec.\ 2019.

\bibitem{bluestein1970}
L.~I.~Bluestein, ``A linear filtering approach to the computation of discrete Fourier transform,'' \textit{IEEE Trans.\ Audio Electroacoust.}, vol.~18, no.~4, pp.~451--455, Dec.\ 1970.

\bibitem{ozaktas1996}
H.~M.~Ozaktas, O.~Arikan, M.~A.~Kutay, and G.~Bozdagi, ``Digital computation of the fractional Fourier transform,'' \textit{IEEE Trans.\ Signal Process.}, vol.~44, no.~9, pp.~2141--2150, Sep.\ 1996.

\bibitem{amari1983}
S.~Amari, ``A foundation of information geometry,'' \textit{Electron.\ Commun.\ Japan (Part~I: Commun.)}, vol.~66, no.~6, pp.~1--10, 1983.

\bibitem{shannon1948}
C.~E.~Shannon, ``A mathematical theory of communication,'' \textit{Bell Syst.\ Tech.\ J.}, vol.~27, no.~3, pp.~379--423, July 1948.

\bibitem{critchley2016}
F.~Critchley and P.~Marriott, ``Information geometry and its applications: An overview,'' in \textit{Computational Information Geometry: For Image and Signal Processing}.\ Cham: Springer, 2016, pp.~1--31.

\bibitem{lee2007}
J.~A.~Lee and M.~Verleysen, ``Distance preservation,'' in \textit{Nonlinear Dimensionality Reduction} (Information Science and Statistics).\ New York, NY: Springer, 2007, ch.~4.

\bibitem{su2024rope}
J.~Su, M.~H.~Ahmed, Y.~Lu, S.~Pan, W.~Bo, and Y.~Liu, ``{RoFormer}: Enhanced transformer with rotary position embedding,'' \textit{Neurocomputing}, vol.~568, p.~127063, 2024.

\bibitem{vaswani2017attention}
A.~Vaswani, N.~Shazeer, N.~Parmar, J.~Uszkoreit, L.~Jones, A.~N.~Gomez, {\L}.~Kaiser, and I.~Polosukhin, ``Attention is all you need,'' in \textit{Adv.\ Neural Inform.\ Process.\ Syst.\ (NeurIPS)}, vol.~30, 2017.

\bibitem{elhage2021mathematical}
N.~Elhage, N.~Nanda, C.~Olsson, T.~Henighan, N.~Joseph, B.~Mann, A.~Askell, Y.~Bai, A.~Chen, T.~Conerly, N.~DasSarma, D.~Drain, D.~Ganguli, Z.~Hatfield-Dodds, D.~Hernandez, A.~Jones, J.~Kernion, L.~Lovitt, K.~Ndousse, D.~Amodei, T.~Brown, J.~Clark, J.~Kaplan, S.~McCandlish, and C.~Olah, ``A mathematical framework for transformer circuits,'' \textit{Transformer Circuits Thread}, Anthropic, 2021.

\bibitem{michel2019heads}
P.~Michel, O.~Levy, and G.~Neubig, ``Are sixteen heads really better than one?'' in \textit{Adv.\ Neural Inform.\ Process.\ Syst.\ (NeurIPS)}, vol.~32, 2019.

\bibitem{liu2024kivi}
Z.~Liu, A.~Desai, F.~Liao, W.~Wang, V.~Xie, Z.~Xu, A.~Kyrillidis, and A.~Shrivastava, ``{KIVI}: A tuning-free asymmetric 2-bit quantization for {KV} cache,'' in \textit{Proc.\ Int.\ Conf.\ Machine Learning (ICML)}, 2024.

\bibitem{thornton2026framework_arxiv}
M.~A.~Thornton, ``Algebraic Diversity: Principles of a Group-Theoretic Approach to Signal Processing,'' arXiv:2604.19983 [eess.SP], April 2026.

\bibitem{thornton2026unification}
M.~A.~Thornton, ``Unification of signal transform theory,'' arXiv:2605.11589 [eess.SP], 2026.

\end{thebibliography}
\end{document}